\def\year{2019}\relax
\documentclass[letterpaper]{article} %DO NOT CHANGE THIS
\usepackage{aaai19}  %Required
\usepackage{times}  %Required
\usepackage{helvet}  %Required
\usepackage{courier}  %Required
\usepackage{url}  %Required
\usepackage{graphicx}  %Required
\usepackage{amssymb}
\usepackage{microtype}
\usepackage{amsthm}
\usepackage{amsmath}
\usepackage{thmtools}
\usepackage{thm-restate}
\usepackage{algorithm}
\usepackage{algorithmic}
\usepackage{comment}
\usepackage{color}
\usepackage{natbib}
\usepackage{hyperref}
\newtheorem{thm}{Theorem}%[section]
\newtheorem{lem}[thm]{Lemma}

\newtheorem{cor}[thm]{Corollary}
\newtheorem{defn}[thm]{Definition}

\newtheorem{rmk}{Remark}

\usepackage{pifont}% http://ctan.org/pkg/pifont
\newcommand{\cmark}{\ding{51}}%
\newcommand{\xmark}{\ding{55}}%
\newcommand{\R}{{\mathbb R}}
\newcommand{\bS}{{\bar S}}
 
\newcommand{\Z}{{\mathbb Z}}
\newcommand{\N}{{\mathbb N}}

\newcommand{\E}{{\mathbb E}}
\newcommand{\I}{{\mathbf I}}
\newcommand{\1}{{\mathbf 1}}

\newcommand{\cA}{{\mathcal A}}
\newcommand{\A}{{\mathbf A}}
\renewcommand{\S}{{\mathbf S}}

\newcommand{\D}{{\mathbf D}}

\newcommand{\G}{{\mathcal G}}
\newcommand{\cH}{{\mathcal H}}

\newcommand{\cG}{{\mathcal G}}

\newcommand{\cV}{{\mathcal V}}
\renewcommand{\L}{{\mathbf L}}

\newcommand{\cK}{{\mathcal K}}
\newcommand{\K}{{\mathbf K}}

\newcommand{\tK}{{\mathbf {\tilde K}}}

\newcommand{\tU}{{\mathbf{\tilde U}}}

\newcommand{\M}{{\mathbf M}}

\newcommand{\Q}{{\mathbf Q}}
\newcommand{\U}{{\mathbf U}}

\newcommand{\W}{{\mathbf W}}
\newcommand{\X}{{\mathbf X}}

\newcommand{\cP}{{\mathcal P}}

\renewcommand{\c}{{\mathbf c}}
\newcommand{\e}{{\mathbf e}}
\newcommand{\f}{{\mathbf f}}
\newcommand{\g}{{\mathbf g}}
\newcommand{\h}{{\mathbf h}}

\newcommand{\q}{{\mathbf q}}

\renewcommand{\u}{{\mathbf u}}
\newcommand{\tu}{\mathbf{\tilde u}}
\renewcommand{\v}{{\mathbf v}}
\newcommand{\w}{{\mathbf w}}
\newcommand{\x}{{\mathbf x}}
\newcommand{\tx}{{\mathbf {\tilde x}}}

\newcommand{\y}{{\mathbf y}}

\newcommand{\bsig}{{\boldsymbol {\sigma}}}
\newcommand{\bhsig}{\boldsymbol {\hat {\sigma}}}
\newcommand{\hsig}{{\hat {\sigma}}}
\newcommand{\balpha}{\boldsymbol \alpha}
\newcommand{\talpha}{\boldsymbol{\tilde{\alpha}}}
\newcommand{\bbeta}{\boldsymbol \beta}

\newcommand{\bgam}{{\boldsymbol \gamma}}
\newcommand{\bLambda}{{\boldsymbol \Lambda}}
\def \lo{{\it L\'ovasz-$\vartheta$}}
\def \ls{{\it LS}-labelling}
\def \alg{{\it Pref-Rank}}
\def \sa{{\it Graph Rank}}

\def\strPr{G \boxtimes G}
\def\spLab{Kron-Lab($\strPr$)}
\def\pdLab{PD-Lab($G$)}

\begin{document}
% \title{How many preference pairs suffice to rank a  graph consistently?}
\title{How Many Pairwise Preferences Do We Need to Rank a Graph Consistently?}
\author{Aadirupa Saha\\Indian Institute of Science, Bangalore \\ aadirupa@iisc.ac.in \And Rakesh Shivanna \\ Google Inc., Mountain View \\rakeshshivanna@google.com \And Chiranjib Bhattacharyya\\Indian Institute of Science, Bangalore\\ chiru@iisc.ac.in}
%\author{Siddharth Barman \and Aditya Gopalan \and Aadirupa Saha\\
%Indian Institute of Science\\
%Bangalore 560012 \\
%{\tt \{barman, aditya, aadirupa\}@iisc.ac.in} }%\author{}

\maketitle

%\keywords{Rank aggregation, Pairwise preferences, Consistency, Graphs, Kendall-tau loss, Spearman footrule, Lovasz number, Strong product, Chromatic number.}

%%\vspace*{-20pt}

\begin{abstract}
We consider the problem of optimal recovery of true ranking of $n$ items from a randomly chosen subset of their pairwise preferences. It is well known that without any further assumption, one requires a sample size of $\Omega(n^2)$ for the purpose. We analyze the problem with an additional structure of relational graph $G([n],E)$ over the $n$ items added with an assumption of \emph{locality}: Neighboring items are similar in their rankings. 
Noting the preferential nature of the data, we choose to embed not the graph, but, its \emph{strong product} to capture the pairwise node relationships. 
Furthermore, unlike existing literature that uses Laplacian embedding for graph based learning problems, we use a richer class of graph embeddings---\emph{orthonormal representations}---that includes (normalized) Laplacian as its special case. 
Our proposed algorithm, \alg, predicts the underlying ranking using an SVM based approach over the chosen embedding of the product graph, and is the first to provide \emph{statistical consistency} on two ranking losses: \emph{Kendall's tau} and \emph{Spearman's footrule}, with a required \emph{sample complexity} of $O(n^2 \chi(\bar{G}))^{\frac{2}{3}}$ pairs, $\chi(\bar{G})$ being the \emph{chromatic number} of the complement graph $\bar{G}$. Clearly, our sample complexity is smaller for dense graphs, with $\chi(\bar G)$ characterizing the degree of node connectivity, which is also intuitive due to the~\emph{locality} assumption e.g. $O(n^\frac{4}{3})$ for union of $k$-cliques, 
%complement of $k$-colorable graphs (for any constant $k$), 
or $O(n^\frac{5}{3})$ for random and power law graphs etc.---a quantity much smaller than the fundamental limit of $\Omega(n^2)$ for large $n$. This, for the first time, relates ranking complexity to structural properties of the graph.
%Our analysis crucially exploits the fact that $\vartheta(G \boxtimes G) = \vartheta(G)^2$, $\vartheta$ being the famous Lov\'asz function of the graph, to obtain the generalization performance of \alg. 
We also report experimental evaluations on different synthetic and real datasets, where our algorithm is shown to outperform the state-of-the-art methods.
%We use this to analyse the {\it rademacher complexity} of the function class learned by \alg, that determines its generalization performance.
%which is higher for sparsely connected graph and vice versa -- thus  expressing the difficulty of the ranking problem in terms of structural properties of the underlying graph, an intuitive result yet unattained so far. 

% and in fact $O(n^\frac{4}{3})$ for graphs with $\chi(\bar G) = O(1)$, e.g. union of $k$-cliques, complement of $k$-colorable graphs etc. -- a quantity much smaller than $\Omega(n^2)$
\end{abstract}

% ------------------ Introdcution ----------------

%\vspace*{-8pt}

\section{Introduction}
%\vspace*{-18pt}
\begin{table*}[t]
        \caption{Summary of sample complexities for ranking from pairwise preferences.}
\label{tab:sum_con}
\vspace*{-4pt}
%\hspace*{-10pt}
\begin{center}
\scalebox{0.75}{
\begin{tabular}{|c|c|c|c|}
\hline
\textbf{Reference} & \textbf{Assumption on the Ranking Model} & \textbf{Sampling Technique}  & \textbf{Sample Complexity} \\
\hline
\cite{Mossel08} & Noisy permutation & Active  & $O(n \log n)$ \\
\hline
\cite{Jamieson11} & Low $d$-dimensional embedding & Active  & $O(d \log^2 n)$ \\
\hline
\cite{Ailon12} & Deterministic tournament & Active  & $O(n \text{poly}(\log n))$\\
\hline
\cite{Gleich+11} & Rank-$r$ pairwise preference with $\nu$ incoherence & Random  & $O(n\nu r(\log n)^2)$ \\
\hline
\cite{Negahban+12} & Bradley Terry Luce (BTL) & Random  & $O(n \log n)$\\
\hline
\cite{Wauthier+13} & Noisy permutation & Random & $O(n \log n)$\\
\hline
\cite{Rajkumar+16} & Low $r$-rank pairwise preference & Random  & $O(nr \log n)$\\
\hline
\cite{NiranjanRa17} & Low $d$-rank feature with BTL & Random  & $O(d^2 \log n)$\\
\hline
\cite{Agarwal10} & Graph + Laplacian based ranking & Random  & \xmark\\
\hline
\alg\, (This paper) & Graph + {Edge similarity based ranking} & Random  & $O(n^2 \chi(\bar{G}))^{\frac{2}{3}}$\\
\hline
\end{tabular}}
\vspace*{-20pt}
\end{center}
\end{table*}
%\vspace*{-20pt}

The problem of ranking from pairwise preferences has widespread applications in various real world scenarios e.g. web search~\cite{Page98,Hits99}, gene classification, recommender systems~\cite{Music13}, image search~\cite{MSR} and more. Its of no surprise why the problem is so well studied in various disciplines of research, be that computer science, statistics, operational research or computational biology.
In particular, we study the problem of ranking (or ordering) of set of $n$ items, given some partial information of the relative ordering of the item pairs.

It is well known from the standard results of classical sorting algorithms, for any set of $n$ items associated to an unknown deterministic ordering, say $\bsig_n^*$, and given the learner has access to only preferences of the item pairs, in general one requires to observe $\Omega(n \log n)$ \emph{actively} selected pairs (where the learner can choose which pair to observe next) to obtain the true underlying ranking $\bsig_n^*$; whereas, with \emph{random} selection of pairs, it could be as bad as $\Omega(n^2)$.

\vspace*{2pt}
\noindent \textbf{Related Work.} Over the years, numerous attempts have been made to improve the above sample complexities by imposing different structural assumptions on the set of items or the underlying ranking model. In active ranking setting,~\cite{Jamieson11} gives a sample complexity of $O(d \log^2 n)$, provided the true ranking is realizable in a $d$-dimensional embedding;~\cite{Mossel08} and~\cite{Ailon12} proposed a near optimal recovery with sample complexity of $O(n \log n)$ and $O(n \text{poly}(\log n))$ respectively, under \emph{noisy permutation} and tournament ranking model. For the non-active (random) setting,~\cite{Wauthier+13} and~\cite{Negahban+12} gave a sample complexity bound of $O(n \log n)$ under noisy permutation (with $O(\log n)$ repeated sampling) and BTL ranking model. Recently,~\cite{Rajkumar+16} showed a recovery guarantee of $O(n r \log n)$, given the preference matrix is rank $r$ under suitable transformation.

%However, surprisingly enough, there has not been much work to analyze the required sample complexity for \emph{graph based ranking problem}
However, existing literature on sample complexity for~\emph{graph based ranking problems} is sparse, where it goes without saying that the underlying structural representation of the data is extremely relevant in various real world applications where the edge connections model item similarities e.g. In social network, connection among friends can be modelled as a graph, or in recommender systems, movies under same the genre should lie in close neighbourhood. It is important to note that a relational graph is different from imposing item dependencies through feature representations and much more practical, since side information of exact features may not even be available to the learner as required in the later case.

Furthermore, the only few algorithmic contributions made on the problem of ranking on graphs -- ~\cite{Page98,BiRank17,Multi16,AttriRank17} have not explored their theoretical performance. \cite{Agarwal10,AgarwalT08} proposed an SVM-rank based algorithm, with generalization error bounds for the inductive and transductive graph ranking problems. \cite{Alekh07} derived generalization guarantees for PageRank algorithm. 
To the best of our knowledge, we are not aware of any literature which provide \emph{statistical consistency} guarantees to recover the true ranking and analyze the required sample complexity, which remains the primary focus of this work.

\vspace*{2pt}
\noindent \textbf{Problem Setting} We precisely address the question: Given the additional knowledge of a relational graph on the set of $n$ items, say $G([n],E)$, can we find the underlying ranking $\bsig_n^*$ faster (i.e. with a sample complexity lesser than $\Omega(n^2)$)? 
%Can we get a better sample complexity bound based on the structural properties/ degree of connectivity of the graph? 
Of course, in order to hope for achieving a better sample complexity, there must be a connection between the graph and the underlying ranking -- question is how to model this? 

A natural modelling could be to assume that similar items connected by an edge are close in terms of their rankings or similar node pairs have similar pairwise preferences. E.g. In movie recommendations, if two movies $A$ and $B$ belongs to thriller genre and $C$ belongs to comedy, and it is known that $A$ is preferred over $C$ (i.e. the true ranking over latent topics prefers thriller over comedy), then it is likely that $B$ would be preferred over $C$; and the learner might not require an explicit $(B,C)$ labelled pair -- thus one can hope to reduce the sample complexity by inferring preference information of the neighbouring similar nodes. However, how to impose such a smoothness constraint remains an open problem.

One way out could be to assume the true ranking to be a smooth function over the graph Laplacian as also assumed in \cite{Agarwal10}. However, why should we confine ourself to the notion of Laplacian embedding based similarity when several other graph embeddings could be explored for the purpose? In particular, we use a broader class of \emph{orthonormal representation} of graphs for the purpose, which subsumes (normalized) Laplacian embedding as a special case, and assume the ranking to be a \emph{smooth function} with respect to the underlying embedding (see Sec. \ref{sec:prb_st} for details).
%assume that pairwise preferences respect underlying strong product graph, we assume it is smooth with respect to the underlying embedding.

\vspace*{2pt}
\noindent  \textbf{Our Contributions.} Under the smoothness assumptions, we show a sample complexity guarantee of $O(n^{2}\chi(\bar G))^\frac{2}{3}$ to achieve \emph{ranking consistency} -- the result is intuitive as it indicates smaller sample complexity for densely connected graph, as one can expect to gather more information about the neighboring nodes compared to a sparse graph. 
%Our analysis crucially exploits the fact that $\vartheta(G \boxtimes G) = \vartheta(G)^2$, $\vartheta(G)$ being the famous \lo~ number of graph $G$. 
Our proposed \alg\ algorithm, to the best of our knowledge, is the first attempt in proving \emph{consistency} on large class of graph families with $\vartheta(G) = o(n)$, in terms of \emph{Kendall's tau} and \emph{Spearman's footrule} losses -- It is developed on the novel idea of embedding nodes of the strong product graph $\strPr$, drawing inference from the preferential nature of the data and finally uses a kernelized-SVM approach to learn the underlying ranking. We summarize our contributions:

%We use this to analyse the {\it rademacher complexity} of the function class learned by \alg, that determines its generalization performance.

%Of course the question remains is how to achieve this? Does to preferential nature of the data, we embed the node pairs instead of the nodes and use orthogonal embedding of the strong product graphs for learning purpose. 

%\vspace*{4pt}
%\noindent \textbf{Contributions:} 
\begin{itemize}

\item \emph{The choice of graph embedding}: Unlike the existing literature, which is restricted to Laplacian graph embedding~\cite{AndoZh07}, we choose to embed the strong product $\strPr$ instead of $G$, as our ranking performance measures penalizes every pairwise misprediction; and use a general class of orthonormal representations, which subsumes (normalized) Laplacian as a special case.

%\item Our primary contribution lies in the observation that the \emph{choice of the graph embedding is fundamental} for obtaining consistent graph ranking algorithms. There has been less research in this direction, which always have used Laplacian based graph embedding. In a significant departure from existing literature, we discover that for good generalization guarantees, it is useful to embed the strong product graphs $\strPr$ onto an unit sphere (Thm. \ref{thm:opt_embed}), instead of the existing popular approach of embedding $G$.
% ; although the later is the usually followed approach.% in general graph based learning. %with {\it orthogonal embedding} of $\strPr$ (Sec. \ref{sec:embed_strpr}).
%The optimal choice among all such unit-sphere embeddings for graph ranking is un-clear and is the central theme of this paper.  
%, or equivalently embedding the nodes on a unit sphere  

\item \emph{Our proposed preference based ranking algorithm:} \alg\ is a kernelized-SVM based method that inputs an embedding of pairwise graph $\strPr$. % runs given any pairwise node embedding of $G$.
The generalization error of \alg\ involves computing the transductive \emph{rademacher complexity} of the function class associated with the underlying embedding used (see Thm. \ref{thm:gen_err}, Sec. \ref{sec:alg}). 

\item For the above, we propose to embed the nodes of $\strPr$ with $3$ different orthonormal representations: $(a)$ \spLab\, $(b)$ \pdLab\ and $(c)$ \ls; and derive generalization error bounds for the same (Sec.~\ref{sec:embedding}). 
%(a) and (b) are based on \emph{Orthogonal Representation} \cite{lovasz_shannon} of $\strPr$ on an unit sphere, where as (c) is built upon \ls~ by \cite{Jethava+13}. 

\item \emph{Consistency:} We prove the existence of an optimal embedding in \spLab\ for which \alg\ is statistically consistent (Thm. \ref{thm:const}, Sec. \ref{sec:consistency}) over a large class of graphs, including power law and random graphs. % where $\vartheta(G) = o(n)$, including includes power-law and random graphs.
% This is one of our main contribution, as none of existing work offer consistency.
To the best of our knowledge, this is the first attempt at establishing algorithmic consistency for graph ranking problems. 

\item \emph{Graph Ranking Sample Complexity:}
Furthermore, we show that observing $O(n^2 \chi(\bar{G}))^{\frac{2}{3}}$ pairwise preferences a sufficient for \alg\ to be consistent (Thm.~\ref{thm:sam_com_clr}, Sec.~\ref{sec:sample_complexity}), which implies that a \textit{densely connected graph requires much smaller training data compared to a sparse graph} for learning the optimal ranking -- as also intuitive. Our result is the first to connect the complexity of graph ranking problem to its structural properties. Our proposed bound is a significant improvement in sample complexity (for \emph{random} selection of pairs) for dense graphs e.g. $O(n^\frac{4}{3})$ for union of $k$-cliques; and $O(n^\frac{5}{3})$ for random and power law graphs -- a quantity much smaller than $\Omega(n^2)$.
\end{itemize}

Our experimental results demonstrate the superiority of \alg\ algorithm compared to \sa ~\cite{Agarwal10}, Rank Centrality~\cite{Negahban+12} and Inductive Pairwise Ranking~\cite{NiranjanRa17} on various synthetic and real-world datasets; validating our theoretical claims. Table \ref{tab:sum_con} summarizes our contributions.

%%%%%%%%%%%%%%%%%%%%%%%%%%%%%%%%%%%%%%%%%%%%%%%%%%%

% ------------------ Prelims -------------------

\section{Preliminaries and Problem Statement}
\label{sec:prel_prb}

%In this section, we will introduce some useful concepts that will come handy throughout the paper.

%\subsection{Preliminaries}
%\label{sec:prelims}

\textbf{Notations.}
 Let $[n]:=\{1,2, \ldots n\}$, for $n \in \Z_{+}$. Let $x_i$ denote the $i^\text{th}$ component of a vector $\x \in \R^n$.
%and $\|\x\|_2, \|\x\|_{\infty}$ denote the $\ell_2, \ell_\infty$ norms of $\x$ respectively. 
Let $\1\{\varphi\}$ denote an indicator function that takes the value $1$ if the predicate $\varphi$ is true and $0$ otherwise. Let $\1_n$ denote an $n$-dimensional vector of all $1$'s. 
Let $S^{n-1} = \big\{\u \in \mathbb{R}^{n}\big|\|\u\|_2=1\big\}$ denote a $(n-1)$ dimensional sphere. 
For any given matrix $\M \in \R^{m \times n}$, we denote the $i^{th}$ column by $\M_i, \, \forall i \in [n]$ and $\lambda_1(\M)\ge\ldots\ge\lambda_{n}(\M)$ to denote its sorted eigenvalues, $tr(\M)$ to be its trace. 
Let $\S_n^+ \in \R^{n \times n}$ denote $n\times n$ square symmetric positive semi-definite matrices. 
 %We use $O$, $o$ and $\tilde{O}$ to denote big O, small O and small Omega measures for asymptotic analysis.
%For $i\in[n]$, let $\M_i$ denote the $i^{th}$ row of $\M$.
$G(V, E)$ denotes a simple undirected graph, with vertex set $V=[n]$ and edge set $E \subseteq V \times V$. We denote its adjacency matrix by $A_G$.% and Lov\'asz number by $\vartheta(G)$ (see Def. \ref{defn:lovasz_theta}, Appendix). %The complement graph of $G$ is denoted by $\bar{G}$, with adjacency matrix $\A_{\bar{G}}=\1_n^\top\1_n-\I-\A_{G}$. We denote the  strong product of $G$ with itself by $\strPr$ (see Defn. \ref{def:str_pr} in Appendix for details).

\iffalse %%%%%%%%%%%%%%%%%%%%%%%%%%%%%%%%%%%%%%%%

For a binary classification problem over $n$ instances, with kernel matrix $\K \in \S_n^+$ and labels $\y \in \{\pm1\}^n$, the dual formulation of SVM algorithm is given by
\begin{align}
\label{eq:dual_svm1}
 \omega(\K,\y) &= \max_{\balpha \in \R^n_+}  g(\balpha,\K,\y) \\
\text{ where } g(\balpha,\K,\y) &=\sum\limits_{i=1}^n\alpha_i -\frac{1}{2} \sum\limits_{i,j=1}^n \alpha_i\alpha_jy_iy_j\K_{ij} \nonumber
\end{align}
We denote similar optimization problems involving graph  kernel matrices and binary labels as the above.

\fi%%%%%%%%%%%%%%%%%%%%%%%%%%%%%%%%%%%%%%%%%%%%%%%

\noindent {\bf Orthonormal Representation of Graphs.} \label{sec:ortho_lab} \cite{lovasz_shannon}
%introduced the idea of orthonormal representations for the problem of embedding a graph on an unit sphere. 
An orthonormal representation of $G(V, E),~V = [n]$ is $\U=[\u_1,\ldots,\u_n] \in \mathbb{R}^{d\times n}$ such that $\u_i^\top\u_j=0$ whenever $(i,j)\notin E$ and $\u_i\in \S^{d-1}~\forall i\in [n]$.
Let $Lab(G)$ denote the set of all possible orthonormal representations of $G$ given by
$Lab(G):=\{\U ~|~ \U\; \mbox{is an Orthonormal Representation}\}$.
%\cite{Jethava+13} introduces graph embedding to Machine Learning community and shows its connections to graph-kernel matrices. 
Consider the set of graph kernels $\mathcal{K}(G):=\{\K\in S^+_n ~|~ K_{ii}=1, \forall i\in[n]; ~K_{ij}=0, \forall (i,j)\notin E\}$. 
\cite{Jethava+13} showed the two sets to be equivalent i.e. for every $\U \in Lab(G)$, one can construct $\K \in \cK(G)$ and vice-versa.
%We find it convenient to denote $\K_{mm}(G,\y_n) = \underset{\K\in\cK(G)}{\text{argmin }}\omega(\K,\y_n)$, where $\y_{n} \in \{\pm1\}^n$.

\begin{defn} {\bf Lov\'asz Number.} \cite{lovasz_shannon}
\label{defn:lovasz_theta}
Orthonormal representations $Lab(G)$ of a graph $G$ is associated with an interesting quantity -- {\it Lov\'asz number} of $G$, defined as 
\[
\vartheta(G): = \min_{\U \in Lab(G)}\min_{\c \in S^{d-1}}\max_{i \in V}\frac{1}{(\c^{\top}\u_i)^2}
\]
\end{defn}

\noindent {\it Lov\'asz Sandwich Theorem}: If $I(G)$ and $\chi(G)$ denote the independence number and chromatic number of the graph $G$, then $I(G) \le \vartheta\left( G\right) \le \chi(\bar G)$ \cite{lovasz_shannon}.

\noindent {\bf Strong Product of Graphs.} 
Given a graph $G=(V,E)$, strong product of $G$ with itself, denoted by $\strPr$, is defined over the vertex set $V(\strPr) = V \times V$, such that two nodes $(i, j),(i', j') \in V(\strPr)$ is adjacent in $\strPr$ if and only if $i = i'$ and $(j, j') \in E$, or $(i, i') \in E$ and $j = j'$, or $(i, i') \in E$ and $(j, j') \in E$. Also it is known from the classical work of \cite{lovasz_shannon} that $\vartheta(\strPr) = \vartheta^2(G)$ (see Def. \ref{def:str_pr}, Appendix for details).

%It is well known that $\vartheta(\strPr) = \vartheta(G)^2$~\cite{lovasz_shannon} (see Def. \ref{def:str_pr}, Appendix for details).
%Note that for every node $k \in V( \strPr)$, there exists a corresponding node pair $(i_k,j_k) \in V \times V$ in the original graph $G$.

% ---------------- Problems ------------------------

%\vspace*{-3pt}
\subsection{Problem Statement}
\label{sec:prb_st}
We study the problem of graph ranking on a simple, undirected graph $G = (V, E), ~V=[n]$. Suppose there exists a true underlying ranking $\bsig^*_{n} \in \Sigma_{n}$ of the nodes $V$, where $\Sigma_{n}$ is the set of all permutations of $[n]$, such that for any two distinct nodes $i,j \in V$, $i$ is said to be preferred over $j$ iff $\sigma_n^*(i) < \sigma_n^*(j)$. Clearly, without any structural assumption on how $\bsig_n^*$ relates to the underlying graph $G(V,E)$, the knowledge of $G(V,E)$ is not very helpful in predicting $\bsig_n^*$:

\vspace*{2pt}
\noindent \textbf{Ranking on Graphs: Locality property.} 
A ranking $\bsig_n^*$ is said to have \emph{locality property} if $\exists$ at least one ranking function $\f \in \R^n$ such that $f(i) > f(j)$ iff $\sigma(i) < \sigma(j)$ and
\begin{align}
|f(i) - f(j)| \le c, \text{ whenever } (i,j) \in E, 
\end{align}
where $c > 0$ is a small constant that quantifies the ``locality smoothness" of $\f$.
One way is to model $\f$ as a smooth function over the Laplacian embedding $\L$ \cite{Agarwal10} such that $\f^{\top}\L \f = \sum_{(i,j) \in E}A_G(i,j)\big(f_i - f_j \big)^2$ is small. However, we generalize this notion to a broader class of embeddings:

\emph{Locality with Orthonormal Representations:} Formally, we try to solve for $\f \in \text{RKHS}(\K)$\footnote{RKHS: Reproducing Kernel Hilbert Space} i.e. $\f = \K \balpha$, for some $\balpha \in \R^n$, where the \emph{locality} here implies $\f$ to be a smooth function over the embedding $\K \in \cK(G)$, or alternatively $\f^\top \K^\dagger \f \le B$, where $\K^\dagger$ is the pseudo inverse of $\K$ and $B > 0$ is a small constant (see Appendix \ref{app:rkhs_smooth} for more details).
%or $\lambda_1(\K) = \Theta(n)$, $\f$ lies in $\R^n$,
Note that if $G$ is a completely disconnected graph, $\cK(G) = \{\I_n\}$ is the only choice for $\K$ and $f_i$'s are independent of each other, and the problem is as hard as the classical sorting of $n$ items. But as the density of $G$ increases, or equivalently $\vartheta(G) \le \chi(\bar G)  \ll n$, then $\cK(G)$ becomes more expressive and the problem enters into an interesting regime, as the node dependencies come to play aiding to faster learning rate. Recall that, however we only have access to $G$, our task is to find a suitable $\K$ that fits $\f$ on $G$ and estimate $\bsig^*_n$ accurately. %for the following objective:

%$\K$ is full rank, , e.g. $\K = \I_n$ is a valid orthonormal representation for any graph $G$. However, in the regime where $\lambda_1(\K) \ll n$, 

\noindent \textbf{Problem Setup.} Consider the set of all node pairs $\cP_n = \{(i, j) \in V \times V ~|~ i < j\}$.
Clearly $|\cP_n| = \binom{n}{2}$. We will use $N = \binom{n}{2}$ and denote the pairwise preference label of the $k^\text{th}$ pair $(i_k,j_k)$ as $y_k \in \{\pm 1\}$, such that $y_k:= \mbox{sign}(\sigma^*_n(i_k) - \sigma^*_n(j_k)),~\forall k\in[N]$. %where $y_k = 1$ if $\sigma^{*}(i_k) < \sigma^{*}(j_k)$, $-1$ otherwise, i.e. $y_k = 2\Big(\1((\sigma^*(i_k)-\sigma^*(j_k)) < 0)\Big) - 1, \, \forall k \in [N]$. {\color{blue} $2\Big(\1(\sigma^*(i_k) < \sigma^*(j_k))\Big) - 1$ or $sign(\sigma^*(j_k) - \sigma^*(i_k))$?}
The learning algorithm is given access to a set of randomly chosen node-pairs $S_m \subseteq \cP_n$, such that $|S_m| = m \in [N]$. Without loss of generality, by renumbering the pairs we will assume the first $m$ pairs to be labelled $S_m = \{(i_k,j_k)\}_{k = 1}^{m}$, with the corresponding pairwise preference labels $\y_{S_m} = \{y_k\}_{k =1}^{m}$, and set of unlabelled pairs $\bS_m = \cP_n \backslash S_m = \{(i_k,j_k)\}_{k = m+1}^{N}$.
Given $G$, $S_m$ and $\y_{S_m}$, the goal of the learner is to predict a ranking $ \bhsig_n \in \Sigma_{n}$ over the nodes $V$, that gives an accurate estimate of the underlying true ranking $\bsig_n^*$. We use the following ranking losses to measure performance~\cite{Monjardet98}:
% \textbf{Ranking losses.} Consider any two rankings $\bsig^*, \bhsig \in \Sigma_n$. Then two popular measures for computing their ranking differences are:
~\textbf{Kendall's Tau loss:} ~ $ d_k(\bsig^* ,\hat \bsig) = \frac{1}{N}\sum_{k = 1}^{N}\1\big((\sigma^*(i_k) - \sigma^*(j_k))(\hsig(i_k) - \hsig(j_k)) < 0\big)$ and \textbf{Spearman's Footrule loss:} $ d_s(\bsig^* ,\hat \bsig) = \frac{1}{n}\sum_{i = 1}^{n}\Big|\sigma^*(i)-\hsig(i)\Big| $. $d_k$ measures the average number of mispredicted pairs, whereas $d_s$ measures the average displacement of the ranking order. By Diaconi-Graham inequality \cite{KuVas10}, we know for any $\bsig, \bsig' \in \Sigma_n$,  $d_k(\bsig, \bsig') \le d_s(\bsig, \bsig') \le 2d_k(\bsig, \bsig')$.
% \vspace*{-5pt}
% \begin{equation}
% \label{eq:spear_kendall}
% d_k(\bsig, \bsig') \le d_s(\bsig, \bsig') \le 2d_k(\bsig, \bsig')
% \end{equation}
% \vspace*{-15pt}

\iffalse %%%%%%%%%%%%%%%
Note that for any pair $(i_k,j_k) \in [n]\times[n], ~i_k < j_k$, if we define $0-1$ pairwise ranking loss as 
$ \ell^{0-1}(\sigma^*,\hsig,k) = \1((\sigma^*(i_k) - \sigma^*(j_k))(\hsig(i_k) - \hsig(j_k)) < 0)$,
then the kendall-tau ranking loss becomes:
\begin{align}
d(\sigma^* ,\hsig) = \frac{1}{\binom{n}{2}}\sum_{k = 1}^{\binom{n}{2}}\ell^{0-1}(\sigma^*,\hsig,k).
\end{align}
\fi %%%%%%%%%%%%%%%%%

Now instead of predicting $\bhsig_n \in \Sigma_{n}$, suppose the learner is allowed to predict a pairwise score function $\f: \cP_n \mapsto \R\setminus \{0\}$ (note, $\f = [f_k]_{k=1}^{N} \in (\R\setminus \{0\})^N$ can also be realized as a vector), where $f_k$ denotes the score for every $k^\text{th}$ pair $(i_k,j_k), ~k\in[N]$). We measure the prediction accuracy as \textbf{pairwise $(0$-$1)$ loss:} $\ell^{0-1}(y_k,f_k)=\1\left( f_ky_k < 0 \right)$, or using the convex surrogate loss functions -- \textbf{hinge loss:} $\ell^\text{hinge}(y_k,f_k)=\left( 1-f_{k}y_k \right)_+$ or \textbf{ramp loss:} $\ell^\text{ramp}(y_k,f_k)=\min\{1, \left( 1-f_{k}y_k\right)_+ \}$, where $(a)_+=\max(a,0)$.
%\todoa{Make $f$ consistent}

In general, given a transductive learning framework, following the notations from~\cite{AndoZh07,rademacher_app}, for any pairwise preference loss $\ell$, we denote the empirical (training) $\ell$-error of $\f$ as  
${er}^{\ell}_{S_m}(\f) = \frac{1}{m}\sum_{k = 1}^{m}\ell(y_k,f_k)$, the generalization (test set) error as
${er}^{\ell}_{\bS_m}(\f) = \frac{1}{N-m}\sum_{k = m+1}^{N}\ell(y_k,f_k)$
and the average pairwise misprediction error as $er_{n}^{\ell}(\f) = \frac{1}{N}\sum_{k =1}^{N}\ell(y_k,f_k)$. 

%Our objective is to design a {\it consistent ranking algorithm}, defined as:

%\vspace*{-3pt}
\subsection{Learners' Objective - Statistical Consistency for Graph Ranking from Pairwise Preferences}
\label{sec:obj}
Let $\cG$ be a graph family with infinite sequence of nodes $\cV=\{v_n\}_{n=1}^\infty$. Let $V_n$ denote the first $n$ nodes of $\cV$ and $G_n\in\cG$ be a graph instance defined over ($V_n,E_1 \cup \ldots \cup E_n)$, where $E_n$ is the edge information of node $v_n$ with previously observed nodes $V_{n-1},~n\ge2$. Let $\bsig_n^{*} \in \Sigma_{n}$ be the true ranking of the nodes $V_n$.
Now given $G_n$ and $f\in(0,1)$ a fixed number, let $\Pi_f$ be a uniform distribution on the random draw of $m(f) = \lceil Nf \rceil$ pairs of nodes from $N$ possible pairs $\cP_n$. Let $S_{m(f)} = \{(i_k,j_k) \in \cP_n\}_{k = 1}^{m(f)}$ be an instance of the draw, with corresponding  pairwise preferences $\y_{S_{m(f)}} = \{y_k\}_{k = 1}^{m(f)}$. Given $(G_n,S_{m(f)},\y_{S_{m(f)}})$, a learning algorithm $\cA$ that returns a ranking $\hat{\sigma}_{n}$ on the node set $V_n$ is said to be statistically $d$-rank consistent \emph{w.r.t.} $\cG$ if  
\begin{align*}
Pr_{S_{m(f)} \sim \Pi_f}\left( d(\bsig^*_n, \hat{\bsig}_n) \ge \epsilon \right) \rightarrow 0\quad{ as }\quad n  \rightarrow \infty,
\end{align*}
for any $\epsilon > 0$ and $d$ being the Kendall's tau $(d_k)$ or Spearman's footrule $(d_s)$ ranking losses. 
In the next section we propose \alg\ an SVM based graph ranking algorithm 
%and derive its generalization bound for ${er}^{\ell}_{\bS_m}[\f]$ (Section \ref{sec:gen_err}).
and prove it to be statistically $d$-rank consistent (Sec. \ref{sec:consistency}) with `optimal embedding' in \spLab\,  (Sec. \ref{sec:embed_strpr}).
%-- orthonormal representation of the strong product graph $\strPr$

% ----------------- Algorithm --------------------

\section{\alg\, - Preference Ranking Algorithm}
\label{sec:alg}
%Recall from Sec. \ref{sec:prb_st}, we are given a training set of $m$ pairs $S_m = \{(i_1,j_1), \ldots (i_m,j_m)\}$, along with their corresponding pairwise preferences $y_{S_m} = \{y_1, \ldots, y_m\}$ and the set of unlabelled test pairs $\bS_m = \{ (i_{m+1},j_{m+1}, y_{m+1}), \ldots, (i_{\binom{n}{2}},j_{\binom{n}{2}}, y_{\binom{n}{2}}) \}$. Our objective is to predict ranking a ranking $\hsig$ over $[n]$, consistent to the true underlying ranking $\sigma^*$. 

Given a graph $G(V,E)$ and training set of pairwise preferences $(S_{m},\y_{S_{m}})$, we design an SVM based ranking algorithm that treats each observed pair in ${S_m}$ as a binary labelled training instance and outputs a pairwise score function $\f \in \R^N$, which is used to estimate the final rank $\bhsig_n$.

\textbf{Step 1. Select an embedding $(\tU)$:} Choose a pairwise node embedding $\tU = [\tu_1, \cdots \tu_N] \in \R^{d \times N}$, where any node pair $(i_k,j_k) \in \cP_n$ is represented by $\tu_k,~\forall k \in [N]$. We discuss the suitable embedding schemes in Sec. \ref{sec:embedding}.

\textbf{Step 2. Predict pairwise scores ($\f^* \in \R^N) $:} We solve the binary classification problem given the embeddings $\tU$ and pairwise node preferences $\{(\tu_k,\y_k)\}_{k = 1}^{m}$ using SVM:

%\vspace*{-5pt}
\begin{align} 
\label{eq:svm1}
\underset{\w \in \R^{d}}{\min} ~ \frac{1}{2}\|\w\|^2_2 + C\sum_{k = 1}^{m} \ell^\text{hinge}(y_k, \w^{\top}\tu_k) 
\end{align}
%\vspace*{-5pt}

%\begin{align} \label{eq:svm1} & \underset{\w \in \R^{d}}{\min} ~ \frac{1}{2}\|\w\|_2 + C\sum_{k = 1}^{m}\zeta_k \\ \nonumber \text{s.t. } & y_k\w^{\top}\tu_k \ge 1 - \zeta_k, ~\zeta_k \ge 0,\forall k \in [m], \end{align}
%Recall that $y_k = 2\Big(\1((\sigma^*(i_k)-\sigma^*(j_k)) < 0)\Big) - 1$, 

where $C > 0$ is a regularization hyperparameter. Note that the dual of the above formulation is given by:
\begin{equation*}
%\omega_C(\tK,\y) = 
\underset{\balpha \in \R^{m}_{+},~ \|\balpha\|_\infty\le C}{\max} ~ \sum_{k = 1}^{m}\alpha_k - \frac{1}{2}\sum_{k,k'\in[m]}\alpha_k\alpha_{k'}y_ky_{k'}
\tK_{k,k'}
\end{equation*}
where $\tK = \tU^\top \tU$ denotes the embedding kernel of the pairwise node instances. From standard results of SVM, we know that optimal solution of \eqref{eq:svm1} gives $\w^* = \sum_{k = 1}^{m}y_k\tu_k\alpha_k = \tU\bbeta$, where $\bbeta \in \R^N$ is such that $\beta_k = y_k\alpha_k,~\forall k \in [m]$ and $0$ otherwise. Since $y_k \in \{\pm 1\}$, $\|\balpha\|_{\infty} = \|\bbeta\|_{\infty} \le C$. 
Thus for any $k \in [N]$, the score of the pair $(i_k,j_k)$ is given by $f^*_k = \w^{*\top}\tu_k = \sum_{l \in [m]}y_l\alpha_l\tu_{l}^{\top}\tu_k$ or equivalently $\f^* = \tU^{\top}\w^* = \tU^{\top}\tU\bbeta = \tK\bbeta$, which suggests an alternate formulation of SVM:
\begin{equation}
\label{eq:svm3}
\underset{\f \in \R^N}{\max} ~ \frac{1}{2}\f^{\top}\tK^{\dagger}\f + Cm~\hat{er}^{\ell^{hinge}}_{S_m}(\f)
\end{equation}
% where $\hat{er}^{\ell^{hinge}}_{S_m}[\f] = \frac{1}{m}\sum_{k=1}^{m}\left( 1-f_{k}y_k \right)_+$.
Clearly, if $\f^*$ denotes the optimal solution of \eqref{eq:svm3}, then we have $\f^* \in \{\f ~|~ \f = \tK\bbeta, ~\bbeta \in \R^{N}, ~\|\bbeta\|_{\infty} \le C\}$.

\begin{rmk} 
\emph{
The regularization $\f^{\top}\tK^{\dagger}\f$, precisely enforces the \emph{locality} assumption of Sec. \ref{sec:prb_st}} (see Lem.~\ref{lem:tf_smooth}, Appendix).
\end{rmk}

%{\color{blue} Discussion/Remark: Our SVM objective regularizes based on the smoothness property of $\f$ in the RKHS norm: $\f^\top\tK\F$. The resulting pairwise score function lies in the RKHS($\tK$).}
%We now compute the pairwise scores of every pair $k \in [\binom{n}{2}]$ as $f_k = \w^{\top}\tU_{k}$.

\textbf{Step 3. Predict $\bhsig_n \in \Sigma_n$ from pairwise scores $\f^*$:}
Given the score vector $\f^* \in \R^{N}$ as computed above, predict a ranking $\bhsig_n \in \Sigma_{n}$ over the nodes $V$ of $G$ as follows:

\begin{enumerate}
\item Let $c(i)$ denote the number of wins of node $i\in V$ given by 
$\underset{\{k = (i_k,j_k) | i_k = i\}}{\sum}\hspace*{-4pt}\1 \big( f^*_{k}>0 \big)  +  \underset{\{k = (i_k,j_k) | j_k = i\}}{\sum}\1 \big( f^*_{k}<0 \big)$.
% (or nodes beaten by)
%\item Let $c(i)$ denotes the count of the number of wins of node $i\in V$ (i.e. number of items that are beaten by node $i \in V$). Clearly, $c(i) = \underset{\{k | i_k = i\}}{\sum}\1 \big( f_{k}>0 \big)  + \underset{\{k | j_k = i\}}{\sum}\1 \big( f_{k}<0 \big) , ~\forall i \in [n]$.
\item Predict the ranking of nodes by sorting \emph{w.r.t.} $c(i)$, i.e. choose any $\hat{\bsig}_n \in \text{argsort}(\c)$, where
$\text{argsort}(\c) = \big\{\bsig \in \Sigma_{n} ~|~ \sigma(i) < \sigma(j),\text{ if }c(i) > c(j), ~\forall i,j \in  V \big\}$.
\end{enumerate}

A brief outline of \alg\ is given below:

\vspace{-10pt}
\begin{center}
\begin{algorithm}[h]
\renewcommand{\thealgorithm}{}
   \caption{\textbf{\alg}}
   \label{alg:kron}
\begin{algorithmic}
  \STATE {\bfseries Input:} $G(V,E)$ and subset of preferences $(S_m,\y_{S_m})$
  \STATE {\bfseries Init:} Pairwise graph embedding $\tU \in \R^{d \times N}, ~d \in \N_+$
   %\FOR{$i = 1,2, \cdots n$}
   \STATE Compute preference scores $\f^* = \tU^\top \w^*$ using \eqref{eq:svm1}
   \STATE Count number of wins for each node $i \in V$
   \STATE $c(i):= \hspace{-10pt}\underset{\{k=(i_k,j_k) | i_k = i\}}{\sum}\1 \big( f^*_{k}>0 \big)  + \hspace{-10pt}\underset{\{k =(i_k,j_k)| j_k = i\}}{\sum}\1 \big( f^*_{k}<0 \big)$
   \STATE Return ranking of nodes $\hat{\bsig}_n \in \text{argsort}(\c)$
\end{algorithmic}
\end{algorithm}
\vspace{-15pt}
\end{center}

% ----------------- Gen-Err ---------------------

\subsection{Generalization Error of \alg}
\label{sec:gen_err}
We now derive generalization guarantees of \alg~ (Sec. \ref{sec:alg}) on its test error ${er}^{\ell^{\rho}}_{\bS_m}(\f^*) = \frac{1}{N-m}\sum_{k = m+1}^{N}\ell^{\rho}(y_k,f^*_k)$, \emph{w.r.t.} some loss function $\ell^{\rho}: \{\pm1\}\times\R \mapsto \R_{+}$, where $\ell^{\rho}$ is assumed to be $\rho$-lipschitz ($\rho >0$) with respect to its second argument i.e. $|\ell^{\rho}(y_k,f_k) - \ell^{\rho}(y_k,f'_k)| \le \frac{1}{\rho}|f_k - f'_k|$,
%\vspace*{-4pt}
% \[ \vspace*{-4pt} |\ell^{\rho}(y_k,f_k) - \ell^{\rho}(y_k,f'_k)| \le \frac{1}{\rho}|f_k - f'_k|, \vspace*{-4pt} \]
where $\f, \f': \cP_n \mapsto \R$ be any two pairwise score functions. 
%output by \alg~ upon pairwise score estimation; e.g. note that when $\rho = 1$ for raml loss $\ell^{ramp}$. 
We find it convenient to define the following function class complexity measure associated with orthonormal embeddings of pairwise preference strong product of graphs (as motivated in \cite{rademacher}):
%the transductive Rademacher complexity associated to any function class $\cH$ as follows:

%{\color{blue} The following assumes $\cH$ is countably finite, using $|\cH|$, can we define $\cH = \{h|h: \R^{d} \mapsto \R\}$ instead?}

\begin{defn}[\textbf{Transductive Rademacher Complexity}] Given a graph $G(V,E)$, let $\tU \in \R^{d \times N}$ be any pairwise embedding of $G$ and let $col(\tU)$ denote the column space spanned by $~\tU$. Then for any function class $\cH_{\tU} = \{\h \mid \h: col(\tU) \mapsto \R \}$ associated with $~\tU$, its transductive Rademacher complexity is defined as
\[
R(\cH_{\tU},\tU,p) = \frac{1}{N} \E_{\bgam}\left[ \sup_{\h \in \cH_{\tU}} \sum_{k = 1}^{N}\gamma_k \h(\tu_{k}) \right],
\]
where for any fixed $p \in (0, 1/2]$, $\bgam = (\gamma_1, \ldots, \gamma_{N} )$ is a vector of \emph{i.i.d.} random variables such that $\gamma_i \sim \{+1, -1, 0\}$ with probability $p$, $p$ and $1 - 2p$ respectively.
\end{defn}

%\begin{restatable}[{{Rademacher Complexity and Eigen Values}}]{thm}{ThmRadEig}
%\label{thm:rad_eig}
%Given a graph $G(V,E)$, let $\tU \in \R^{d \times N}$ is a pairwise embedding of $G$, $\tK = \tU^{\top}\tU$, $p \in [0,1]$. Then the Rademacher complexity of the function class $\cH_{\tU} = \{\h ~|~ \h = \tU\bbeta, ~\bbeta \in R^{N},~\|\bbeta\|_\infty \le C, \, C > 0\}$ is
%\[
%R(\cH_{\tU}, \tU, p) \le  C\sqrt{\frac{2p\lambda_1(\tK)(tr(\tK))}{N}},
%\]
%where $tr(\tK) = \sum_{k = 1}^{N}\tK_{kk}$ denotes the trace of $\tK$. 
%\end{restatable}

We bound the generalization error of \alg~ in terms of the rademacher complexity. 
Note the result below crucially depends on the fact that any score vector $\f^*$ returned by \alg, is of the form $\f^* = \tU^{\top}\w^*$, for some  $\w^* \in \{ \h \mid \h = \tU\bbeta, \bbeta \in \R^N, \|\bbeta\|_\infty \le C \}$, where $\tU \in \R^{d \times N}$ be the embedding used in \alg~  (refer \eqref{eq:svm1}, \eqref{eq:svm3} for details).
% Based on this, we now derive the generalization guarantee for \alg~ as follows:

\begin{restatable}[\textbf{Generalization Error of \alg}]{thm}{ThmGenErr}
\label{thm:gen_err}
Given a graph $G(V,E)$, let $\tU \in \R^{d \times N}$ be any pairwise embedding of $G$. For any $f \in (0,1/2]$, let $\Pi_f$ be a uniform distribution on the random draw of $m(f) = \lceil Nf \rceil$ pairs of nodes from $\cP_n$, such that  $S_{m(f)} = \{(i_k,j_k) \in \cP_n\}_{k = 1}^{m(f)} \sim \Pi_f$, with corresponding pairwise preference $\y_{S_{m(f)}}$. Let $\bS_{m(f)} =  \cP_n \backslash S_{m(f)}$. Let $\cH_{\tU} = \{\w ~|~ \w = \tU\bbeta, ~\bbeta \in \R^{N}, ~\|\bbeta\|_{\infty} \le C, ~C >0\}$ and $\ell^{\rho}: \{\pm1\}\times\R \mapsto [0,B]$ be a bounded, $\rho$-Lipschitz loss function. For any $\delta > 0$, with probability $\ge 1-\delta$ over $S_{m(f)} \sim \Pi_f$

\vspace*{-20pt} 
\begin{align*}
%\label{eq:generr1}
er_{\bS_{m(f)}}^{\ell^{\rho}}(\f^*) \le er_{S_{m(f)}}^{\ell^{\rho}}(\f^*) + \frac{R(\cH_{\tU},\tU,p)}{\rho {f(1-f)}}
 + \frac{C_1B \sqrt{\ln\left( \frac{1}{\delta} \right)}}{(1-f)\sqrt{Nf}},
\end{align*}
\vspace*{-5pt} 

where $p = f(1-f)$ and $\f^* = \tU^\top\w^* \in \R^N$ is pairwise score vector output by \alg~ and $C_1 > 0$ is a constant.
%, i.e. $\exists \w \in \cH_{\tU}$ such that $\f^*  = $. $C_1 > 0$, $\delta \in [0,1]$ are constants. %$R(\cH_{\tU}) = $ being the transductive Rademacher complexity of the function class $\cH_{\tU}$ with parameter.
\iffalse%%%%%%%%%%%%%%%%%%%%%%%%%%%%%%%
Moreover, we can further bound $R(\cH_{\tU},\tU,p)$ as:
\[
R(\cH_{\tU}, \tU, p) \le  C\sqrt{\frac{2p\lambda_1(\tK)(tr(\tK))}{N}},
\]
for any $p \in [0,1]$, $tr(\tK) = \sum_{k = 1}^{N}\tK_{kk}$ and $\lambda_1(\tK)$ being the trace and maximum eigenvalue  of $\tK$ respectively. 
\fi%%%%%%%%%%%%%%%%%%%%%%%%%%%%%%%%
\end{restatable}
%\frac{C\lambda_1(\K)}{\rho\sqrt{f(1-f)}}

%{\color{red}Proof needs to be changed accordingly}

\begin{rmk}
\textup{
It might appear from above that a higher value of $R(\cH_{\tU},\tU,p)$ leads to increased generalization error. However, note that there is a {\it tradeoff} between the first and second term since a higher rademacher complexity implies a richer function class $\cH_{\tU}$, which in turn is capable of producing a better prediction estimate $\f^*  = \tU^\top\w$, resulting in a much lower training set error $er_{S_{m(f)}}^{\ell^{\rho}}[\f^*]$. Thus, a \textit{higher value of $R(\cH_{\tU})$ is desired} better generalization performance.
% From \eqref{eq:generr1}, it might appear that a higher value of $R(\cH_{\tU},\tU,p)$ leads to increased generalization error $er_{\bS_{m(f)}}^{\ell^{\rho}}[\f^*]$. However, note that there is a {\it tradeoff} between the first and second term since a higher rademacher complexity $R(\cH_{\tU}, \tU,p)$ implies a richer function class $\cH_{\tU}$, which in turn is capable of producing a better prediction estimate $\f^*  = \tU^\top\w$, resulting a much lower training set error $er_{S_{m(f)}}^{\ell^{\rho}}[\f^*]$. Thus a \textit{higher value of $R(\cH_{\tU})$ is desired} better generalization performance.
}
\end{rmk}

Taking insights from Thm. \ref{thm:gen_err}, it follows that the performance of \alg~ crucially depends on the {\it rademacher complexity} $R(\cH_{\tU}, \tU,p)$ of the underlying function class $\cH_{\tU}$, which boils down to the problem of finding a ``good'' embedding $\tU$. We address this issue in the next section. %-- analyse different choices of graph embeddings along with their {\it rademacher complexities} that leads to optimal generalization guarantees, and eventually to \emph{ranking consistency}.

% ---------------- Embedding ------------------

\section{Choice of Embeddings}
\label{sec:embedding}

We discuss different classes of pairwise graph embeddings and their generalization guarantees. 
Recalling the results of \cite{AndoZh07} (see Thm. $1$), which provides a crucial characterization of the class of optimal embeddings for any graph based regularization algorithms,
we choose to work with embeddings with normalized kernels, i.e. $\tK = \tU^{\top}\tU$ such that $\tilde K_{kk} = 1, \forall k \in [N]$. 
%(detailed analysis in Cor. \ref{thm:opt_embed}, Appendix \ref{app:char_opt_embed}).
The following theorem analyses the {rademacher complexity}  of `normalized' embeddings:

\begin{restatable}[\textbf{Rademacher Complexity of Orthonormal Embeddings}]{thm}{ThmRadEig}
\label{thm:rad_eig}
Given $G(V,E)$, let $\tU \in \R^{d \times N}$ be any `normalized' node-pair embedding of $\strPr$, let $\tK = \tU^{\top}\tU$ be the corresponding graph-kernel, then % such that $\tilde K_{kk} = 1, ~\forall k \in [N]$ % Then the Rademacher complexity of the function class $\cH_{\tU} = \{\h ~|~ \h = \tU\bbeta, ~\bbeta \in R^{N},~\|\bbeta\|_\infty \le C, \, C > 0\}$ is
$R(\cH_{\tU}, \tU, p) \le  C\sqrt{{2p\lambda_1(\tK)}}$, where $\lambda_1(\tK)$ is the largest eigenvalue of $~\tK$. 
\end{restatable}

Note that the above result does not educate us on the choice of $\tU$ -- we impose more structural constraints and narrow down the search space of optimal `normalized' graph embeddings and propose the following special classes:
%We further narrowed down the search space of optimal graph embeddings to {\it Orthonormal Representations} of the strong product graph $\strPr$ and derived the generalization guarantees for the following choices of graph embeddings:

%\vspace*{-2pt}
\subsection{\spLab: Kronecker Product Orthogonal Embedding}
\label{sec:embed_strpr}
Given any graph $G(V,E)$, with $\U = [\u_1, \u_2, \ldots \u_n] \in \R^{d \times n}$ being an orthogonal embedding of $G$, i.e. $\U \in \text{Lab}(G)$, its Kronecker Product Orthogonal Embedding:
\begin{align*}
\text{\spLab}  := \{&\tU \in \R^{d^2 \times n^2} \mid {\tU} = \U \otimes \U, \\ &\U \in \R^{d \times n} \mbox{ such that } \U \in \mbox{Lab}(G)\},
\end{align*} 
where $\otimes$ is the kronecker (or outer) product of two matrix.
%as defined in Sec. \ref{sec:prelims}. 
The \emph{`niceness'} of the above embedding lies in the fact that one can construct $\tU \in$ ~\spLab\, from any orthogonal embedding of the original graph $\U \in$ Lab$(G)$ -- let $\K := \U^\top\U$ and $\tK := \tU^{\top}\tU$, we see that for any two $k,k' \in [n^2]$, $\tK_{kk'} = \tu_{k}^{\top}\tu_k' = (\u_{i_k} \otimes \u_{j_k})^{\top}(\u_{i_{k'}} \otimes \u_{j_{k'}}) = (\u_{i_k}^{\top}\u_{i_{k'}})(\u_{j_k}^{\top}\u_{j_{k'}}) = \K_{i_ki_{k'}}\K_{j_kj_{k'}}$, where $(i_{(\cdot)},j_{(\cdot)}) \in [n]\times[n]$ are the node pairs corresponding to $k,k'$. Hence $\tK = \K \otimes \K$. Note that when $k = k'$, we have $\tK_{kk} = 1$, as $\U \in$ Lab$(G)$, $K_{ii} = 1, \forall i \in [n]$. This ensures that the kronecker product graph kernel $\tK$ satisfies the optimality criterion of `normalized' embedding as previously discussed. 
%We now analyse the Rademacher complexity of the function class associated with \spLab.

%Now consider the function following classes associated to $\tU \in$ \spLab and $\tU_P \in$ PSP-Lab$(G)$:
%
%$\cH_{\tU} = \{\h ~|~ \h = \tU\bbeta, \bbeta \in \R^{n^2} \}$, where $\tU \in$ SP-Lab$(G)$
%and 
%$\cH_{\tU_P} = \{\h ~|~ \h = \tU\bbeta, \bbeta \in \R^{n^2} \}$, where $\tU_P \in$ PSP-Lab$(G)$. 

\begin{restatable}[\textbf{Rademacher Complexity of \spLab}]{lem}{ThmRadsp}
\label{thm:rad}
Consider any $\U \in$ Lab(G), $\K = \U^{\top}\U$ and the corresponding $\tU \in$ {\spLab}. 
%Then if $\tK = \tU^{\top}\tU$ then
%\[
%\lambda_1(\tK) \le \lambda_1^2(\tK).
%\]
%Using this we further get that 
Then for any $p \in [0, 1]$ and $\cH_{\tU} = \{\w ~|~ \w = \tU\bbeta, ~\bbeta \in R^{N},~\|\bbeta\|_\infty \le C, \, C > 0\}$ we have, %Then %the Rademacher complexity of $\cH_{\tU}$ is
$R(\cH_{\tU}, \tU, p) \le C\lambda_1(\K)\sqrt{2p}$.
\end{restatable}

%{\color{blue} Thm.~\ref{thm:gen_err} works with $\cH_{\tU} = \{\f ~|~ \f = \tU\bbeta, ~\bbeta \in \R^{N}, ~\|\bbeta\|_{\infty} \le C, ~C >0\}$. It is good to be consistent.}

Above leads to the following generalization guarantee:

\begin{restatable}[\textbf{Generalization Error of \alg~ with \spLab}]{thm}{CorRadsp}
\label{cor:rad}
\vspace*{-4pt}
For the setting as in Thm.~\ref{thm:gen_err} and Lem.~\ref{thm:rad}, for any $\tU \in$ \spLab, we have
\begin{equation*}
er_{\bS}^{\ell^{\rho}}[\f^*] \le er_{S}^{\ell^{\rho}}[\f^*] + \frac{C\lambda_{1}(\K)\sqrt{2}}{\rho \sqrt{f(1-f)}} + \frac{C_1B}{1-f}\sqrt{\frac{\log (\frac{1}{\delta})}{Nf}}
\vspace*{-4pt}
\end{equation*}
%where $C_1>0$ is a constant.
\end{restatable}

%\vspace*{-2pt}
\subsection{Pairwise Difference Orthogonal Embedding}
\label{sec:pd_label}

Given any graph $G(V,E)$, let $\U = [\u_1, \u_2, \ldots \u_n] \in \R^{d \times n}$ be such that $\U \in \text{Lab}(G)$. We define the class of {\it Pairwise Difference Orthogonal Embedding} of $G$ as:
\begin{align*}
\text{\pdLab} := \{& \tU \in \R^{d \times N} \mid {\tu}_{ij} = \u_i - \u_j ~\forall (i,j) \in \cP_n, \\
& \U \in \R^{d \times n} \mbox{ such that } \U \in \mbox{Lab}(G)\}
\end{align*} 
Let $\E = [\e_i-\e_j]_{(i,j)\in \cP_n} \in \{0,\pm 1\}^{n \times N}$, where $\e_i$ denotes the $i^{th}$ standard basis of $\R^n$, $\forall i \in [n]$; then it is easy to note that $\tU = \U\E \in $ \pdLab\, and the corresponding graph kernel is given by $\tK = \E^\top\K\E$. For PD embedding, we get: 
% Moreover, defining $\K = \U^\top\U$ and $\tK = \tU^{\top}\tU$, we have that for any two $k,k' \in [n^2]$, $\tK_{kk'} = \tu_{k}^{\top}\tu_k' = (\u_{i_k} \otimes \u_{j_k})^{\top}(\u_{i_k'} \otimes \u_{j_k'}) = (\u_{i_k}^{\top}\u_{i_k'})(\u_{j_k}^{\top}\u_{j_k'}) = \K_{i_k,i_k'}\K_{j_k,j_k'}$, $(i_k,j_k) \in [n]\times[n]$ being the node pair associated to $k$. Hence $\tK = \K \otimes \K$. In particular when $k = k'$, we have $\tK_{kk} = 1$, as $U \in$ Lab$(G)$, $K_{ii} = 1, \forall i \in [n]$. This thus ensures that $\tK$ indeed satisfies the optimality criterion of the embeddings as inferred in Cor. \ref{thm:opt_embed}. 
%Now consider the function following classes associated to 

\begin{restatable}[\textbf{Rademacher Complexity of \pdLab}]{lem}{ThmRadpd}
\label{thm:radpd}
Consider any $\U \in \text{Lab}(G)$, $\K = \U^{\top}\U$ and the corresponding $\tU \in$ {\pdLab}. 
%Then if $\tK = \tU^\top\tU$, then
%\[
%\lambda_1(\tK) \le 2n\lambda_1(\K).
%\]
%%Let 
Then for any $p \in [0, 1]$ and $\cH_{\tU} = \{\w ~|~ \w = \tU\bbeta, ~\bbeta \in R^{N},~\|\bbeta\|_2 \le tC{\sqrt{N}}, \, C > 0\}$, we have %. Then %the Rademacher complexity of $\cH_{\tU}$ is
$R(\cH_{\tU}, \tU, p) \le 2C\sqrt{p n \lambda_1(\K)}$.
\end{restatable}

Similarly as before, using above we can show that:

\begin{restatable}[\textbf{Generalization Error of \alg~ with \pdLab}]{thm}{CorRadpd}
\label{cor:radpd}
For the setting as in Thm.~\ref{thm:gen_err} and Lem.~\ref{thm:radpd}, for any $\tU \in$ {\pdLab}, we have
\begin{equation*}
er_{\bS}^{\ell^{\rho}}[\f^*] \le er_{S}^{\ell^{\rho}}[\f^*] + \frac{2C\sqrt{n\lambda_{1}(\K)}}{\rho \sqrt{f(1-f)}} + \frac{C_1B}{1-f}\sqrt{\frac{\log (\frac{1}{\delta})}{Nf}}
\end{equation*}
%where $C_1>0$ is a constant.
\end{restatable}

Recall from Thm \ref{thm:gen_err} that $\f^* = \tU^\top\w$. Thus the \emph{`niceness'} of \pdLab\ lies in the fact that it comes with the free transitivity property -- for any two node pairs $k_1: = (i,j)$ and $k_2:=(j,l)$, if $\f^*$ scores node $i$ higher than $j$ i.e. $f^*_{k_1} > 0$, and node $j$ higher than node $l$ i.e. $f^*_{k_2} > 0$; then for any three nodes $i,j,l \in [n]$, this automatically implies $f^*_{k_3} > 0$, where $k_3: = (i,l)$ i.e. node $i$ gets a score higher than node $l$.% -- this respects the inherent ordering induced by any underlying ranking $\bsig^*_n$.% as in any ranking $i$ precedes $j$ ($\sigma^*_n(i)<\sigma^*_n(j)$), and $j$ precedes $l$ ($\sigma^*_n(j)<\sigma^*_n(l)$) implies, $i$ precedes $l$ ($\sigma^*_n(i)<\sigma^*_n(l)$) as well.

\begin{rmk}
\textup{
Although Lem.~\ref{thm:rad} and~\ref{thm:radpd} shows that both \spLab\ and \pdLab\ are associated to rich expressive function classes with high rademacher complexity, \textit{the superiority of \spLab\ comes with an additional consistency guarantee}, as we will derive in Sec. \ref{sec:consistency}.}%) which \pdLab lacks as of now. though might be an interesting direction to explore.
\end{rmk}

\subsection{\ls\ based Embedding}
\label{sec:ls_label}
The embedding (graph kernel) corresponding to \ls~\cite{luz} of graph $G$ is given by:
\begin{equation}
\label{eq:ls_labelling}
\K_{LS}(G)=\frac{\A_G}{\tau} + \I_n, ~\text{where}~ \tau \ge |\lambda_{n}(\A_G)|,
\end{equation}
where $\A_G$ is the adjacency matrix of graph $G$. It is known that $\K_{LS} \in \R^{n \times n}$ is symmetric and positive semi-definite, and hence defines a valid graph kernel; also $\exists \U_{LS} \in  \text{Lab}(G)$  such that $\U_{LS}^\top \U_{LS} = \K_{LS}$. We denote $\U_{LS}$ to be the corresponding embedding matrix for \ls.
% Clearly in this case the embedding dimension $d = n$. Note that $\U_{LS} \in \text{Lab}(G)$.
We define \ls\ of the strong product of graphs as:
\begin{equation}
\label{eq:strpr_ls}
\tK_{LS}(\strPr) = \K_{LS}(G)\otimes \K_{LS}(G)
\end{equation}
%Note that though $\tK_{LS}(\strPr) \neq \K_{LS}(\strPr)$ but the following lemma shows that $\tK_{LS}(\strPr)$ gives a good approximation to $\K_{LS}(\strPr)$.
and equivalently the embedding matrix $\tU_{LS}(\strPr) = \U_{LS}(G)\otimes \U_{LS}(G)$. Similar to \spLab, we have $\tK_{LS}(k,k) = 1, ~\forall k \in [n^2]$, since $\K_{LS}(i,i) = 1, ~\forall i \in [n]$. 
Following result shows that $\tK_{LS}(\strPr)$ has high Rademacher complexity on random $G(n,q)$ graphs. %including $G(n,q)$ random graphs. 

\begin{restatable}{lem}{LemRadLSApprox}
\label{lem:sp_lsrad} 
Let $G(n,q)$ be a Erd\'os-R\'eyni random graph, where each edge is present independently with probability $q \in [0,1],~q=O(1)$. Then the Rademacher complexity of function class associated with $\tK_{LS}(\strPr)$ is $O(\sqrt n)$. 
\end{restatable}

%Now note that $\A_{\strPr} = (\A_{G}+\I_n)\otimes (\A_G + \I_n) - \I_{n^2}$.  Clearly, $\lambda_{n}(\A_{G}+\I_n) = \lambda_{n}(\A_G)+1$. Now since for any matrix $\M \in \R^{n\times n}$ with eigenvalues $\lambda_1(\M), \ldots \lambda_n(\M)$, the eigenvalues of $\M \otimes \M$ are of the form $\lambda_{i}(\M)\lambda_{j}(\M) ~\forall i,j \in [n]$, we have $\lambda_n(\A_{\strPr}) = (\lambda_{n}(\A_G)+1)^2 - 1$. Thus \ls of the strong product graph $\strPr$ is given by
%\begin{equation} \label{eq:ls_labelling2}\K_{LS}(\strPr) = \frac{\A_{\strPr}}{\tau'} + \I_{n^2},\; \tau' \ge (\lambda_{n}(\A_G)+1)^2 -1. \end{equation}

\noindent \textbf{Laplacian based Embedding.}
\label{app:lap}
This is the most popular choice of graph embedding that uses the inverse of the Laplacian matrix for the purpose. Formally, let $d_i$ denotes the degree of vertex $i\in [n]$ in graph $G$, i.e. $d_i = {(\A_G)}_{i}^\top\1_n$, and $\D$ denote a diagonal matrix such that $D_{ii}=d_i,\forall i\in [n]$. %We refer $\D-\A_G$ and $\I-\D^{-\frac{1}{2}}\A_G\D^{-\frac{1}{2}}$ to be the Laplacian and normalized Laplacian matrix of the $G$ \cite{AndoZh07} respectively. 
Then the Laplacian and normalized Laplacian kernel of $G$ is defined as follows: $\K_{Lap}(G)= (\D - \A_G)^{\dagger}$ and $\K_{nLap}(G) = (\I_n - \D^{-1/2}\A_G\D^{-1/2})^{\dagger}$
%\begin{align*}
%& \K_{Lap}(G)= (\D - \A_G)^{\dagger} \text{ and}\\
%& \K_{nLap}(G) = (\I_n - \D^{-1/2}\A_G\D^{-1/2})^{\dagger}.
%\end{align*}
\footnote{$\dagger$ denotes the pseudo inverse.}. 

Though widely used \cite{Agarwal10,AndoZh07}, it is not very expressive on dense graphs with high $\chi(G)$ -- we observe that the Rademacher complexity of function class associated with Laplacian is an order magnitude smaller than that of \ls. See App. \ref{app:lap} for details.

% ---------------- Consistency ----------------

\newcommand{\ty}{\tilde y}
\newcommand{\tby}{{\bf\tilde y}}

\section{Consistency with \spLab}
\label{sec:consistency}
%\todoa{To be refined.}
In this section, we show that \alg~ is provably statistically consistent while working with {\it kronecker product orthogonal embedding} \spLab (see Sec. \ref{sec:embed_strpr}).

\begin{restatable}[\textbf{Rank-Consistency}]{thm}{ThmConst}
\label{thm:const}
%Recall the definition of statistical consistency from Sec. \ref{sec:prb_st}.
%For any $f \in (0,1/2]$, let $\Pi_f$ be the uniform distribution on random draw of $m(f)$ pairs of nodes out of $N$, such that  $S_{m(f)} = \{(i_k,j_k) \in \cP\}_{k = 1}^{m}, ~m(f) = \lceil Nf \rceil$, and suppose $\y_{S_{m(f)}}$ denotes the pairwise preference information of $S_{m(f)}$. 
For the setting as in Sec. \ref{sec:obj}, there exists an embedding $\tU_n \in$ Kron-Lab($G_n \boxtimes G_n$) such that if $\bsig_n \in \R^{N}$ denotes the pairwise scores returned by \alg~ on input $(\tU_n,S_m(f),\y_{S_{m(f)}})$, then $\forall G_n \in \cG$, with probability at least $\Big( 1 - \frac{1}{N} \Big)$ over $S_{m(f)} \sim \Pi_f$
 \begin{equation*}
d(\bsig^*_{n},\bhsig_{n}) = O\Bigg( \bigg( \frac{\vartheta(G_n)}{nf} \sqrt{\frac{1-f}{f}} \bigg)^{\frac{1}{2}} + \sqrt{\frac{\ln n}{Nf}} \Bigg),
\end{equation*}
where $d$ denotes Kendall's tau $(d_k)$ or Spearman's footrule $(d_s)$ ranking loss functions.

%\begin{equation*}
%er_{n}^{\ell^{0-1}}[\f_n] \le \Bigg( 2\bigg( \frac{\vartheta(G_n)}{(n-1)f} \sqrt{\frac{2(1-f)}{f}} \bigg)^{\frac{1}{2}} + \frac{C_1\sqrt{2\ln n}}{\sqrt{Nf}} \Bigg).
%\end{equation*}
%This further implies that if $\bhsig_n$ is the predicted ranking returned by the algorithm, then $\forall G_n \in \cG$, with probability at least $\Big( 1 - \frac{1}{N}\Big)$ over $S_{m(f)} \sim \Pi_f$,

\end{restatable}

Consistency follows from the fact that for large families of graphs including random graphs  \cite{coja} and power law graphs \cite{Jethava+13}, $\vartheta(G_n) = o(n)$.

%Thm. \ref{thm:const} implies that for any graph family  $\cG$, for which $\vartheta(G_n) = o(n)$, e.g. random graphs  \cite{coja} or power law graphs \cite{Jethava+13}, consistency follows.%er_{n}^{\ell^{0-1}}[\f]
%\vspac*{-2pt}
\subsection{Sample Complexity for Ranking Consistency}
\label{sec:sample_complexity}

We analyze the minimum fraction of pairwise node preferences $f^*$ to be observed for \alg\ algorithm to be statistically ranking consistent. 
% In other words, under what $f = f^* \in [0,1]$, $Pr_{S_{m(f^*)} \sim \Pi_{f^*}}\left( d(\bsig^*_n, \hat{\bsig}_n) \ge \epsilon \right) \rightarrow 0\quad{ as }\quad n  \rightarrow \infty$ (Recall definition of \emph{ranking consitency} from Sec. \ref{sec:obj}).
We refer the required sample size $m(f^*) = \lceil Nf^* \rceil$ as \emph{ranking sample complexity}.

\begin{restatable}{lem}{sampcomp}
%[\textbf{Strong Product of SVM-$\vartheta$ graphs}]
\label{cor:sam_com}
%For the class of graph families $\cG$, where $\vartheta(G_n)=o(n),~\forall G_n \in \cG$, then 
%Consider the setting of Thm. \ref{thm:const}, 
If $\G$ in Thm.~\ref{thm:const} is such that $\vartheta(G_n) = n^c$, $0\le c < 1$.
%where $G_n$ be such that $\vartheta(G_n) = n^c$ some for $c \in [0,1)$.
Then observing only $f^* = O\bigg( \frac{\sqrt {\vartheta(G_n)}}{n^{\frac{1}{2}-\varepsilon}} \bigg)^{\frac{4}{3}}$ fraction   of 
%$\Omega((n^2\vartheta(G_n))^{\frac{2}{3}})$ many 
pairwise node preferences is sufficient for \alg~ to be statistically rank consistent, for any $0 < \varepsilon < \frac{(1-c)}{2}$. %in Kendall's tau or Spearman's footrule loss.
\end{restatable}

%The proof immediately follows from Thm. \ref{thm:const} with $f^* = \bigg( \frac{\sqrt{\vartheta(G_n)}}{n^{\frac{1}{2}} - \varepsilon} \bigg)^{\frac{4}{3}}$ for any $\varepsilon > 0$, leading to a required sample complexity of 
%\[
%Nf^* = N\bigg( \frac{\sqrt{\vartheta(G_n)}}{n^{\frac{1}{2}} - \varepsilon} \bigg)^{\frac{4}{3}}  = O((n^{2}\vartheta(G_n))^{\frac{2}{3}}). 
%\]
%(See Appendix for the entire proof).
%Rakesh --- $f=\frac{\tilde{O}(\vartheta(G)n)}{2N}$ 

%where $d$ being the Kendall's tau $(d_k)$ or Spearman's footrule $(d_s)$ ranking loss. 
%Hence the required sample consistency to achieve rank consistency turns out to be 

Note that one could potentially choose any $\varepsilon \in (0, \frac{1-c}{2})$ for the purpose -- the \textit{tradeoff} lies in the fact that a higher $\varepsilon$ leads to \textit{faster convergence rate of $d(\bsig^*_{n},\bhsig_{n}) = O(\frac{1}{n^\varepsilon})$}, although at the \textit{cost of increased sample complexity}; on the contrary setting $\varepsilon \to 0$ gives a smaller sample complexity, with significantly slower convergence rate (see proof of Lem. \ref{cor:sam_com} in App. for details). 
We further extend Lem. \ref{cor:sam_com} and relate ranking sample complexity to structural properties of the graph -- \emph{coloring number} of the complement graph $\chi{(\bar{G})}$.
% Lem. \ref{cor:sam_com} further implies an upper bound on the required sample complexity in terms on \emph{coloring number} of the complement graph $\chi{(\bar{G})}$. More formally,
\begin{restatable}{thm}{sampcompclr}
%[\textbf{Strong Product of SVM-$\vartheta$ graphs}]
\label{thm:sam_com_clr}
Consider a graph family $\G$ such that $\chi(\bar G_n) = o(n)$, $\forall G_n \in \G$. Then observing $O(n^2\chi(\bar{G}))^{\frac{2}{3}}$ pairwise preferences is sufficient for \alg~ to be consistent.
\end{restatable}

%{\color{red} Doesn't this imply $\varepsilon = 0$ and not lead to convergence ? Did we mean to use $\omega(n^2\chi(\bar{G}))^{2/3}$ above ?}

Above conveys that for \textit{dense graphs we need fewer pairwise samples compared to sparse graphs} as $\chi(\bar{G})$ reduces with increasing graph density. We discuss the sample complexities for some special graphs below where $\vartheta(G) = o(n)$.

\begin{restatable}[\textbf{Ranking Consistency on Special Graphs}]{cor}{sampcompspcl}
%[\textbf{Strong Product of SVM-$\vartheta$ graphs}]
\label{cor:sam_com_spcl}
\alg\ algorithm achieves consistency on the following graph families, with the required sample complexities --
%The sample complexity of \alg\, to achieve consistency on the the following special graphs are:
%\begin{enumerate}
$(a)$~\emph{Complete graphs}: $O(n^{\frac{4}{3}})$
$(b)$~\emph{Union of $k$ disjoint cliques}: $O(n^{\frac{4}{3}}k^{\frac{2}{3}})$
$(c)$~\emph{Complement of power-law graphs}: $O(n^{\frac{5}{3}})$
$(d)$~\emph{Complement of $k$-colorable graphs}: $O(n^{\frac{4}{3}}k^{\frac{2}{3}})$
$(e)$~\emph{Erd\H{o}s R\'eyni random $G(n,q)$ graphs} with $q = O(1)$: $O(n^{\frac{5}{3}})$.
%$(f)$~\emph{Disjoint graphs}: $O(n^{2})$
%$(g)$~\emph{Complement of Planar graphs}: $O(n^{\frac{4}{3}})$
%\end{enumerate}
\end{restatable}

%\cite{chi_power_law}%\cite{coja}

%Clearly, for graphs with $\chi(\bar{G})=O(1)$ observing only $O(n^{\frac{4}{3}})$ node preferences suffices to achieve consistency e.g. union of $k$-cliques, complement of $k$-colorable graphs (for some constant $k \le n$), planar graphs\footnote{By the \emph{Four-Color Theorem}, any planar graph can be colored with at most $4$ colors in polynomial time \cite{color4}} etc. 
%Even for graphs with $\vartheta(G) = \Theta(\sqrt n)$ e.g. mixture of random $G(n, q)$ graphs with constant edge probability $q=O(1)$ \cite{coja}, power-law graphs \cite{chi_power_law} etc., consistency follows without the knowledge of all $n \choose 2$ pairwise preferences.

%For bipartite ranking on mixture of random $G(n,p,q)$ graphs, where $V=[n],~p>q $ and labels $y_i = +1$ if $i \le n/2$; $-1$ otherwise. An edge $(i, j)$ is present w.p. $p$ if $y_i = y_j$; $q$ othe rwise. Then for $p,q=O(1)$, $\vartheta(G(n,p,q)) = O(\sqrt{n})$ \cite{coja}. %, hence observing $\tilde{O}(n\sqrt{n})$ is sufficient.

%Such novel insights connecting ranking sample complexity to graph structural properties was lacking previously.

\begin{rmk}
%{\color{red} The relation between node preferences and the graph structure.} 
\textup{Thm. \ref{thm:const} along with Lem. \ref{cor:sam_com} suggest that if the graph satisfies a crucial structural property: $\vartheta(G) = o(n)$ and given sufficient sample of $\Omega(n^2\vartheta(G))^{\frac{2}{3}}$ 
%{\color{red} ToDo (Aadirupa): Motivate how $K_n$ relates to optimal ranking -- perfect pairwise classification  ?}
pairwise preferences, \alg\ yields consistency. %under the optimal embedding $\tU\in\text{Kron-Lab}(\strPr)$. 
%Note that if we \textit{reorder} the node rankings, the optimal embedding $\tU_n$ changes accordingly, yet it belongs to \spLab, as our embedding class \spLab\ is a \textit{universal function class}, which could potentially learn any pairwise preference relation -- it is rich enough to always contain an optimal $\tU_n$ that leads to rank consistency. 
%Moreover, the required sample size for this to hold is O(theta(G)n) (Cor 10), which directly depends on the graph structural property via theta(G) (Sec 7.2). Thus our analysis is very much adapted to the graph structure. 
Note that $\vartheta(G) \le \chi(\bar G) \le n$, where the last inequality is tight for completely disconnected graph -- which implies one need to observe $\Omega(n^2)$ pairs for consistency, as a disconnected graph does not impose any structure on the ranking. % ({\color{red} see -- ?}).
Smaller the $\vartheta(G)$, denser the graph and we attain consistency observing a smaller number of node pairs, the best is of course when $G$ is a clique, as $\vartheta(G) = 1$!
So for sparse graphs 
%\emph{(``ill structured" with very few node interactions)} 
with $\vartheta(G) = \Theta(n)$, consistency and learnability is far fetched without observing $\Omega(n^2)$ pairs. }
\end{rmk}

%The critical component of the proof
Note that proof of Thm.~\ref{thm:const} relies on the fact that the maximum SVM margin attained for the formulation~\eqref{eq:svm1} is $\vartheta(\strPr)$, which is achieved by \ls\ on Erd\H{o}s R\'eyni random graphs ~\cite{RakeshCh14}; and thus guarantee consistency, with $O(n^{\frac{5}{3}})$ sample complexity.

% ---------------- Experiments ----------------

\section{Experiments}
\label{sec:experiments}
We conducted experiments on both real world and synthetic graphs, comparing \alg\ with the following algorithms:

\vspace*{1pt}
\noindent \textbf{Algorithms.} We thus used the following $5$ algorithms:
$(a)$~\textbf{PR-Kron}: \alg\, with $\tK_{LS}(G \boxtimes G)$ (see Eqn. \eqref{eq:strpr_ls})   
$(b)$~\textbf{PR-PD}: \alg\, with \pdLab\, with \ls\, i.e. $\U = \U_{LS}$,   
$(c)$~\textbf{GR}: Graph Rank~\cite{Agarwal10},   
$(d)$~\textbf{RC}: Rank Centrality~\cite{Negahban+12} and
$(e)$~\textbf{IPR}: Inductive Pairwise Ranking, with Laplacian as feature embedding~\cite{NiranjanRa17}.   

Recall from the list of algorithms in Table \ref{tab:sum_con}. Except~\cite{Agarwal10}, none of the other applies directly to ranking on graphs. Moreover they work only under specific models -- e.g. \emph{noisy permutations} for \cite{Wauthier+13}, ~\cite{Rajkumar+16} requires the knowledge of the preference matrix rank $r$ etc. We compare with \textbf{RC} (works only under BTL model) and \textbf{IPR} (requires item features), but as expected both perform poorly. For better comparison, we present plots comparing only the initial $3$ methods in App. \ref{app:expts}.

\vspace*{1pt}
\noindent \textbf{Performance Measure.}
Note the generalization guarantee of Thm. \ref{thm:gen_err} not only holds for full ranking but for any %graph based learning problem that aims to optimize some function of the pairwise misprediction error.
%Specifically, consider a 
\textit{general preference learning problem}, where the nodes of $G$ are assigned to an underlying preference vector $\bsig^*_n \in \R^n$. %For any $k \in \P$, $$ 
Similarly, the goal is to predict a pairwise score vector $\f \in \R^{N}$ to optimize the average pairwise mispredictions \textit{w.r.t.} some loss function $\ell:\{\pm1\}\times \R\setminus\{0\}\mapsto \R_+$ defined as:
\begin{equation}
\label{eq:err_gen}
er_{D}^{\ell}(\f) = \frac{1}{|D|}\sum_{k \in D}\ell(y_k,f_k),
\end{equation}
where $D = \{(i_k,j_k) \in \cP_n \mid \sigma_n^*(i_k) \neq \sigma_n^*(j_k), \, k \in [N]\} \subseteq \cP_n$ denotes the subset of  node pairs with distinct preferences and $y_k = \mbox{sign}(\sigma_n^*(j_k) - \sigma_n^*(i_k)), ~\forall k \in D$. In particular, \alg\, applies to \textit{bipartite ranking} \textbf{(BR)}, where $\bsig^*_n \in \{\pm1\}^{n}$, \textit{categorical or $d$-class ordinal ranking} \textbf{(OR)}, where $\bsig^*_n \in [d]^n, ~d < n$, and the original \textit{full ranking} \textbf{(FR)} problem as motivated in Sec. \ref{sec:prb_st}. We consider all three tasks in our experiments with \textbf{pairwise} $0$-$1$  \textbf{loss}, i.e. $\ell(y_k,f_k) = \1(y_kf_k < 0)$. $er_{n}^{\ell^{0-1}}(\f^*)$ in Eqn. \eqref{eq:err_gen}. 
%(see experiments Sec. \ref{sec:experiments}).

%As motivated in Sec. \ref{sec:gen_err}, our proposed \alg~ algorithm applies to general class of pairwise preference graph learning problems -- Bipartite Ranking (BR), Ordinal Ranking (OR(d))\footnote{$d$ denotes the number of ordinal labels/classes.} and Full Ranking (FR). We evaluate on both synthetic and real world datasets, and show improved performance of \alg~ using $\tK_{LS}(\strPr)$ embedding (see Sec. \ref{sec:ls_label}) in comparison with \sa  ~\cite{Agarwal10}. 

%\vspace*{1pt}
%\noindent \textbf{Reported Performance.}
%All the performances are averaged across $10$ runs.

\subsection{Synthetic Experiments}
%\vspace*{-10pt}

\textbf{Graphs.}
We use $3$ \emph{types of graphs}, each with $n = 30$ nodes:
$(a)$ \emph{Union of $k$-disconnected cliques} with $k = 2\text{ and }10$, $(b)$ \emph{$r$-Regular graphs} with $r = 5\text{ and }15$; and $(c)$ $G(n, q)$ \emph{Erd\H{o}s R\'eyni random graphs} with edge probability $q = 0.2\text{ and }0.6$.

%\vspace{-5pt}
\begin{figure}[H]
%\begin{widepage}
\hspace{-18pt}
\includegraphics[trim={2.8cm 1.1cm 2.7cm 0.4cm},clip,scale=0.2,width=0.17\textwidth]{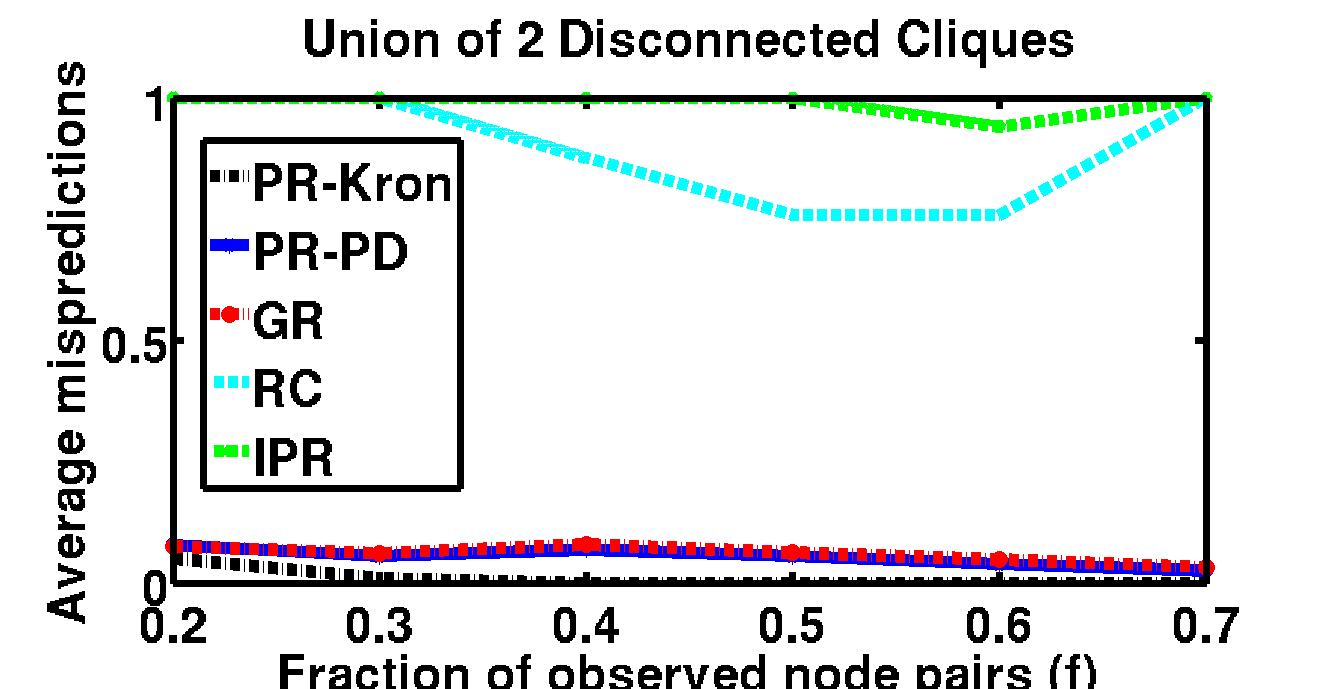}
\hspace{-6pt}
\includegraphics[trim={2.8cm 1.1cm 2.7cm 0.4cm},clip,scale=0.2,width=0.17\textwidth]{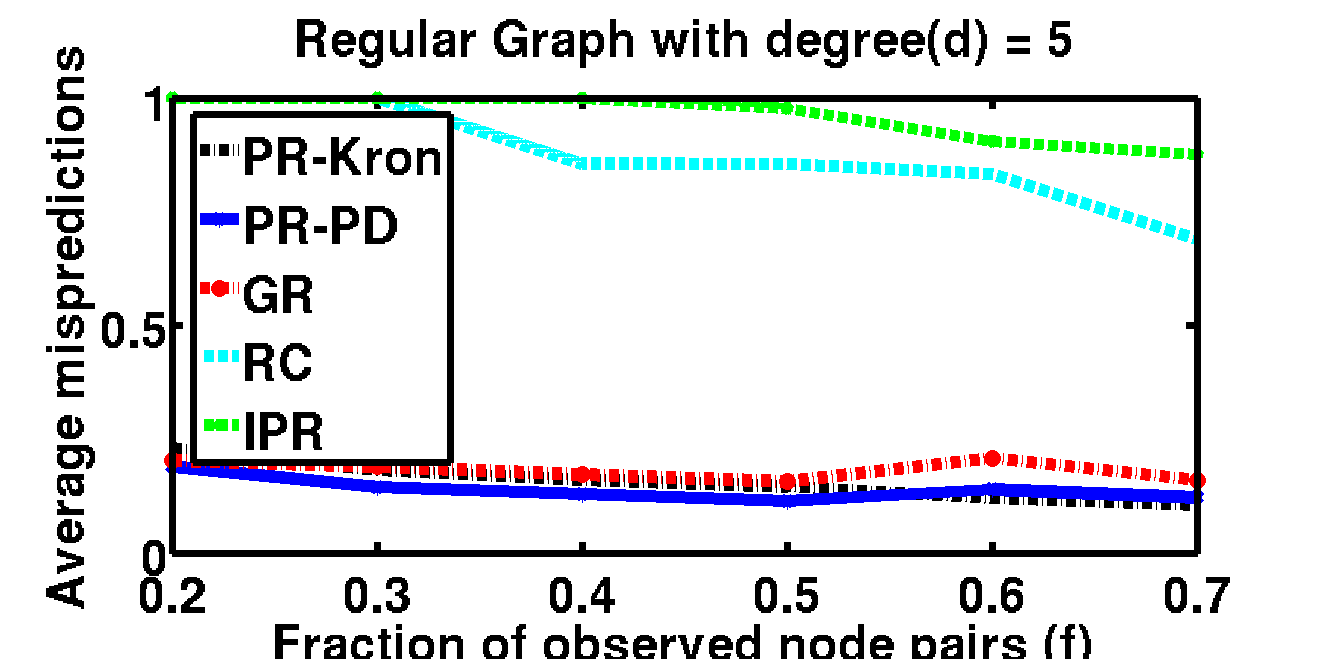}
\hspace{-6pt}
\includegraphics[trim={2.8cm 1.1cm 2.7cm 0.4cm},clip,scale=0.2,width=0.17\textwidth]{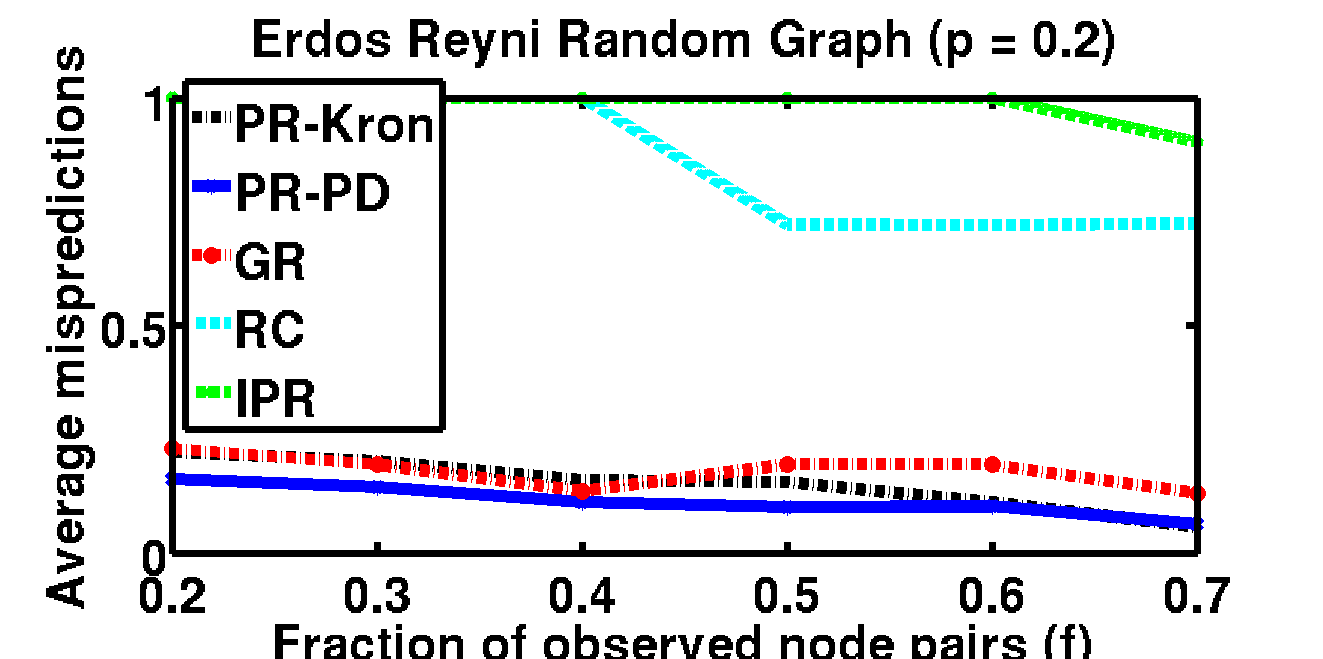}

\hspace{-18pt}
\includegraphics[trim={2.8cm 1.1cm 2.7cm 0.4cm},clip,scale=0.2,width=0.17\textwidth]{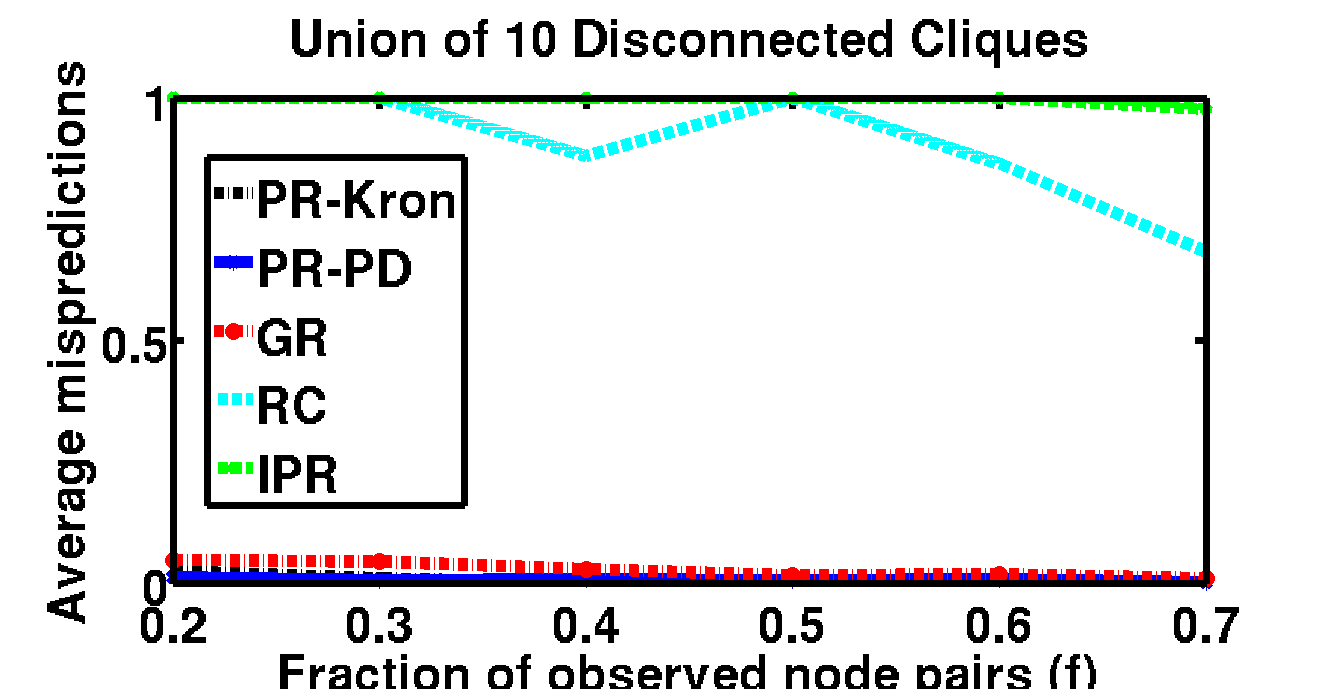}
\hspace{-6pt}
\includegraphics[trim={2.8cm 1.1cm 2.7cm 0.4cm},clip,scale=0.2,width=0.17\textwidth]{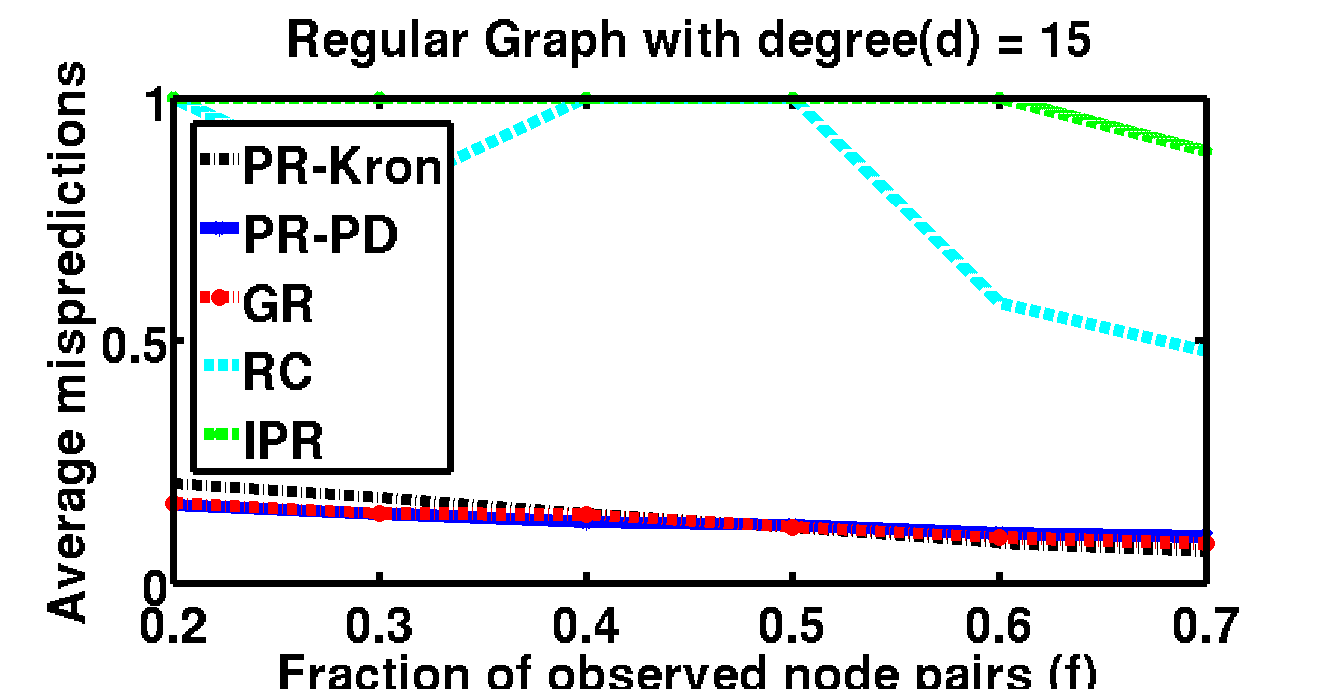}
\hspace{-6pt}
\includegraphics[trim={2.8cm 1.1cm 2.7cm 0.4cm},clip,scale=0.2,width=0.17\textwidth]{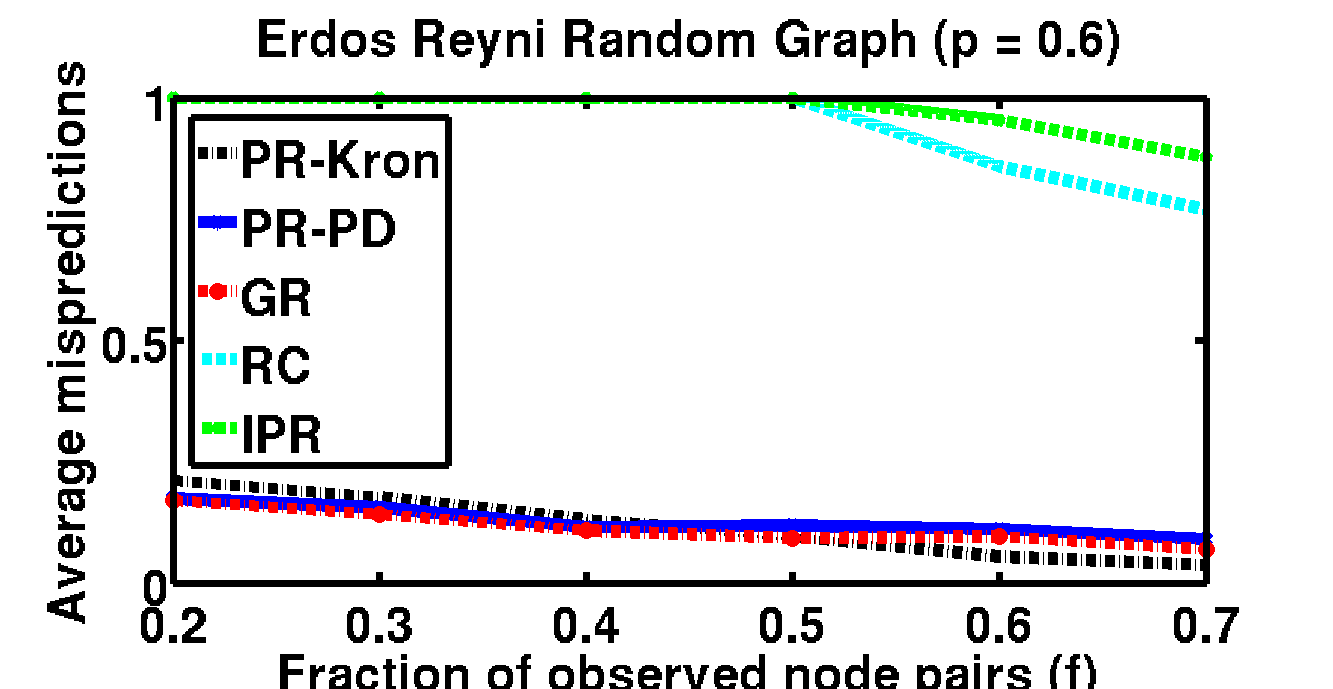}
\label{fig:plots_synth}
\hspace{-33pt}
\vspace*{-10pt}
\caption{Synthetic Data: Average number of mispredictions  ($er_{D}^{\ell^{0\text{-}1}}(\f)$, Eqn.~\eqref{eq:err_gen}) vs fraction of sampled pairs $(f)$.}%on different types of graphs
\hspace{-40pt}
%\end{widepage}
\end{figure}
\vspace{-10pt}

%\vspace*{-10pt}
%\vspace*{1pt}
\noindent \textbf{Generating $\bsig_n^*$.} For each of the above graphs, we compute $\f^* = \A_G\balpha$, where $\balpha \in [0,1]^n$ is generated randomly, and set $\bsig^*_n = \text{argsort}(\f^*)$ (see \alg, Step $3$ for definition).

% {\color{red} $\bsig_n^*$ seem to be a random ranking of nodes -- how are we capturing smoothness or locality property of the graph ?}

All the performances are averaged across $10$ repeated runs. The results are reported in Fig. \ref{fig:plots_synth}. 
In all the cases, our proposed algorithms \textbf{PR-Kron} and \textbf{PR-PD} outperforms the rest, with \textbf{GR} performing competitively well \footnote{See App. \ref{app:synth} for better comparisons of only the first $3$ methods.}.
As expected, \textbf{RC} and \textbf{IPR} perform very poorly as they could not exploit the underlying graph \emph{locality} based ranking property.

\vspace*{-4pt}
\subsection{Real-World Experiments}
\textbf{Datasets.} We use $6$ standard real datasets\footnote{https://www.csie.ntu.edu.tw/~cjlin/libsvmtools/datasets/} for three graph learning tasks -- $(a)$ \textit{Heart} and \textit{Fourclass} for \textbf{BR}, $(b)$ \textit{Vehicle} and \textit{Vowel} for \textbf{OR}, and $(c)$ \textit{House} and \textit{Mg} for \textbf{FR}.

\vspace*{1pt}\noindent
\textbf{Graph generation.}
For each dataset, we select $10$ random subsets of $40$ items each and construct a similarity matrix using RBF kernel, where $(i,j)^\text{th}$ entry is given by $\exp\big( \frac{-\|\x_i - \x_j\|^2}{2\mu ^2} \big)$, $\x_i$ being the feature vector and $\mu$ the average distance. For each of the $10$ subsets, we constructed a graph by thresholding the similarity matrices about the mean.

\vspace*{1pt}\noindent
\textbf{Generating $\bsig_n^*$.} For each dataset, the provided item labels are used as the score vector $\f^*$ and we set $\bsig^*_n = \text{argsort}(\f^*)$.
%The goal is to rank the items with larger relevance scores higher than that with relatively lower scores. 

\vspace{-5pt}
\begin{figure}[H]
%\begin{widepage}
\hspace{-18pt}
\includegraphics[trim={2.8cm 1.1cm 2.7cm 0.4cm},clip,scale=0.2,width=0.17\textwidth]{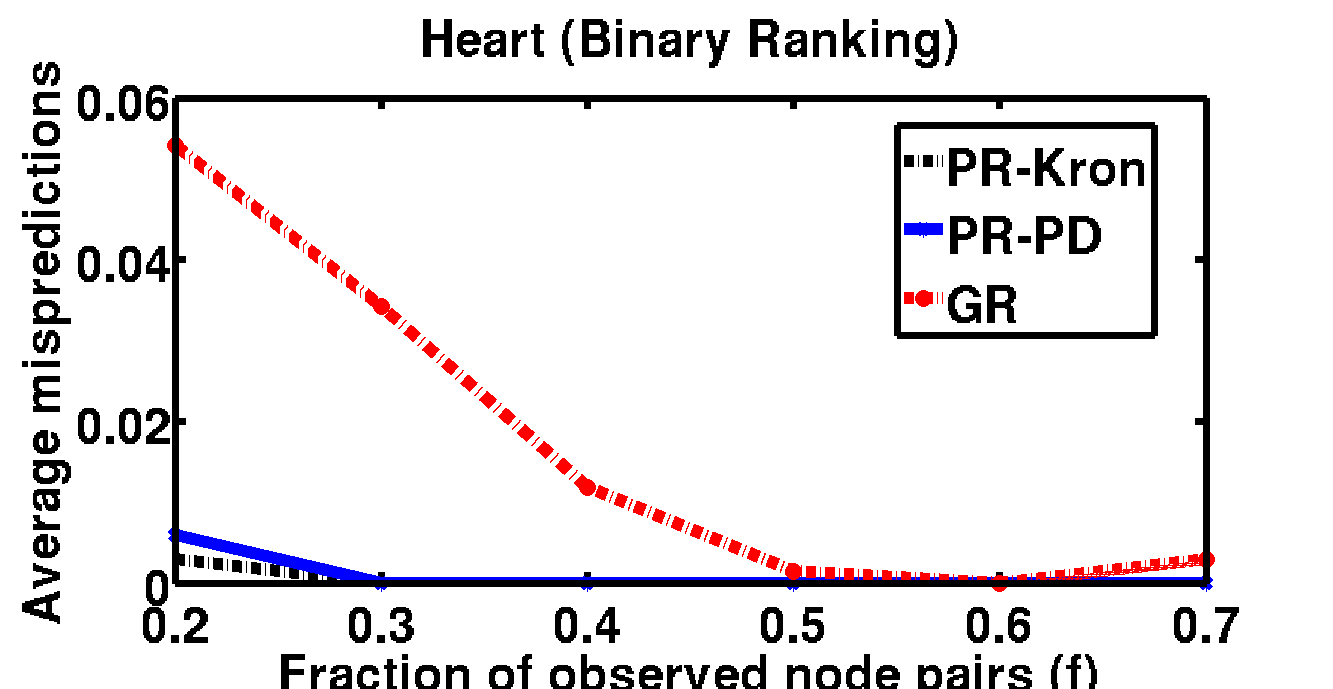}
\hspace{-6pt}
\includegraphics[trim={2.8cm 1.1cm 2.7cm 0.4cm},clip,scale=0.2,width=0.17\textwidth]{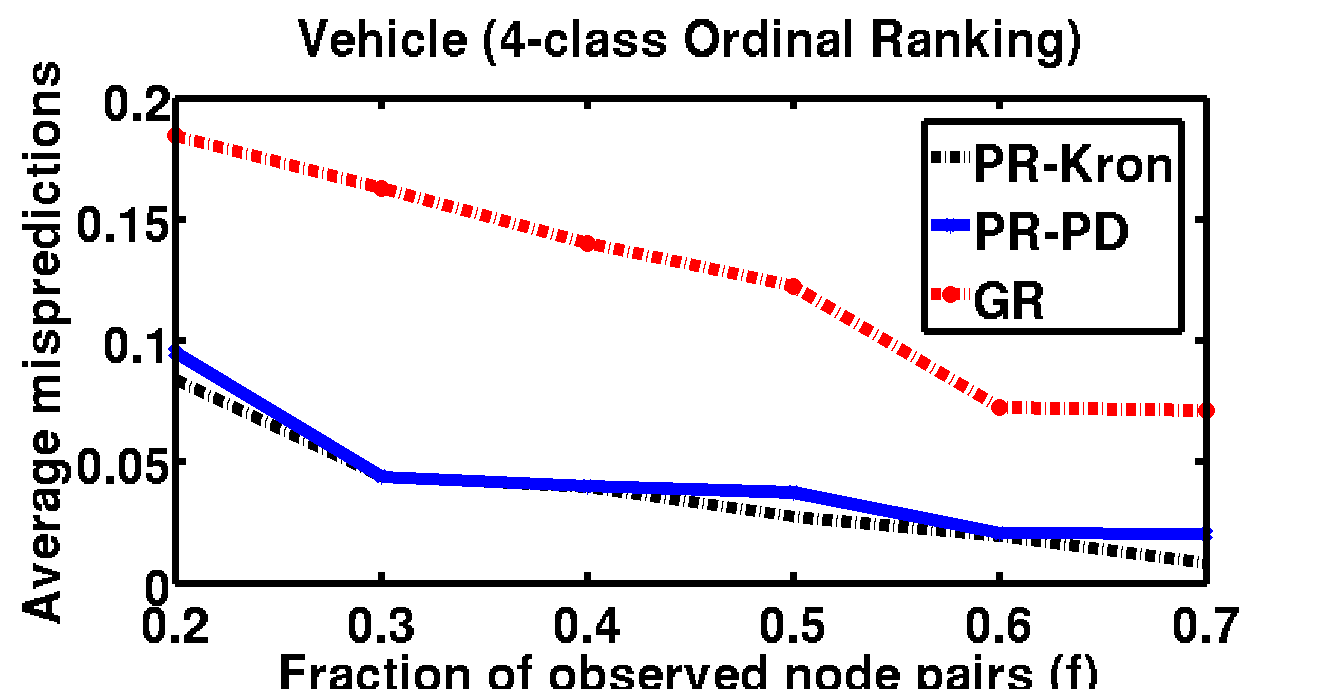}
\hspace{-6pt}
\includegraphics[trim={2.8cm 1.1cm 2.7cm 0.4cm},clip,scale=0.2,width=0.17\textwidth]{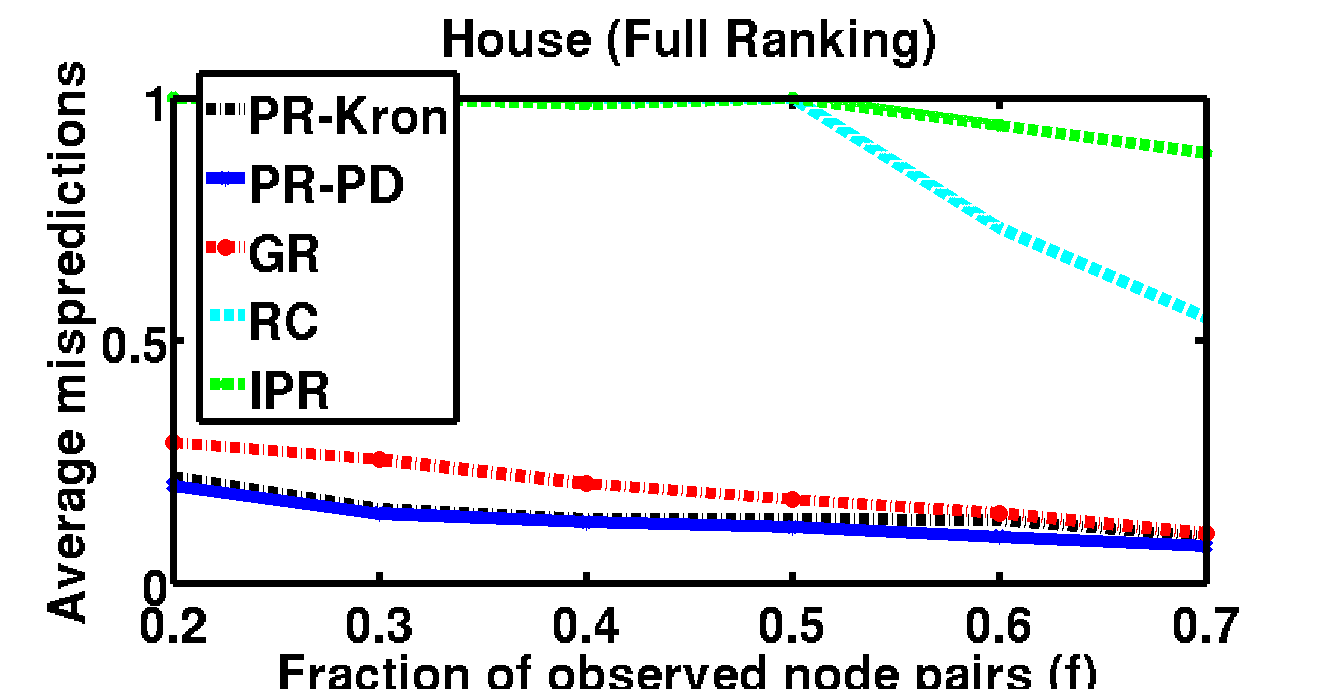}

\hspace{-18pt}
\includegraphics[trim={2.8cm 1.1cm 2.7cm 0.4cm},clip,scale=0.2,width=0.17\textwidth]{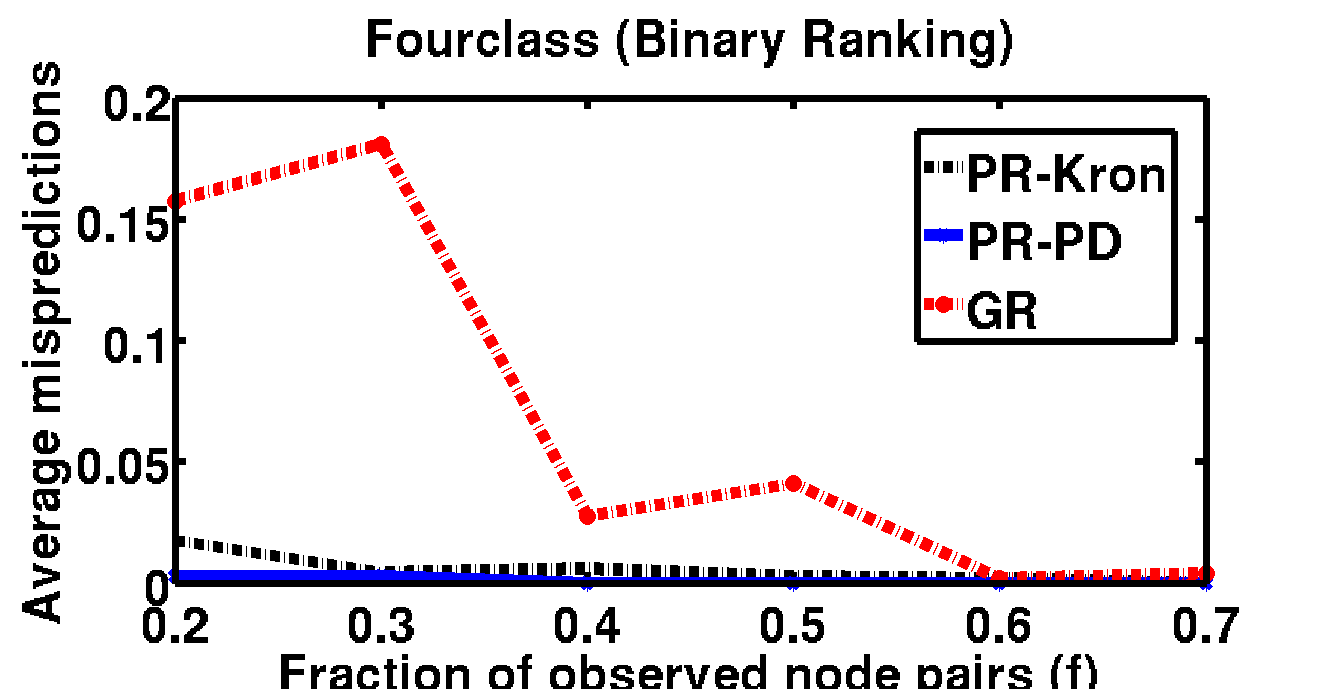}
\hspace{-6pt}
\includegraphics[trim={2.8cm 1.1cm 2.7cm 0.4cm},clip,scale=0.2,width=0.17\textwidth]{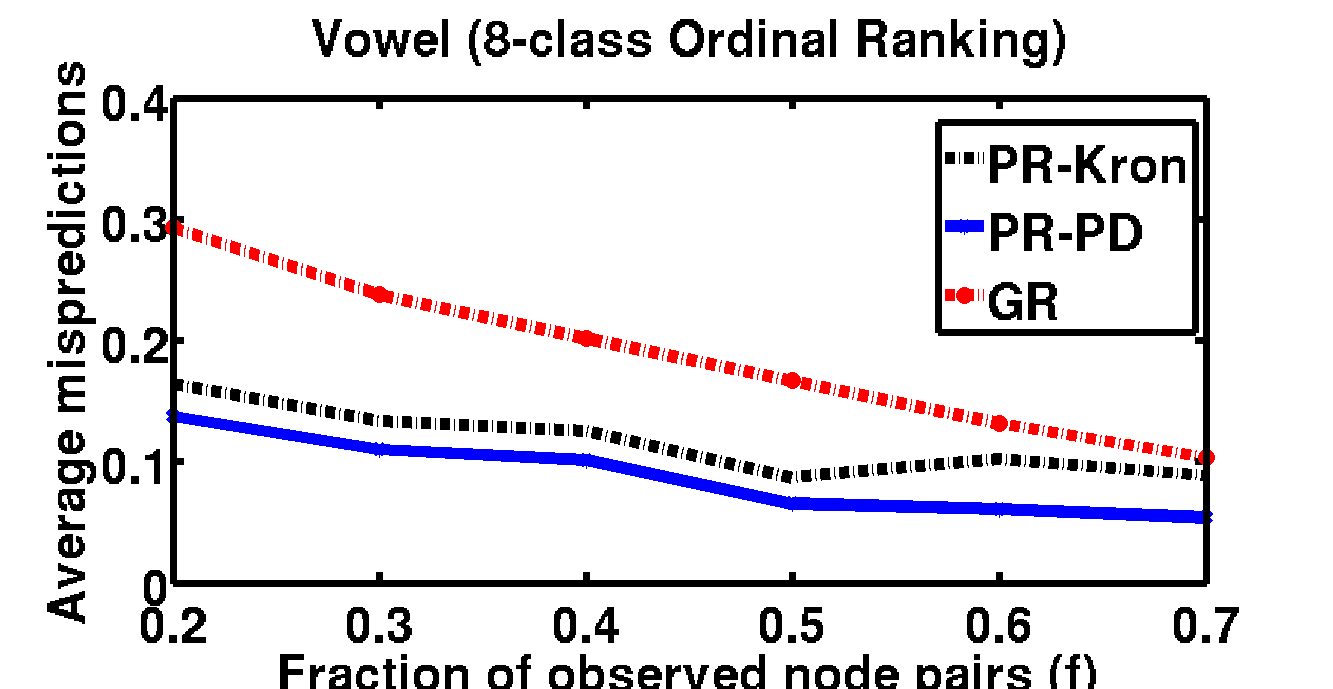}
\hspace{-6pt}
\includegraphics[trim={2.8cm 1.1cm 2.7cm 0.4cm},clip,scale=0.2,width=0.17\textwidth]{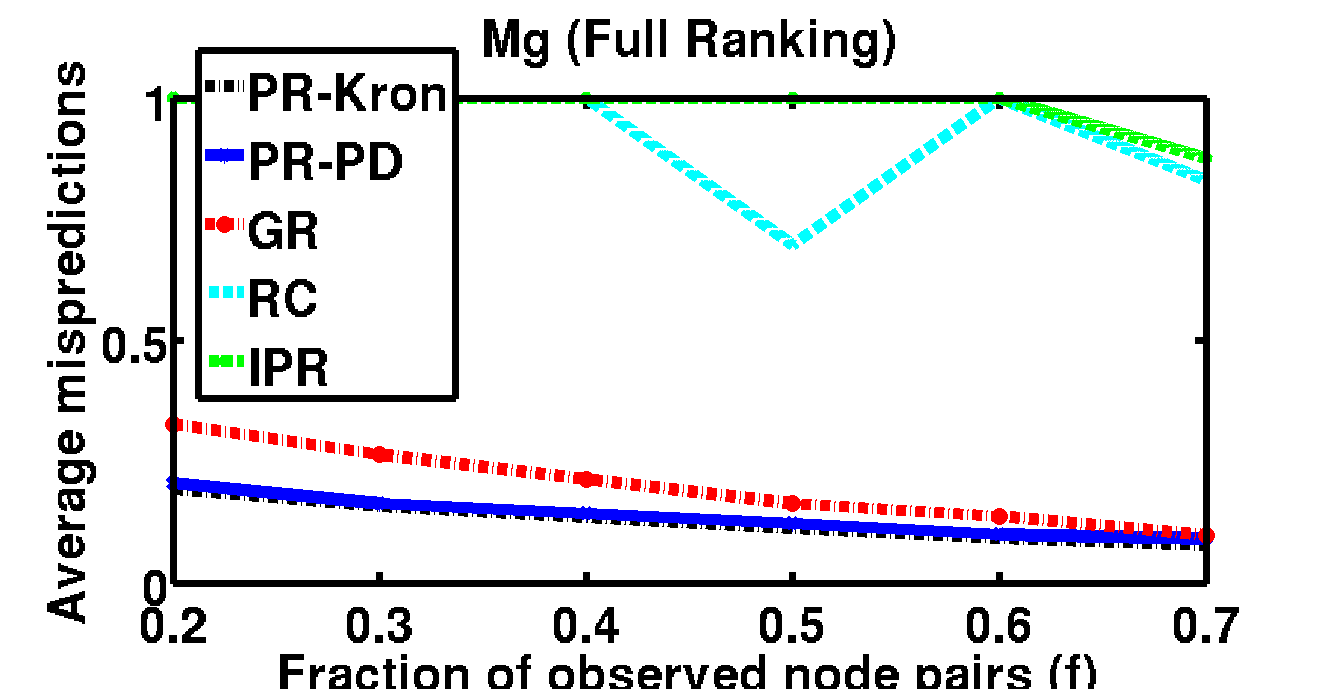}
\label{fig:plots_real}
\vspace*{-15pt}
\caption{Real-World Data: Average number of mispredictions ($er_{D}^{\ell^{0\text{-}1}}(\f)$, Eqn.~\eqref{eq:err_gen}) vs fraction of sampled pairs $(f)$.}
\hspace{-40pt}
%\end{widepage}
\end{figure}
\vspace{-12pt}

For each of the task, the averaged result across $10$ randomly subsets are reported in Fig.~\ref{fig:plots_real}. As before, our proposed methods~\textbf{PR-Kron} and~\textbf{PR-PD} perform the best, followed by~\textbf{GR}. Once again~\textbf{RC} and~\textbf{IPR} perform poorly\footnote{We omit them for~\textbf{BR} and \textbf{OR} for better comparisons.}.
Note that, the performance error increases from bipartite ranking~\textbf{(BR)} to full ranking~\textbf{(FR)}, former being a relatively simpler task.
Results on more datasets are available in App. \ref{app:real} and \ref{app:real_add}.

% ---------------- Conclusion ----------------

%\vspace*{-20pt}

\section{Conclusion and Future Works}
\label{sec:conclusion}

In this paper we addressed the problem of ranking nodes of a graph $G([n],E)$ given a random subsample of their pairwise preferences. Our proposed algorithm \alg, guarantees consistency with a required \emph{sample complexity} of $O\big(n^2\chi(\bar G)\big)^\frac{2}{3}$ -- also gives novel insights by relating the ranking sample complexity with graph structural properties through chromatic number of $\bar G$, i.e. $\chi (\bar G)$, for the first time.
%We study the problem of ranking nodes of a given graph $G(V,E)$ from pairwise node preferences using \alg, a kernelized learning algorithm. Our main contribution lies in showing consistency of our proposed algorithm using \emph{orthonormal embeddings of strong product graphs}. Further we derive the \emph{required sample complexity} for the purpose in terms of chromatic number of the complement graph $\chi(\bar{G})$ -- a new insights to relate the problem complexity with graph structural properties.
%Noting the computational challenge in finding optimal embeddings, we moreover suggest an alternative $poly(n)$ time \ls\, embedding scheme, and verify its effectiveness of variety of experiments.
%Our experiments also show better performance of \alg~ with \ls\, over Laplacian based state-of-the-art methods.
 %Furthermore, the proposed algorithm and generalization bounds could also be extended to ranking on multiple graphs.
One possible future direction is to extend the setting to noisy preferences e.g. using BTL model \cite{Negahban+12},
%, one only gets to see a noisy version of it, . 
%One can also explore the problem on more general feedback models, e.g. partial preferences over nodes rather than just pairwise comparisons, or subset-wise preferences, where preference information are given on subsets of nodes instead of just nodes pairs \cite{KhetanOh16}. 
or analyse the problem with other measures of ranking losses e.g. NDCG, MAP \cite{AgarwalT08}.
Furthermore, proving consistency of \alg\ algorithm using \pdLab\ also remains an interesting direction to explore.

{
%\small
%\bibliographystyle{aaai}
\bibliographystyle{plainnat}
\bibliography{AAAI-SahaA.770}
}

\appendix
% --------------------- Appendix ---------------------
\newpage
\clearpage
%\allowdisplaybreaks
%\appendix

\twocolumn[\section*{\LARGE{Appendix: How Many Preference Pairs Suffice to Rank a Graph Consistently?}}
\vspace*{10pt}]

%\twocolumn[\section*{\LARGE{Supplementary: Achieving Consistency in Graph Ranking with Strong Products and Connections to Chromatic Number}}
%\vspace*{10pt}]

\section{Discussion of Locality Property on RKHS}
\label{app:rkhs_smooth}
%Similar to \cite{SpectralBandits}, we address the problem of MAB on a given graph $G(V,E)$ with $V = [N]$. In order to exploit the graph structure for achieving lower regret, we further impose the regularity assumption of ``smooth rewards" over the graph. 
By definition, any smooth function $\f: V \mapsto \R$ over a graph $G(V,E)$ implies $\f$ to vary slowly on the neighbouring nodes of the graph $G$; i.e., if $(i, j) \in E$ then $f_i  \approx f_j, ~\forall i,j \in V$. The standard way of defining this is by considering $\f^{\top}\L\f = \sum_{(i,j) \in E}(f_i - f_j)^2$ to be small, say $\f^{\top}\L\f \le B$, for some constant $B >0$. Clearly a small value of $B$ implies $|f_i - f_j|$ to be small for any two neighboring nodes, i.e. $(i,j) \in E$.

We first analyze the RKHS view of the above notion of smooth reward functions. Consider the SVD of the graph Laplacian $\L = \Q \bLambda \Q^T$, where $\Q = [\q_1 ~\q_2 ~\ldots ~\q_{n}] \in \R^{n \times n}$, $\bLambda = diag(\lambda_1, \lambda_2, \ldots, \lambda_n)$ and suppose the singular values $\lambda_i = 0, ~\forall i > d$, for some $d \in [N]$. Now consider the linear space of real-valued vectors--
\[
\cH(G) = \{\g \in \R^n \mid \g^{\top}\q_i = 0 ~~\forall i > d \}
\]
Note since $\L \in \S_n^+ $ is positive semi-definite, the function $\|\cdot\|_\L: \cH(G) \mapsto \R$, such that $\|\g\|_\L = \g^{\top}\L\g$ defines a valid norm on $\cH(G)$. In fact, one can show that $\cH(G)$ along with the inner product $\langle\cdot,\cdot\rangle_\L: \cH(G) \times \cH(G) \mapsto \R$, such that $\langle \g_1,\g_2 \rangle_\L = \g_1^{\top}\L\g_2, ~\forall \g_1,\g_2 \in \cH(G)$, defines a valid RKHS with respect to the reproducing kernel $\K = \L^{\dagger}$. This can be easily verified from the fact that $\forall \g \in \cH(G)$, $\L^{\dagger}\L\g = \g$, and hence $\langle \g,\K_i \rangle_{\L} = \g^{\top}\L\K\e_i = (\L\K\g)^{\top}\e_i = (\L^{\dagger}\L\g)^\top\e_i = g_i, ~\forall i \in [N]$. 

Thus the smoothness assumption on the reward function $\f$, can alternatively be interpreted as $\f$ being small in terms of the RKHS norm $\|\cdot\|_{\L}$. The above interpretation gives us the insight of extending the notion of ``smoothness" with respect to a general RKHS norm associated to some kernel matrix $\K \in \S_n^+ $.
More specifically, we choose the kernel matrix $\K$ from the set of orthonormal kernels $\cK(G)$ and consider $\f$ to be smooth in the corresponding RKHS norm. Note here the Hilbert space of functions $\cH(\K)$ is given by
\begin{equation}
\label{eq:HK}
\cH(\K) = \{\g \in \R^n \mid \g^{\top}\q_i = 0 ~~\forall i > d \},
\end{equation}

where same as before, the SVD of $\K = \Q\bLambda\Q^{\top}$, $\Q = [\q_1, \ldots, \q_n] \in \R^{n \times n}$ being the orthogonal eigenvector matrix of $\K$, $\bLambda = diag(\lambda_1, \ldots \lambda_n)$ be the diagonal matrix containing singular values of $\K$. Clearly $\lambda_i = 0, ~\forall i > d$ implies $r(\K) = d$. Also we define the corresponding inner product $\langle\cdot,\cdot\rangle_\K: \cH(\K) \times \cH(\K) \mapsto \R$, as $\langle \g_1,\g_2 \rangle_\K = \g_1^{\top}\K^{\dagger}\g_2, ~\forall \g_1,\g_2 \in \cH(\K)$. Then similarly as above, we can show that $\cH(\K)$ along with $\langle\cdot,\cdot\rangle_\K$ defines a valid RKHS with respect to the reproducing kernel $\K$, as $\forall \g \in \cH(\K)$, $\langle \g,\K_i \rangle_{\K} = \g^{\top}\K^\dagger\K\e_i = g_i, ~\forall i \in [n]$.

The RKHS norm $\|\g\|_{\K} = \g^{\top}\K^{\dagger}\g$ defines a measure of the smoothness of $\g$, with respect to the kernel function $\K$. 
%and satisfies $\g \in \cH(\K)$ if and only if $\|\g\|_K < \infty$. 
One way to see this is that $\forall \g \in \cH(\K)$, $\|g_i - g_j\| = \|\langle\g, (\K(i,\cdot) - \K(j,\cdot)) \rangle\| \le 
\|\g\|_{\K}\|\K(i,\cdot) - \K(j,\cdot)\|_\K 
 = \|\g\|_{\K}|(K_{ii}+K_{jj}-2K_{ij})|$, where the inequality follows from the Cauchy-Schwarz inequality of RKHS$(\K)$. Note since $\K \in \cK(G)$, $K_{ii} = 1, ~\forall i \in [N]$, we have $|(K_{ii}+K_{jj}-2K_{ij})| \le 4$ $\forall i,j \in [N]$. In particular, for two neighboring nodes $i$ and $j$ such that $(i,j) \in E$, it is expected that $K(i,j) \approx 1$ (i.e. $\u_i \approx \u_j$), in which case the quantity $|(K_{ii}+K_{jj}-2K_{ij})| \approx 0$. Thus to impose a smoothness constraint on $\g$, it is sufficient to upper bound $\|\g\|_{\K} \le B$, for some fixed $B \in \R$, $\forall \g \in \cH(\K)$.

We thus justify our assumption of $\|\f\|_{\K} \le B$ which implies the ranking function (vector) $\f$ to be a smooth functions over the underlying graph $G$, with respect to embedding $\K$. 
Interestingly, $\cH(\K)$ incorporates $\cH(G)$ as its special case with $\K = \L^{\dagger}$. 
Thus our space of ranking functions rightfully generalizes the Laplacian based rankings, as studied by \cite{Agarwal09,Agarwal10}.
From the definition of $\cH(\K)$ in \eqref{eq:HK}, it follows that the unknown ranking function $\f \in \cH(\K)$, lies in the column space of $\K$, i.e. $\f = \K\boldsymbol \balpha$, for some $\boldsymbol \talpha \in \R^N$. Also recall $\forall \K \in \cK(G)$, there exists an $\U \in Lab(G)$, such that $\K = \U^{\top}\U$, $\U \in \R^{N \times N}$. Thus we have $\f = \K\boldsymbol \balpha = \U^{\top}\U\boldsymbol \balpha = \U^{\top}\talpha$, where $\talpha = \U\balpha$. 
%This shows the representation of $\f$ with respect to different bases, which gives the flexibility of embedding the nodes in the corresponding bases and fixing the coefficient (i.e. $\boldsymbol \talpha, \balpha \text{ or } \btalpha'$) accordingly, such that $\f^{\top}\K^{\dagger}\f = \boldsymbol \talpha ^{\top}\K\boldsymbol \talpha = \|\balpha\|_2 \le B$.

\begin{lem}
\label{lem:tf_smooth}
If $\f \in$ \textup{RKHS}$(\K)$, $\f^\top \K \f \le B$, and we define $\tK = \K \otimes \K$, $\tilde \f = \f \otimes \f$, then $\tilde \f \in$ RKHS$(\tK)$, $\tilde \f^\top \tK \tilde \f \le B^2$.
\end{lem}

\begin{proof}
The proof follows from the straightforward properties of tensor products. We describe it below from completeness:
Since $\f \in \text{RKHS}(\K)$, we have $\f = \K\balpha$ for some $\balpha \in \R^n$. Now
\begin{align*}
\tilde \f &= \f \otimes \f = (\K \balpha) \otimes (\K\balpha) \\ &= (\K \otimes \K)*(\balpha \otimes \balpha) = \tK(\balpha \otimes \balpha),
\end{align*}
and hence $\tilde \f \in$ RKHS$(\tK)$, where the second last inequality follows due to the the properties of tensor product. Further more, using the same property, we have
\begin{align*}
\tilde \f^\top \K \tilde \f & = (\f \otimes \f)^\top (\K \otimes \K) (\f \otimes \f) \\ &= (\f \otimes \f)^\top (\K\f \otimes \K\f)\\
& = (\f^\top \K \f)*(\f^\top \K \f) \le B^2
\end{align*}
% which concludes the proof.
\end{proof}

%\begin{defn} {\bf Lov\'asz Number.} \cite{lovasz_shannon}
%\label{defn:lovasz_theta}
%
%Orthonormal representations $Lab(G)$ of a graph $G$ is associated with an interesting quantity -- {\it Lov\'asz number} of $G$, defined as 
%\[
%\vartheta(G): = \min_{\U \in Lab(G)}\min_{\c \in S^{d-1}}\max_{i \in V}\frac{1}{(\c^{\top}\u_i)^2}
%\]
%\end{defn}
%
%{\it Lov\'asz Sandwich Theorem}: If $I(G)$ and $\chi(G)$ denote the independence number and chromatic number of the graph $G$, then $I(G) \le \vartheta\left( G\right) \le \chi(\bar G)$ \cite{lovasz_shannon}.
%Also it is known from the same classical work that $\vartheta(\strPr) = \vartheta^2(G)$. 
%

\begin{defn}{\bf Strong Product of Graphs.} 
\label{def:str_pr}
Given a graph $G=(V,E)$, strong product of $G$ with itself, denoted by $\strPr$, is defined over the vertex set $V(\strPr) = V \times V$, such that two nodes $(i, j),(i', j') \in V(\strPr)$ is adjacent in $\strPr$ if and only if $i = i'$ and $(j, j') \in E$, or $(i, i') \in E$ and $j = j'$, or $(i, i') \in E$ and $(j, j') \in E$. Note that for every node $k \in V( \strPr)$, there exists a corresponding node pair $(i_k,j_k) \in V \times V$ in the original graph $G$.

Let $\U = [\u_1,\ldots, \u_{n}] \in \R^{d \times n}$ and $\mathbf{V} = [\v_1, \ldots, \v_{n}] \in \R^{d' \times n}$ be any two orthonormal representations of $G$. We denote $\u \otimes \v = [u_1v_1  \ldots u_1v_{n} ~ u_2v_1 \ldots u_{n}v_{n}]^{\top} \in \R^{d d'}$ to be the {\it kronecker (or outer) product} of the two vectors $\u \in \R^{d}, \v \in \R^{d'}$. Let $\w_{k} = \u_{i_k} \otimes \v_{j_k} \in \R^{d d'}$, for every node $k \in V( \strPr)$. It is easy to see that any such embedding $\W = [\w_1, \w_2, \cdots \w_{n^2}] \in \R^{dd' \times n^2}$ defines a valid orthonormal representation of $\strPr$. Using above, it can also be shown that $\vartheta(\strPr) = \vartheta^2(G)$~\cite{lovasz_shannon}.
\end{defn}

%%%%%%%%%%%%%%%%%%%%%%%%%%%%%%%%%%%%%%%%%%%%%%%

\section{Appendix for Section \ref{sec:gen_err}}

\subsection{Proof of Theorem \ref{thm:gen_err}}

\ThmGenErr*

\begin{proof}
To proof the above result, let us first recall the  error bound for learning classification models in transductive setting from \cite{YanivPe09}. 

Consider the problem of transductive binary classification over a fixed set $S_{m+u} = \{(\x_i, y_i )\}_{i = 1}^ {m+u}$ of $m + u$ points, where $\x_i \in \R^d$ denotes the instances with their labels $y_i$. The learner is provided with the unlabeled (full) instance set $X_{m+u} = \{\x_i\}_{i = 1}^{m+u}$. A set consisting of $m$ points is selected from $X_{m+u}$ uniformly at random among all subsets of size $m$. These $m$ points together with their labels are given
to the learner as a training set. Renumbering the points, suppose the unlabeled training
set points are denoted by $X_m = \{\x_1, \ldots, \x_m\}$ and the labeled training set by $S_m = \{(\x_i, y_i)\}_{i = 1}^{m}$. The goal is to predict the labels of the unlabeled test points, $X_u = \{x_{m+1}, \ldots, x_{m+u}\} = X_{m+u} \setminus X_m$, given $S_m \cup X_u$. 

Consider any learning algorithm generates soft classification vectors $\h = (h_1,\ldots, h_{m + u}) \in \R^{m+u}$ (or equivalently $\h$ can also be seen as function such that $h: \X_{m+u} \mapsto \R$). $h_i ( = h(\x(i))) \in \R$  denotes the soft label for the example $\x_i$ given by the hypothesis $\h$. For actual (binary) classification of $\x_i$, the algorithm outputs $sgn(h_i)$. 
The soft classification accuracy is measured with respect to the some loss function $\ell: \{\pm 1\} \times \R \mapsto [0,B]$. Thus $\ell(y_i,h_i)$ denotes the loss for the $i^{th}$ instance $\x_i$. We denote by $\ell^{0-1}$, the $0$-$1$ loss vector, i.e. $\ell^{0-1}(y_i,h_i) = \1(y_i \neq sgn(h_i))$.
%We denote the loss vector $$

\begin{thm}[Transductive test error bound (Thm. 2) \cite{YanivPe09}]
Let $\cH_{out} \subseteq \R_ {m+u}$ denotes the set of all possible soft classification vectors generated by the learning algorithm, upon operating on all possible training/test set partitions, the loss function $\ell^{\rho}$ is $\rho$-lipschitz. Then for $c = \sqrt{\frac{32\ln(4e)}{3}} < 5.05$, $Q = \big(\frac{1}{m} + \frac{1}{u}\big)$, and $S = \frac{m+u}{(m+u - 1/2)(1-\frac{1}{2\max{(m,u)}})}$, and a fixed $\rho$, with probability of at least $(1 - \delta)$ over the choice of the training set from $X_{m+u}$, for all $h \in \cH_{out}$
\begin{align}
\label{eg:yanivpe}
\frac{1}{u}\sum_{i = m+1}^{m+u} & \ell^{\rho}(y_i,h_i) \le \frac{1}{m}\sum_{i = 1}^{m}\ell^{\rho}(y_i,h_i) + \frac{R(H_{out})}{\rho}\\
\nonumber & + cBQ\sqrt{\min{(m,u)}} + B\sqrt{\frac{SQ}{2}\ln \frac{1}{\delta}},
\end{align}
\end{thm}
where $R_{m+u}(\cH_{out}) = Q\E_{\bgam}\left[ \sup_{\h \in \cH_{out}} \bgam^{\top}\h \right]$ is the pairwise Rademacher complexity of the function class $\cH_{out}$, $\bgam = (\gamma_1, \ldots, \gamma_{m+u} )$ be a vector of i.i.d. random variables such that $\gamma_i \in \{\pm 1, 0\}, ~i \in [m+u]$, with probability $p$, $p$ and $1 - 2p$ respectively, with $p = \frac{mu}{(m+u)^2}$.

It is now straightforward to see that, for our current problem of interest training and test set sizes are respectively $m = Nf$ and $u = N(1-f)$. This immediately gives that $Q = \frac{1}{Nf(1-f)}$, $\min{(m,u)} = Nf$ and $p = f(1-f)$. The true labels of the pairwise classification problem are given by $y_k = sgn(\sigma^*(i_k) - \sigma^*(j_k)), ~\forall k \in [N]$ and the function class $\cH_{out} = \cH_{\tU}$. Thus $R(\cH_{\tU},{\tU},p) = f(1-f)R(\cH_{out})$. Also note that for large $n$ and $f < \frac{1}{2}$, $S = \frac{N}{(N - 1/2)\left( 1 - \frac{1}{N} \right)} \approx 1$. Thus \eqref{eg:yanivpe} reduces to 
\begin{align*}
%\label{eg:yanivpe2}
er_{\bS}^{\ell^{\rho}}[\f] &= \frac{1}{u} \sum_{i = m+1}^{m+u} \ell^{\rho}(y_i,h_i) \\
 &\le \frac{1}{m}\sum_{i = 1}^{m}\ell^{\rho}(y_i,h_i) + \frac{R(\cH_{\tU},\tU,p)}{\rho f(1-f)}
 \\ & \quad \quad ~~ + \frac{C_1B}{(1-f)\sqrt{Nf}}\left( 1 + \sqrt{\ln\left( \frac{1}{\delta} \right)} \right),
%\nonumber & \le er_{S}^{\ell^{\rho}}[\f] + \frac{C\sqrt{2(tr(\tK))\lambda_1(\tK)}}{\rho \sqrt{Nf(1-f)}}\\
% + &\frac{C_1B}{(1-f)\sqrt{Nf}}\left( \sqrt{\ln\left( \frac{1}{\delta} \right)} \right), (\text{using Thm. } \ref{thm:rad_eig})
\end{align*}
for $C_1 >0$ being the appropriate constant. Thus the claim follows. 
%noting that $er_{\bS}^{\ell^{\rho}}[f] = $ and 
%$er_{S}^{\ell^{\rho}} + \frac{C\lambda_1(\K)}{\rho\sqrt{f(1-f)}} + \frac{c_1B}{1-f}\sqrt{\frac{\log (\frac{1}{\delta})}{n^2f}}$
\end{proof}

\section{Appendix for Section \ref{sec:embedding}}

\subsection{Characterization: Choice of Optimal Embedding}
\label{app:char_opt_embed}

In this section, we discuss different classes of pairwise preference graph embeddings and the corresponding generalization guarantees. 
We start by recalling Thm. $1$ of \cite{AndoZh07}, which provides a crucial characterization for the class of optimal embeddings:%to \eqref{eq:svm3} we get the following guarantee on the generalization bound of \alg:

%\iffalse %%%%%%%%%%%%%%%%%%%%%%%%%%%
Suppose $\f^*$ denotes the score function returned by the following optimization problem
\[
\f^* = \underset{\f \in \R^N}{\text{argmax}} ~ C'\f^{\top}\tK^{-1}\f + \hat{er}^{\ell^{\rho}}_{S_m}(\f),
\]
(note that for \alg~ (Eqn.~\ref{eq:svm3}), $C' = \frac{1}{2Cm}$ and $\ell^{\rho} = \ell^{hinge}$), then drawing a straightforward inference, we get
%\fi %%%%%%%%%%%%%%%%%%%%%%%%%%%%%%%%%%%%%%%%%%

\begin{cor}
\label{thm:opt_embed}
Suppose $\f^*$ denotes the optimal solution of \eqref{eq:svm3}.
Then, over the random draw of $S_m \subseteq \cP_n$, the expected generalization error \textit{w.r.t.} any $\rho$-Lipschitz loss function $\ell^{\rho}$ is given by
\begin{align*}
& \E_{S_m}[er_{\bS_{m}}^{\ell^{0-1}}(\f^*)] = \frac{1}{N-m}\E_{S_m}\left[\sum_{k = m+1}^{N}\ell^{0-1}(y_k,{\f_k^*})\right] \\
& \le \inf_{\f \in \R^{N}}
\frac{1}{c_1}\left[ er_{S_m \cup \bS_m}^{\ell^{\rho}}(\f) + C'\f^{\top}\tK^{-1}\f \right] + c_2\Bigg(\frac{tr_{p}(\tK)}{\rho mC'}\Bigg)^p
\end{align*}
where $tr_{p}(\tK) =  \left( \frac{1}{N}\sum_{k = 1}^{N}\tilde{K}_{kk}^{p} \right)^{\frac{1}{p}}$, $er_{S_m \cup \bS_m}^{\ell^{\rho}}[\f] =\frac{1}{N}\sum_{k = 1}^{N}\ell^{\rho}(y_k,f_k)$ and $p, c_1,c_2 >0$ are fixed constants dependent on $\ell^{\rho}$.
\end{cor}

%{\color{red} What is $C'$ ? Does it relate to $C'=1/Cm$ in \eqref{eq:svm3}.}

Now following a similar chain of arguments as in \cite{AndoZh07}, Cor. \ref{thm:opt_embed} implies that {\it a normalized graph kernel $\tK = \tU^{\top}\tU$ such that $\tilde K_{kk} = 1, \forall k \in [N]$ leads to improved generalization performance}, since it ensures $tr(\tK)_p$ to be constant. 
Furthermore, the following theorem shows that the class of `normalized' graphs embeddings have high {rademacher complexity}.

\subsection{Proof of Theorem \ref{thm:rad_eig}}

\ThmRadEig* 

\begin{proof}
%Since $\|\bbeta\|_2 \le C$, we get that for any $\h \in \cH_{\tU_P}$, $\|\h\|_2 \le \max_{\bbeta}\|\tU\bbeta\|_2 = \max_{\bbeta \in \R^N, \|\bbeta\|_2 \le C}\sqrt{\bbeta^{\top}\tK\bbeta} \le \sqrt{\lambda_1(\tK)}\|\bbeta\|_2 = C\sqrt{\lambda_1(\tK)}$.

Note that for any fixed realization of $\bgam = [\gamma_1, \ldots \gamma_N]$, 
\begin{align*}
& \sup_{\h \in \cH_{\tU}} \sum_{k = 1}^{N}\gamma_k(\h^{\top}{\tu}_{k}) = \sup_{\h \in \cH_{\tU}} \h^{\top}\big(\tU\bgam\big) \\
& = \sup_{\bbeta \in \R^N : \|\bbeta\|_\infty \le C} \bbeta^{\top}\tU^{\top}\big(\tU\bgam\big) \\
& \le \sup_{\bbeta \in \R^N : \|\bbeta\|_\infty \le C} \|\tU\bbeta\|_2\|\tK\bgam\|_2 ~~\big( \mbox{Cauchy-Schwarz Ineq.}\big)\\
& \le \sqrt{\lambda_1(\tK)}\sup_{\bbeta \in \R^N : \|\bbeta\|_\infty \le C}\|\bbeta\|_2\|\tU\bgam\|_2\\
& \le C\sqrt{N\lambda_1(\tK)} \|\tU\bgam\|_2
\end{align*} 

Using above we further get:

\begin{align*}
& R(\cH_{\tU}, \tU, p) \le \frac{1}{N} \E_{\bgam}\left[ C\sqrt{N\lambda_1(\tK)}\|U\bgam\|_2 \right]\\
&= \frac{C\sqrt{\lambda_1(\tK)}}{\sqrt{N}}\E_{\bgam}\left[ \sqrt{{\bgam^\top\tK\bgam}} \right]\\
& \le \frac{C\sqrt{\lambda_1(\tK)}}{\sqrt{N}} \sqrt{  \E_{\bgam}\left[ \bgam^\top\tK\bgam  \right]} ~~\big( \mbox{Jensen's Inequality}\big)\\
& = \frac{C\sqrt{\lambda_1(\tK)}}{\sqrt{N}} \sqrt{2p(tr(\tK))} = \frac{C}{\sqrt N}\sqrt{2p\lambda_1(\tK)(tr(\tK))},
\end{align*}
where the second last equality follows from the fact that 
$\E_{\bgam}\left[ \bgam^\top\tK\bgam  \right] = 
2p\sum_{k = 1}^{N}\tilde K_{kk} = 2p(tr(\tK))$, as $\bgam = (\gamma_1, \ldots, \gamma_{N} )$ be a vector of \emph{i.i.d.} random variables such that $\gamma_i \in \{+1, -1, 0\}, ~i \in [N]$, with probability $p$, $p$ and $1 - 2p$ respectively and $tr(\tK) = \sum_{k = 1}^{N}\tilde K_{kk}$. The proof now follows from the fact that $tr(\tK) = N$, since $\tilde K_{kk} = 1, ~\forall k \in [N]$. 
\iffalse %%%%%%%%%%%%%%%%%%%%%%%%%%%%%%%%%
\begin{align*}
& R(\cH_{\tU}, \tU, p) = \frac{1}{N} \E_{\bgam}\left[ \sup_{\h \in \cH_{\tU}} \sum_{k = 1}^{N}\gamma_k \h(\tu_{k}) \right]\\
&= \frac{1}{N}\sqrt{N\lambda_1(\tK_P)}tC\E_{\bgam}{\left\lVert\sum_{k = 1}^{N}\gamma_k{\tU}_{P_k}\right\rVert_2}  \\
& = \frac{1}{N}\sqrt{N\lambda_1(\tK_P)}tC\E_{\bgam} \left[ \sqrt{\bgam \tK_P \bgam} \right] \\
& \le \frac{1}{N}\sqrt{N\lambda_1(\tK_P)}tC \sqrt{ \E_{\bgam} \left[ {\bgam \tK_{P} \bgam} \right]} \text{(by Jensen's Ineq.)} \\
& = \frac{1}{N}\sqrt{N\lambda_1(\tK_P)}tC \sqrt{ 2p\sum_{i = 1}^{N}\tilde{K}_{Pii}} \\
& = \frac{1}{N}\sqrt{N\lambda_1(\tK_P)}tC \sqrt{ 2pN} = tC\sqrt{2p\lambda_1(\tK_P)} \\
& \le tC\sqrt{2p\lambda_1(\tK)} \text{ (from Lem. \ref{lem:rad2})} \\
& = tC\lambda_1(\K)\sqrt{2p} \text{ (from Lem. \ref{lem:rad})}.
\end{align*}
\fi %%%%%%%%%%%%%%%%%%%%%%%%%%%%%%%%%%%%%%%%%%%%%

\end{proof}

\subsection{Proof of Lemma \ref{thm:rad}}

\ThmRadsp*

\begin{proof}
To show this, we first proof the following lemmas.

\begin{lem}
\label{lem:rad2}
Let $\tU_P = [\tu_{ij}]_{(i,j) \in \cP_n} \in \R^{d^2\times N}$ be the embedding matrix only for the node-pairs in $\cP_n$. $\tK_P = \tU_P^{\top}\tU_P,\,\tK = \tU^{\top}\tU$. Then $\lambda_{1}(\tK_P) \le \lambda_{1}(\tK)$.
\end{lem}

\begin{proof}
We have that $\lambda_{1}(\tK_P) = \sup_{\x \in \R^{N}}\frac{\x^{\top}\tK_P \x}{\|\x\|_2^2}$. 
Let $\x_1 = \underset{\x \in \R^N}{\text{argsup}}\frac{\x^{\top}\tK_P \x}{\|\x\|_2^2}$.

Note that $\tU = [\tU_{1}, \tU_{2}, \ldots, \tU_{n^2}] \in \R^{d^2 \times n^2}$, and $\tU_P = [\tU_{P_1}, \tU_{P_2}, \ldots, \tU_{P_N}] \in \R^{d^2 \times N}$, where $\tu_k = \u_{i_k}\circ\u_{j_k}, \, \forall (i_k,j_k) \in [n]\times[n]$ and ${(\tu_P)}_k = \u_{i_k}\circ\u_{j_k}, \, \forall (i_k,j_k) \in \cP_n$.

Let us define $k'(i,j) = n(i-1)+j, ~\forall (i,j) \in [n]\times [n]$ and $k(i,j) = \sum_{l=1}^{i-1}(n-l) + (j-i), ~~\forall (i,j) \in \cP_n$.

Clearly $\tK(k'(i,j),k'(u,v)) = \tK_P(k(i,j),k(u,v))$, $\forall (i,j), (u,v) \in \cP_n$ such that 
%$k_1 = \sum_{l=1}^{i-1}(n-l) + (j-i), ~k_2 = \sum_{l=1}^{u-1}(n-l) + (v-u)$, $k'_1 = n(i-1)+j, ~k'_2 = n(u-1)+v$.
%$\sum_{l=1}^{i}(n-l) + (j-i) = \frac{n(2i-2)-i^2+2j-i}{2}$
Now let us consider $\tx_1 \in \R^{n^2}$ such that
\begin{align*}
\tx_1(k'(i,j)) = \begin{cases} 
      \x_1(k(i,j)), \,\forall (i,j) \in \cP_n,\\
      0, \text{ otherwise}
   \end{cases}
\end{align*}

Note that this implies $\lambda_{1}(\tK) = \sup_{\tx \in \R^{n^2}}\frac{\tx^{\top}\tK \tx}{\|\tx\|_2^2} \ge \frac{\tx_1^{\top}\tK \tx_1}{\|\tx_1\|_2^2} = \frac{\x_1^{\top}\tK_P \x_1}{\|\x_1\|_2^2} = \lambda_{1}(\tK_P)$, proving the claim.
\end{proof}

\begin{lem}
\label{lem:rad}
Let $\tK = \tU^{\top}\tU,\,\K = \U^{\top}\U$, for any $\tU \in \text{SP-Lab}(G)$, and the corresponding $\U \in \text{Lab}(G)$. Then $\lambda_{1}(\tK) = (\lambda_{1}(\K))^{2}$.
\end{lem}

\begin{proof}
Note that $\lambda_{1}(\K) = \sup_{\x \in \R^n}\frac{\x^{\top}\K \x}{\|\x\|_2^2}$. Let $\x_1 = \text{argsup}_{\x \in \R^n}\frac{\x^{\top}\K \x}{\|\x\|_2^2}$.

The crucial observation is that 
\[
\tK = \K \circ \K = 
  \begin{bmatrix}
    K_{11}\K & \cdots & K_{1n}\K \\
    . & . & .\\
    . & . & .\\
    . & . & .\\
    K_{n1}\K & \cdots & K_{nn}\K \\
  \end{bmatrix}.
  \]
Let us define $\tx_1 = \x_1 \circ \x_1 \in \R^{n^2}$. Note that $\|\tx\|_2 = \|\x\|_2^2$. Then $\lambda_{1}(\tK) = \sup_{\tx \in \R^{n^2}}\frac{\tx^{\top}\tK \tx}{\|\tx\|_2^2} = \frac{\tx_1^{\top}\tK \tx_1}{\|\tx_1\|_2^2} = \frac{\big(\x_1^{\top}\K \x_1\big)^2}{\|\x_1\|_2^4} = (\lambda_{1}(\K))^2$.  
%It is easy to see that $\nexists \tx' \in \R^{n^2}$ such that $\frac{\tx'^{\top}\tK \tx'}{\|\tx'\|_2^2} \ge (\lambda(\K))^2$.
\end{proof}

Thus applying Lem. \ref{lem:rad2} and \ref{lem:rad}, we get, $\lambda_{1}(\tK_P) \le \lambda_{1}(\tK) = (\lambda_{1}(\K))^{2}$. The proof of Lem. \ref{thm:rad} now follows from Thm. \ref{thm:rad_eig}.

%%%%%%%%%%%%%%%%%%%%%%%%%%%%%%%%%%%%%%%%%%%%%%%%
\end{proof}

\subsection{Proof of Theorem \ref{cor:rad}}

\CorRadsp*

The proof follows by applying Lem. \ref{thm:rad} to Thm. \ref{thm:gen_err} for $p = f(1-f)$.

\subsection{Proof of Lemma \ref{thm:radpd}}

\ThmRadpd*

\begin{proof}
Let $\E = [\e_i-\e_j]_{(i,j)\in \cP_n} \in \{0,\pm 1\}^{n \times N}$, where $\e_i$ denotes the $i^{th}$ standard basis of $\R^n$, $\forall i \in [n]$. We start by proving the following lemma:

\begin{lem}
\label{lem:rad_pd}
If $\U \in Lab(G)$, $\K = \U^\top\U$, $\tU = \U\E \in$ \pdLab~ and $\tK = \tU^\top\tU$, then $\lambda_1(\tK) = 2n\lambda_1(\K)$.  
\end{lem}

\begin{proof}
By definition of $\lambda_{1}(\tK)$, we know that  \begin{align*}
\lambda_{1}(\tK) & = \sup_{\x \in \R^n}\frac{\x^{\top}\tK \x}{\|\x\|_2^2}\\
& = \sup_{\x \in \R^N}\frac{\x^{\top}\E^{\top}\K \E\x}{\|\x\|_2^2}\\
& = \sup_{\x \in \R^N}\frac{(\E\x)^{\top}\K (\E\x)}{\|\x\|_2^2}\\
& = \sup_{\x \in \R^N}\frac{\lambda_{\K}\|\E\x\|_2^2}{\|\x\|_2^2}\\
& \le 2n\lambda_{\K},
\end{align*}
where the last inequality follows from the fact that, for any $\x \in \R^N$, $\|\E\x\|_2^2 \le 2n\|\x\|_2^2$.
\end{proof}

Now further applying Thm. \ref{thm:rad_eig} for $\tU \in $ \pdLab, we get 
$R(\cH_{\tU}, \tU, p) \le  C\sqrt{2p\lambda_1(\tK)} \le C\sqrt{2p\lambda_1(\tK)}$, since $tr(\tK) \le N$ and the result now follows from Lem. \ref{lem:rad_pd}.
\end{proof}

\subsection{Proof of Theorem \ref{cor:radpd}}

\CorRadpd*

\begin{proof}
The proof follows by applying Lem. \ref{thm:radpd} to Thm. \ref{thm:gen_err} for $p = f(1-f)$. 
\end{proof}

\subsection{Proof of Lemma \ref{lem:sp_lsrad}}

\LemRadLSApprox*

\begin{proof}
For $G(n,q)$ graphs, \cite{furedi} showed that with high probability $1-e^{-\sqrt{n}}$, $\lambda_{1}(\A_G) = nq(1+o(1))$ and $|\lambda_{n}(\A_G)| \le 2\sqrt{nq(1-q)}$. As $q=O(1)$, note that $\lambda_1(\A_{G}) = \Theta(n)$ and $\lambda_n(\A_{G}) = \Theta(\sqrt{n})$. Thus, choosing $\tau=\Theta(\sqrt{n})$ makes $\K_{LS}(G)$ a positive semi-definite matrix, and clearly $\lambda_{1}(\K_{LS}(G)) = \Theta(\sqrt{n})$. Moreover since $\tK_{LS} = \K_{LS} \otimes \K_{LS}$, we have $\lambda_{1}(\tK_{LS}) = \big( \lambda_{1}(\K_{LS}) \big)^2$, as follows from Lem. \ref{lem:rad}). The claim now follows from Thm.~\ref{thm:rad_eig} and Lem. \ref{lem:rad2}.

%For the choice of $\K_{LS}(G) = \frac{\A_G}{\sqrt{\tau'}}+\I_n$ where $\tau'$ as defined in ~\eqref{eq:ls_labelling}, note the following relation --
%$$ \tK_{LS}(\strPr) = \K_{LS}(\strPr) + \frac{\sqrt{\tau'} -1 }{\tau'} (\A_G \otimes \I_n + \I_n \otimes \A_G) $$

%Using Weyl's perturbation inequality, we get $\lambda_1(\tK_{LS}(\strPr)) \le \lambda_1(\K_{LS}(\strPr)) + \lambda_1(\P)$, with perturbation matrix $\P$ defined as second term above. 

%For the choice of $G(n,q)$ graph, \cite{furedi} showed that with high probability $1-e^{-\sqrt{n}}$, $\lambda_1(\A_G) = \Theta(n)$, $\lambda_1(\A_{\strPr}) = \Theta(n^2)$ and $\lambda_n(\A_{\strPr}) = \Theta(n)$. Using $\tau'=\Theta(n)$, note that $\lambda_1(\K_{LS}(\strPr)) = \Theta(n)$ from Lem.~\ref{lem:sp_ls2} and $\lambda_1(\P)=\Theta(\sqrt{n})$. Using Weyl's inequality and Thm.~\ref{thm:rad} proves the claim.

% $\K_{LS}(G) = \U_{LS}^{\top}\U_{LS}$, and $\K'_{LS}(\strPr) = \tU_{LS}^{\top}\tU_{LS}$. To show that $R(\cH_{\tU_{LS}}, \tU_{LS}, p) = \Theta(n^{1/2})$. From Cor. $4.3$ of \cite{RakeshCh14}, we know that $R(\cH_{\U_{LS}}, \U_{LS}, p) = \Theta(n^{1/4})$ for $G(n,q)$ be a random graphs. Now one can show that $R(\cH_{\U_{LS}}, \U_{LS}, p) = \Theta(\lambda_1(\K_{LS}(G)))$ and $R(\cH_{\tU_{LS}}, \tU_{LS}, p) = \Theta(\lambda_1(\K'_{LS}(\strPr))) = \Theta((\lambda_1(\K_{LS}(G)))^2) = \Theta(n^{1/2})$.
\end{proof}

\subsection{Embedding with graph Laplacian.}
\label{app:lap}
The popular choice of graph kernel uses the inverse of the Laplacian matrix. Formally, let $d_i$ denotes the degree of vertex $i\in [n]$ in graph $G$, $d_i = {(\A_G)}_{i}^\top\1_n$, and $\D$ denote a diagonal matrix such that $D_{ii}=d_i,\forall i\in [n]$. %We refer $\D-\A_G$ and $\I-\D^{-\frac{1}{2}}\A_G\D^{-\frac{1}{2}}$ to be the Laplacian and normalized Laplacian matrix of the $G$ \cite{AndoZh07} respectively. 
Then the Laplacian and normalized Laplacian kernel of $G$ is defined as follows:\footnotemark
\begin{align*}
& \K_{Lap}(G)= (\D - \A_G)^{\dagger} \text{ and}\\
& \K_{nLap}(G) = (\I_n - \D^{-1/2}\A_G\D^{-1/2})^{\dagger}.
\end{align*}
\footnotetext{$\dagger$ denotes the pseudo inverse.}

Simlar to \ls, one could define the embedding of $\strPr$ using \spLab\ or \pdLab\ with $\K_{Lap}$ and $\K_{nLap}$.
%As before, similarly one can define the Laplacian and normalized-Laplacian embedding of $\strPr$ as 
%\begin{align*}
%& \tK_{Lap}(\strPr) = \K_{Lap}(G)\otimes \K_{Lap}(G), \text{ and}\\
%& \tK_{nLap}(\strPr) = \K_{nLap}(G)\otimes \K_{nLap}(G).
%\end{align*}
%\cite{RakeshCh14} shows that for a special case of random graphs - complete graph, normalized Laplacian inverse, the most widely kernel for graph based learning problems, has $O(1)$ complexity, thus making it a suboptimal choice of graph embedding for graphs with high connectivity compared to \ls. Our experimental results illustrate the observation.
%{\color{red} @Rakesh: Elaborate further if you have better insights on this. Btw, no need to analyse $\K_{nLap}(G)\otimes \K_{nLap}(G)$ right?} {\color{blue} No this should be good, since $\lambda_1(\K_{nLap}(G)\otimes \K_{nLap}(G)) =\lambda_1(\K_{nLap}(G))^2$, it folllows directly.}
% However for certain families of graphs, the Rademacher complexity of the function class associated $\K_{nLap}(G)$ can be quite low; e.g. for complete graphs $\K_{nLap}(G)$ has a complexity of $O(1)$ \cite{RakeshCh14}, thus making it a suboptimal choice of graph embedding for graphs with high connectivity compared to \ls. Our experimental results verify the observation.
% However for certain families of graphs, the Rademacher complexity of the function class associated $\K_{nLap}(G)$ can be quite low -- we summarize our findings in Table~\ref{table:ls_vs_lap}. 
However, we observe that the Rademahcer complexity of function associated with Laplacian is an order magnitude smaller than that of \ls\ for graphs with high connectivity -- we summarize our findings in Table~\ref{table:ls_vs_lap}.  Experimental results in Section~\ref{sec:experiments} illustrate our observation.

\vspace{-5px}
\begin{table}[ht]
\centering
\begin{tabular}{|c|c|c|}
\hline
{ \bf Graph } & {\bf Laplacian }& { \bf \ls }\\
 \hline
Complete graph $K_{n}$		& $\Theta(1)$  			&$\Theta(n)$\\
\hline
Random Graphs $G(n,1/2)$	& $\Theta(1)$ 			&$\Theta(\sqrt{n})$\\
\hline
Complete Bipartite 	& $\Theta(1)$			&$\Theta(1)$\\
\hline
Star $S_n$				& $\Theta(1)$			& $\Theta(1)$\\
\hline
\end{tabular}
\centering
\caption{Rademacher complexity measure of Laplacian and \ls\ graph embeddings (assuming $C,p=O(1)$).}
\label{table:ls_vs_lap}
\vspace{-2px}
\end{table}
\vspace{-5px}
%Table ~\ref{table:ls_vs_lap} summarizes our findings on other graph families.

\section{Appendix for Section \ref{sec:consistency}}

\subsection{Proof of Theorem \ref{thm:const}}

\ThmConst*

\begin{proof}
We first bound the total number of pairwise mispredictions of $f$, given by $N{er}^{\ell^{0-1}}_{n}[\f]$. Note that  
$N{er}^{\ell^{0-1}}_{n}[\f] \le N{er}^{\ell^{ramp}}_{n}[\f]$. (see Sec. \ref{sec:prb_st} for definitions of $\ell^{0-1}_{n}[\f]$ and $\ell^{ramp}_{n}[\f]$).

Now applying Cor. \ref{cor:rad} for ramp loss $\ell^{ramp}$ with $\delta = \frac{1}{n}$, we get that with probability atleast $(1-\frac{1}{N})$, %\label{eg:yanivpe3}
\begin{align*}
& Ner_{n}^{\ell^{ramp}}[\f] =  \sum_{k = 1}^{N} \ell^{ramp}(y_k,f_k)\\ 
& = \left( \sum_{k = Nf+1}^{N} \ell^{ramp}(y_k,f_k) + \sum_{k = 1}^{Nf}\ell^{ramp}(y_k,f_k) \right) \\
& = \left( (N-Nf){er}^{l^{ramp}}_{\bS}[\f] + Nf{er}^{l^{ramp}}_{S}[\f] \right) \\
& \le \Bigg( (N(1-f)+Nf){er}^{l^{ramp}}_{S}[\f] \\
& \qquad + \frac{N\sqrt{2(1-f)}C\lambda_1(\K)}{\rho \sqrt{f}} +\frac{C_1NB \sqrt{\ln N}}{\sqrt{Nf}} \Bigg)\\
& \qquad\qquad\qquad\qquad\qquad\qquad\qquad~ (\text{ from Thm. }\ref{cor:rad})\\
& \hspace{-8pt} = N\Bigg(er_{S}^{\ell^{ramp}}[\f] + \frac{C\lambda_1(\K)\sqrt{2(1-f)}}{\sqrt f} +\frac{C_1 \sqrt{2\ln n}}{\sqrt{Nf}} \Bigg),\\
& \hspace{-8pt} \le N\Bigg( er_{S}^{\ell^{hinge}}[\f] + \frac{C\lambda_1(\K)\sqrt{2(1-f)}}{\sqrt f} +\frac{C_1 \sqrt{2\ln n}}{\sqrt{Nf}} \Bigg)
\end{align*}
%\todoa{why hinge loss equal svm objective?}
where the second last inequality is because for ramp loss, $\rho$ and $B$ both are $1$. The last equality follows from the fact that hinge loss is an upper bound of the ramp loss.

Let us define $\tU_P = [\tu_{ij}]_{(i,j) \in \cP_n} \in \R^{d^2\times N}$ to be the embedding matrix only for the node-pairs in $\cP_n$. $\tK_P = \tU_P^{\top}\tU_P$.
Also let us define
\begin{align*}
\text{PSP-Lab}(\strPr) & = \{ \tU_P \in \R^{d^2 \times N} \mid \U \in \mbox{Lab}(G)\}.
\end{align*} 

The key of the proof lies in the following derivation that maps $\vartheta(G)$ to the training set error $er_{S}^{\ell^{hinge}}[\f]$. Specifically, note that:

%observation here is the training set error when used along with the optimal orthogonal embedding (corrected with respect to the pairwise labels) 
\begin{align}
\label{eq:con_thm10}
\nonumber & 2C(Nf)er_{S}^{\ell^{hinge}}[\f] = 
2C\sum_{i = 1}^{Nf}\ell^{hinge}(y_k,f_k)\\
\nonumber & \le
\min_{\tU \in \mbox{PSP-Lab}(\strPr), \|\c\|_{2}=1} \max_{k = 1}^{N}\frac{1}{(c^{\top}\tU_k)^{2}} \\
\nonumber & \le
\min_{\tU \in \text{\spLab }, \|\c\|_{2}=1} \max_{k = 1}^{n^2}\frac{1}{(c^{\top}\tU_k)^{2}} \\
\nonumber & = \min_{\tU \in \mbox{Lab}(\strPr), \|\c\|_{2}=1} \max_{k = 1}^{n^2}\frac{1}{(c^{\top}\tU_k)^{2}}  \\
\nonumber &= {\vartheta(\strPr )} \\
& = {(\vartheta(G))^{2}},
\end{align}
where the first inequality follows from a similar derivation as given in Thm. $5.2$ of \cite{RakeshCh14} which relates optimum SVM objective to \lo. The second inequality is obvious as PSP-Lab$(\strPr) \subset $ \spLab. 
Thus we get that $er_{S}^{\ell^{hinge}}[f] = \frac{(\vartheta(G))^{2}}{2CNf}$. Combining everything we now have:
\begin{align*}
& er_{n}^{\ell^{0-1}}[\f] \le er_{n}^{\ell^{ramp}}[f] = \frac{1}{N}\sum_{k = 1}^{N} \ell^{ramp}(y_k,f_k)\\ 
& \le \Bigg( \frac{\vartheta(G_n)^{2}}{2CNf} + \frac{C\lambda_1(\K)\sqrt{2(1-f)}}{\sqrt f} +\frac{C_1 \sqrt{2\ln n}}{\sqrt{Nf}} \Bigg)\\
& \le \Bigg( \frac{\vartheta(G_n)^{2}}{2CNf} + \frac{Cn\sqrt{2(1-f)}}{\vartheta(G)\sqrt f} +\frac{C_1 \sqrt{2\ln n}}{\sqrt{Nf}} \Bigg)
%& \le \Bigg( \frac{(1-f)\vartheta(G_n)^{2}}{CNf} + \frac{C\lambda_1(\K)}{f} +\frac{C_1\sqrt{\ln n}}{\sqrt{Nf}} \Bigg)\\
\end{align*}
Where the last inequality follows from $\lambda_1(\K) \le \vartheta(\bar{G}_n)$, $\vartheta({G}_n)\vartheta(\bar{G}_n) = n$ \cite{lovasz_shannon}. Further optimizing over $C$ we get that at $C^* =  \bigg(\frac{\vartheta(G_n)^3}{Nn\sqrt{8f(1-f)}}\bigg)^{\frac{1}{2}}$, using which we get
\begin{align*}
& er_{n}^{\ell^{0-1}}[\f] = \frac{1}{N} \sum_{k = 1}^{N} \ell^{0-1}(y_k,h_k) \\
& \le \Bigg( 2\bigg( \frac{\vartheta(G_n)}{(n-1)f} \sqrt{\frac{2(1-f)}{f}} \bigg)^{\frac{1}{2}} + \frac{C_1\sqrt{2\ln n}}{\sqrt{Nf}} \Bigg) \\
%& = O\left( \frac{2\sqrt{\vartheta(G_n)}}{\sqrt{f(n-1)}}  + \frac{1}{1-f}\sqrt{\frac{2\log n}{Nf}} \right)
\end{align*}
Above proves the first half of the result. The second result immediately follows from above with the additional observation that $d_k(\sigma^*,\hsig) \le 3er_{n}^{\ell^{0-1}}[\f]$ and and the fact that 
\begin{equation}
d_s(\bsig_1, \bsig_2) \le 2d_k(\bsig_1, \bsig_2)
\end{equation}
where $d$ being the Kendall's tau $(d_k)$ or Spearman's footrule $(d_s)$ ranking loss, which concludes the proof.
\end{proof}

\textbf{Proof of Lemma \ref{cor:sam_com}}
\sampcomp*

\begin{proof}
From Thm. \ref{thm:const} we have that there exists a constant positive $C_0 >0$ and an positive integer $n_0 \in \N$ such that $\forall n \ge n_0$
\begin{align}
\label{eq:sc1}
\nonumber d(\bsig^*_{n},\bhsig_{n}) & \le C_0\Bigg( \bigg( \frac{\vartheta(G_n)}{nf} \sqrt{\frac{(1-f)}{f}} \bigg)^{\frac{1}{2}} + \frac{C_1\sqrt{2\ln n}}{\sqrt{Nf}} \Bigg)\\
& \le C_0\Bigg(  \frac{1}{f^{\frac{3}{4}}}\sqrt{\bigg( \frac{\vartheta(G_n)}{n}  \bigg)} + \frac{C_1\sqrt{2\ln n}}{\sqrt{Nf}} \Bigg)
\end{align}
Now that if $\vartheta(G_n) = o(n^{c})$ for some $c \in [0,1)$ (recall $\vartheta(G_n) \in [1,n]$) and if we choose $\varepsilon \le \frac{(1-c)}{2}$ this makes $f^* = \bigg( \frac{\sqrt {\vartheta(G_n)}}{n^{\frac{1}{2}-\varepsilon}} \bigg)^{\frac{4}{3}}$ to be a valid assignment as that ensures $f^* \in [0,1]$.
Furthermore, \eqref{eq:sc1} suggests that observing only $f^*$ fraction of nodes would suffice to achieve ranking consistency since that implies $d(\bsig^*_{n},\bhsig_{n}) = O(\frac{1}{n^\varepsilon}) \to 0$, as $n \to \infty$.
\end{proof}

\textbf{Proof of Theorem \ref{thm:sam_com_clr} }

\sampcompclr*

\begin{proof}
From Lemma \ref{cor:sam_com}, using $f^*$ fraction of nodes immediately leads to the  sample complexity:
\[
Nf^* \le \frac{n^2}{2}\bigg( \frac{{\vartheta(G_n)}}{n^{1-2 \varepsilon}} \bigg)^{\frac{2}{3}}  = \frac{1}{2}(n^{2 + 2\varepsilon}\vartheta(G_n))^{\frac{2}{3}}
\]
The result now follows from Lem. \ref{cor:sam_com} and Lov\'asz sandwich theorem: $\vartheta(G) \le \chi(\bar{G})$ for any graph $G$ \cite{lovasz_shannon}.
\end{proof}

\textbf{Proof of Corollary \ref{cor:sam_com_spcl}}

\sampcompspcl*

\begin{proof}
The result follows from the proof of Theorem \ref{thm:sam_com_clr} upon  by substituting the values of $\vartheta(G)$ or $\chi(\bar G)$ (note $\vartheta(G) \le \chi(\bar G)$) in the corresponding graphs as given below:

\begin{enumerate}
\item \emph{Complete graphs}: $\chi(\bar G) = 1$
\item \emph{Union of $k$ disjoint cliques}: $\chi(\bar G) = k$
\item \emph{Complement of Power-law graphs}: $\vartheta(\bar{G}) = \Theta(\sqrt n)$ \cite{RakeshCh14,chi_power_law}
\item \emph{Random graphs}: $\vartheta(G) = \Theta(\sqrt{n})$, with high probability~\cite{coja}.
\item \emph{Complement of $k$-colorable graphs}: $\chi(\bar G) = k$. 
% \item \emph{Disjoint graphs}: $\chi(\bar G) = n$
%\item Also \emph{Complement of Planar graphs}: $\chi(\bar G) \le 4$ (by famous ``Four color" theorem \cite{color4})
\end{enumerate} 
\end{proof}

%%%%%%%%%%%%%%%%%%%%%%%%%%%%%%%%%%%%%%%%%%%%

%%%%%%%%%%%%%%%%%%%%%%%%%%%%%%%%%%%%%%%%%%%%%%%%%%

\section{Additional Experiments}
\label{app:expts}

\subsection{Additional Results: Experiments of Synthetic Datasets}
\label{app:synth}

\textbf{Plots comparing only PR-Kron, PR-PD, and GR}

\vspace{-10pt}
\begin{figure}[H]
%\begin{widepage}
\hspace{-15pt}
\includegraphics[trim={2.8cm 1.1cm 2.7cm 0.4cm},clip,scale=0.2,width=0.25\textwidth]{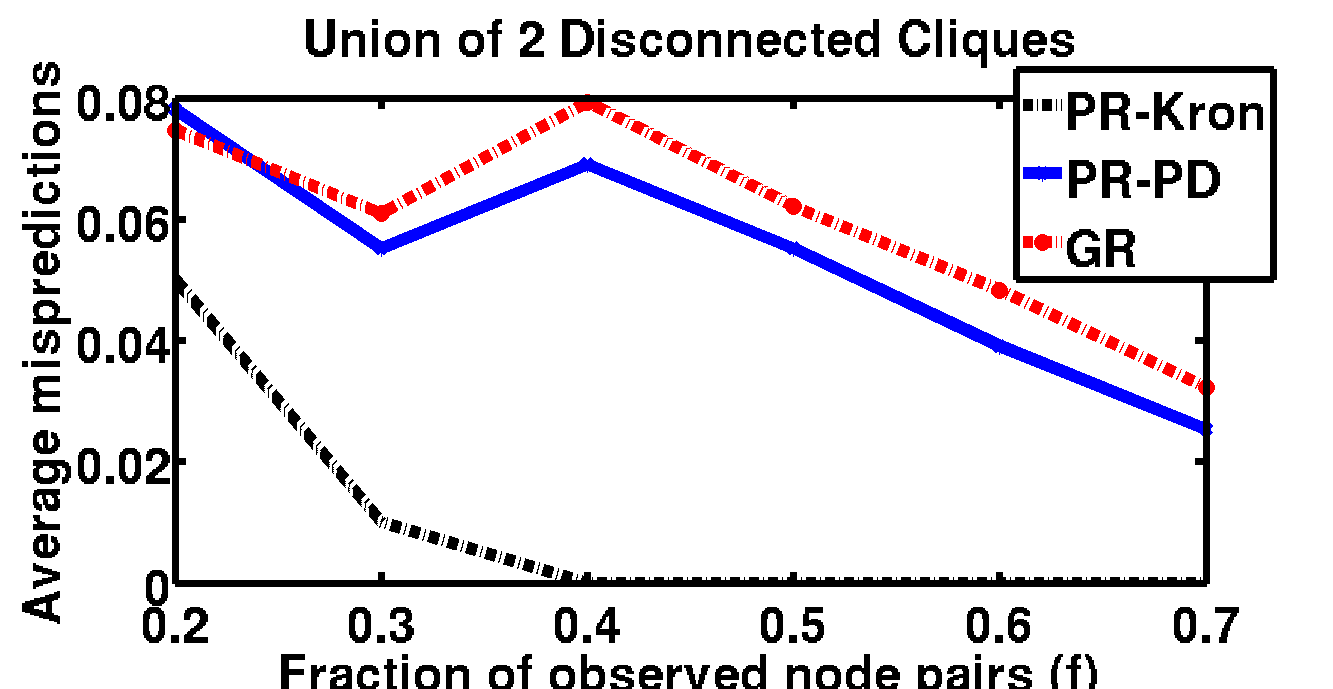}
\hspace{-5pt}
\includegraphics[trim={2.8cm 1.1cm 2.7cm 0.4cm},clip,scale=0.2,width=0.25\textwidth]{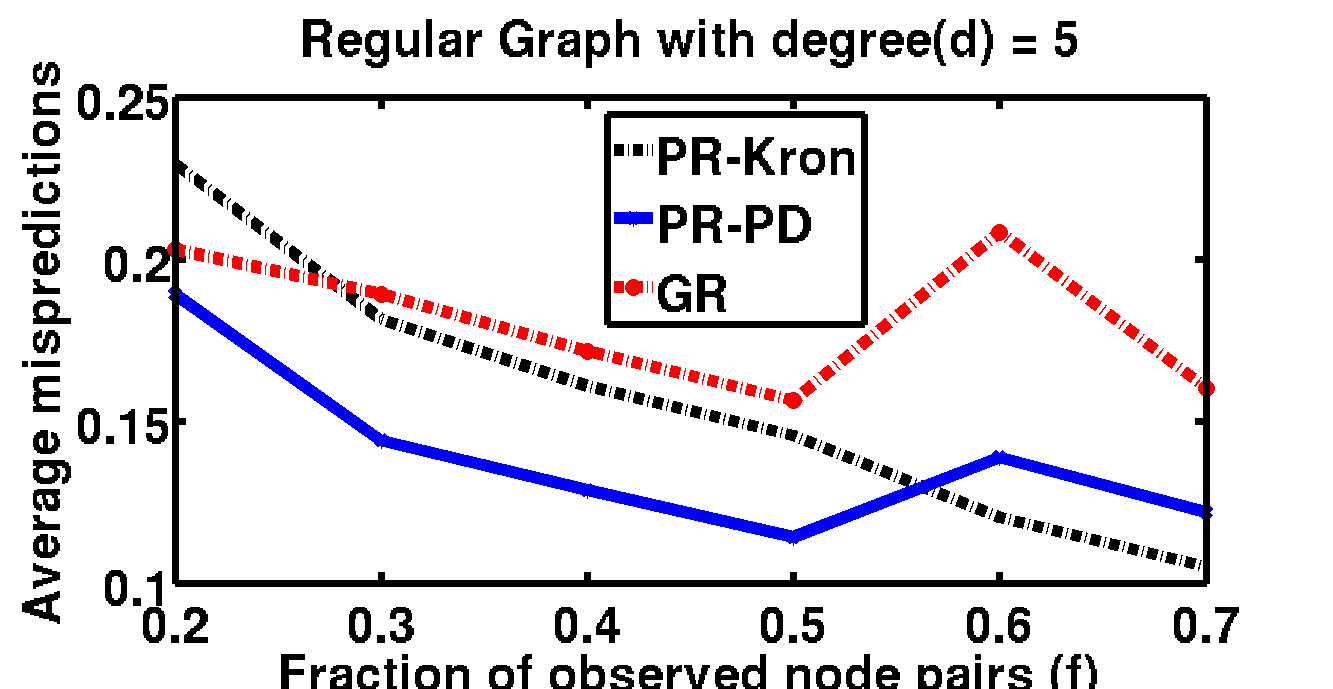}

\hspace{-15pt}
\includegraphics[trim={2.8cm 1.1cm 2.7cm 0.4cm},clip,scale=0.2,width=0.25\textwidth]{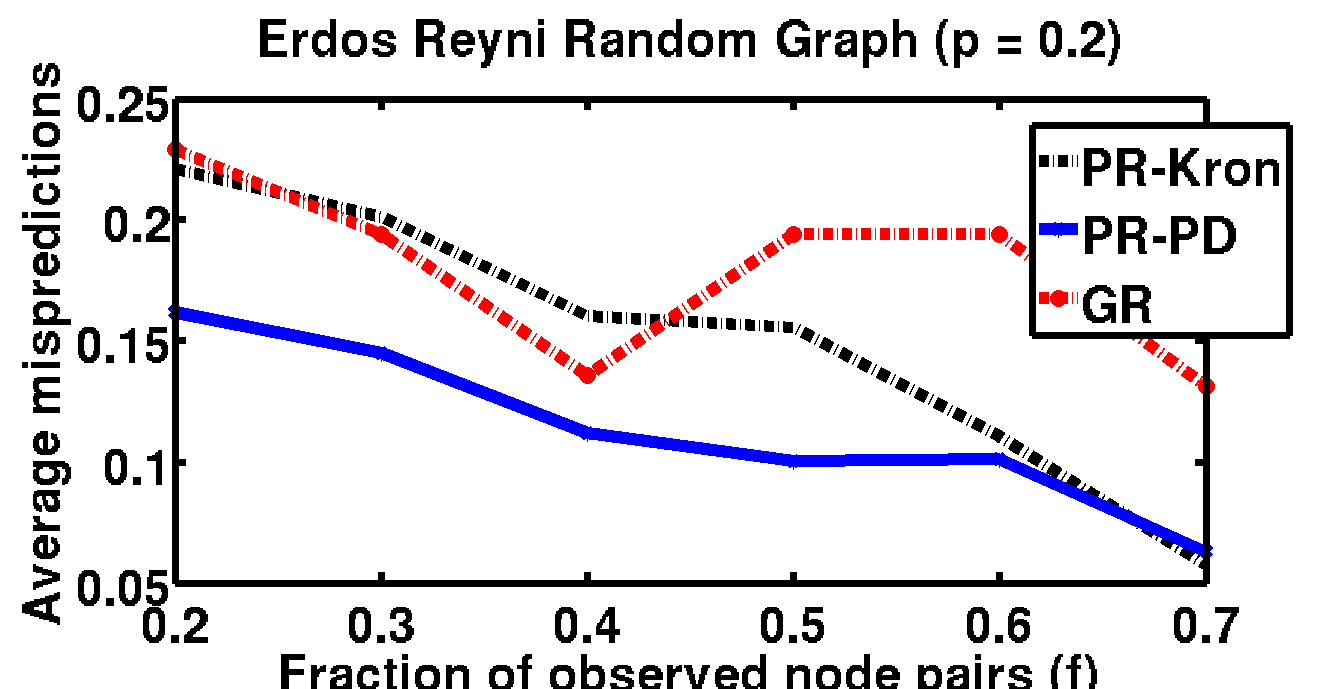}
\hspace{-5pt}
\includegraphics[trim={2.8cm 1.1cm 2.7cm 0.4cm},clip,scale=0.2,width=0.25\textwidth]{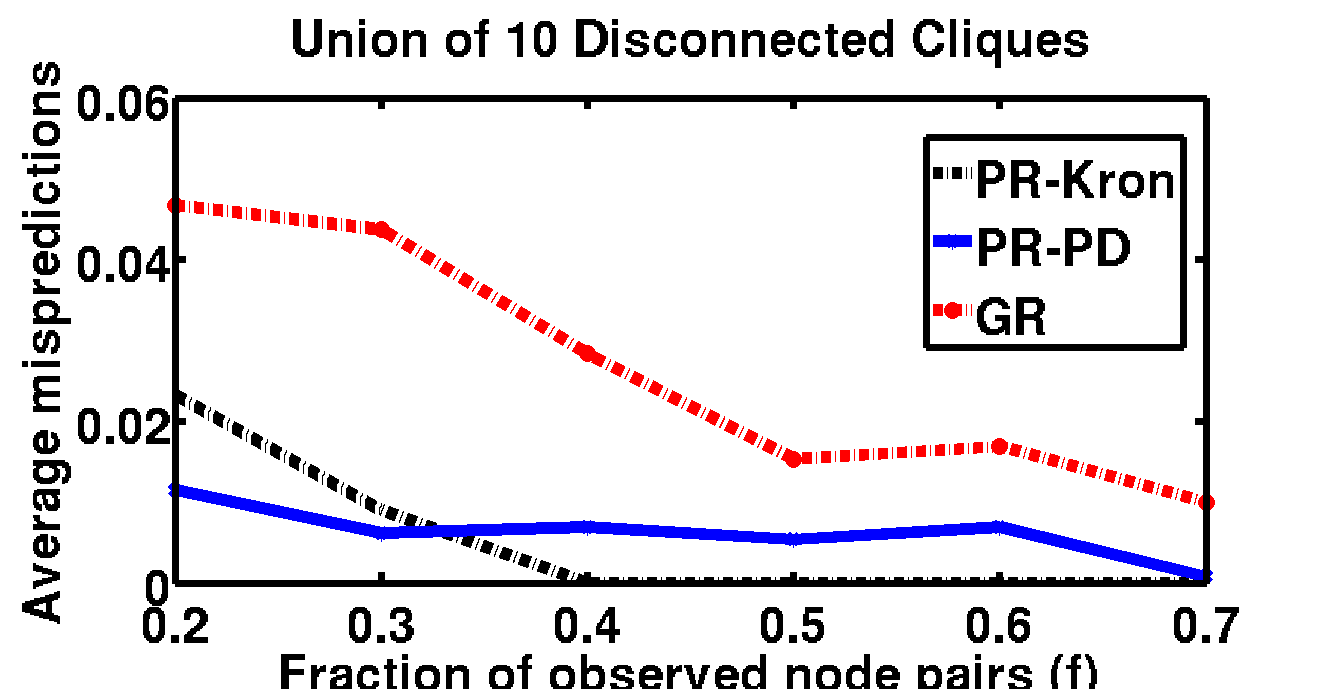}

\hspace{-15pt}
\includegraphics[trim={2.8cm 1.1cm 2.7cm 0.4cm},clip,scale=0.2,width=0.25\textwidth]{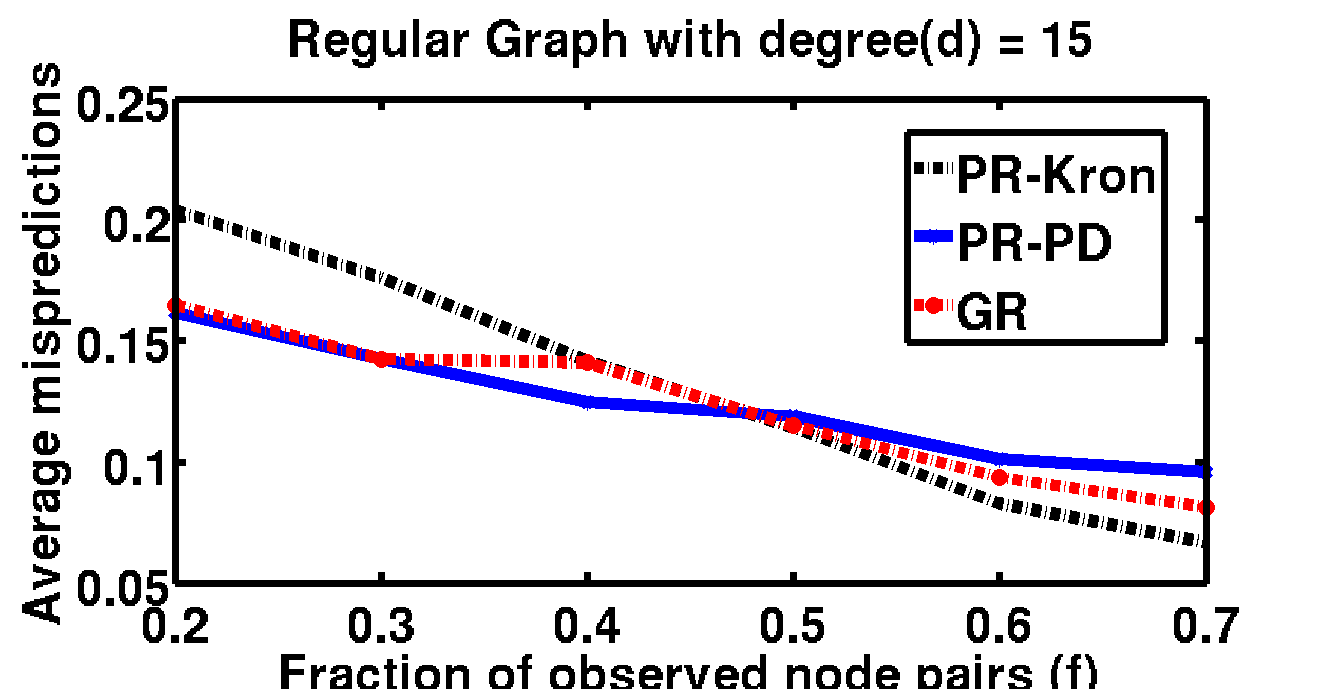}
\hspace{-5pt}
\includegraphics[trim={2.8cm 1.1cm 2.7cm 0.4cm},clip,scale=0.2,width=0.25\textwidth]{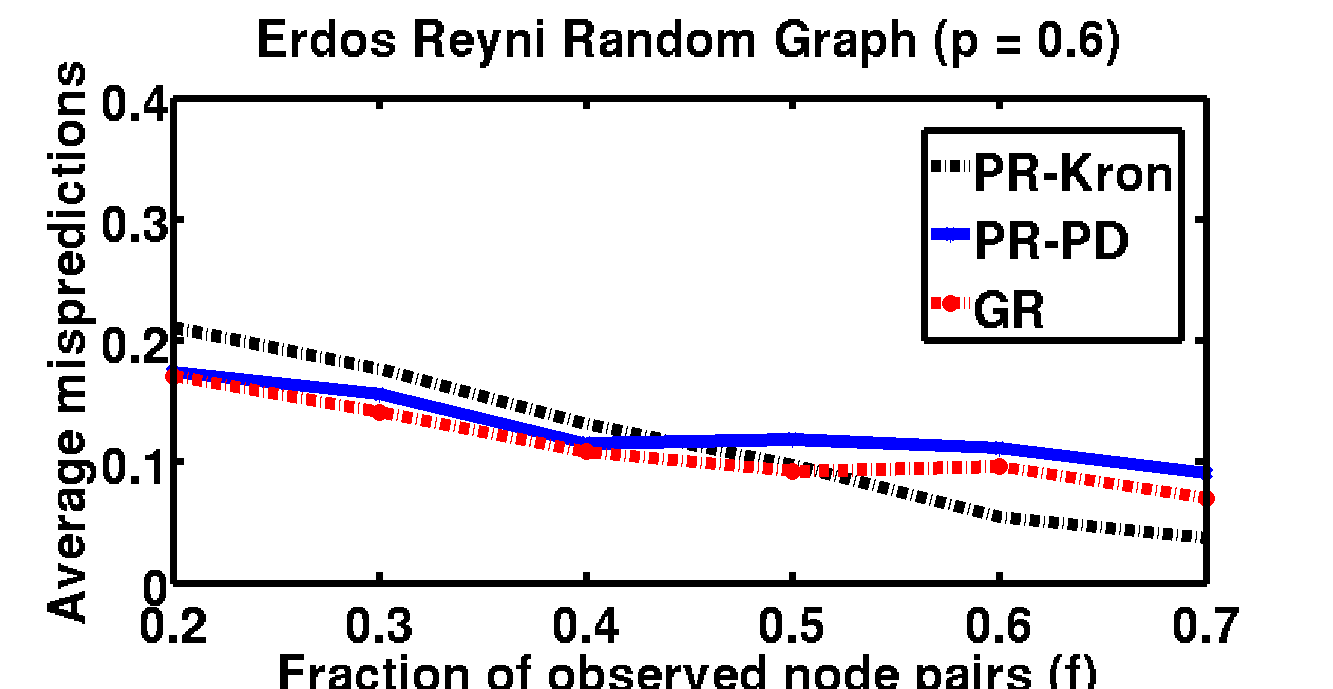}
\label{fig:plots_synth2}
\hspace{-10pt}
\vspace*{-10pt}
\caption{Synthetic Data: Average number of misprediction  ($er_{D}^{\ell^{0\text{-}1}}(\f)$, Eqn. \ref{eq:err_gen}) vs fraction of sampled pairs$(f)$ }%on different types of graphs
\hspace{-40pt}
%\end{widepage}
\end{figure}
\vspace{-10pt}

\subsection*{More Synthetic Experiments}
We consider a $G(n,p,q)$ random graph with $n=100$ nodes, $p = 0.6, q = 0.1$,
where nodes $[1$-$50]$ and $[51$-$100]$ are densely clustered, and nodes within the same cluster are connected with edge probability $p$ and that of two different clusters are connected with probability $q$. 
We also consider the nodes within same cluster to be closer in terms of their preference scores. 
More specifically, for the task of full ranking, we randomly assign a permutation to the 100 nodes such that all nodes in cluster 1 are ranked above all nodes in cluster 2 (below 50 and all nodes $(51$-$100)$ are ranked above 50). Similarly for ordinal ranking we randomly assign a rating from $1-10$ to each graph node such that all nodes in cluster 1 are rated higher than that of cluster 2. Finally for Bipartite ranking, we randomly assign a (0,1) binary label to each node such that nodes in cluster one are $80\%$ more likely to score higher than that of cluster 2. 
For each of the three tasks, we repeat the experiment for $10$ times and compare the averaged performances of \textbf{PR-Kron} with \textbf{GR}. Table \ref{tab:synth} shows that on an average \alg~ with \spLab\, performs better than \sa~ for all three tasks. 

\vspace{-5pt}
\begin{table}[H]
\begin{center}
\scalebox{0.95}{
 \begin{tabular}{| c c c |} 
 \hline
 Task & PR-Kron (in \%) & GR (in \%) \\ 
 \hline
  BR & {\bf 07.5} & 08.2 \\
  OR(10) & {\bf 12.3} & 17.6 \\
  FR & {\bf 11.8} & 18.6 \\ 
 \hline
\end{tabular}
}
\vspace*{5pt}
\caption {Synthetic data: Average number of mispredictions.}
\label{tab:synth}
\end{center}
\end{table}
\vspace{-20pt}

\subsection{Additonal Results: Experiments of Real Datasets}
\label{app:real}

\textbf{Plots comparing only PR-Kron, PR-PD, and GR}

\vspace{-10pt}
\begin{figure}[H]
%\begin{widepage}
\hspace{-15pt}
\includegraphics[trim={2.8cm 1.1cm 2.7cm 0.4cm},clip,scale=0.2,width=0.25\textwidth]{./heart_t.png}
\hspace{-5pt}
\includegraphics[trim={2.8cm 1.1cm 2.7cm 0.4cm},clip,scale=0.2,width=0.25\textwidth]{./vehicle_t.png}

\hspace{-15pt}
\includegraphics[trim={2.8cm 1.1cm 2.7cm 0.4cm},clip,scale=0.2,width=0.25\textwidth]{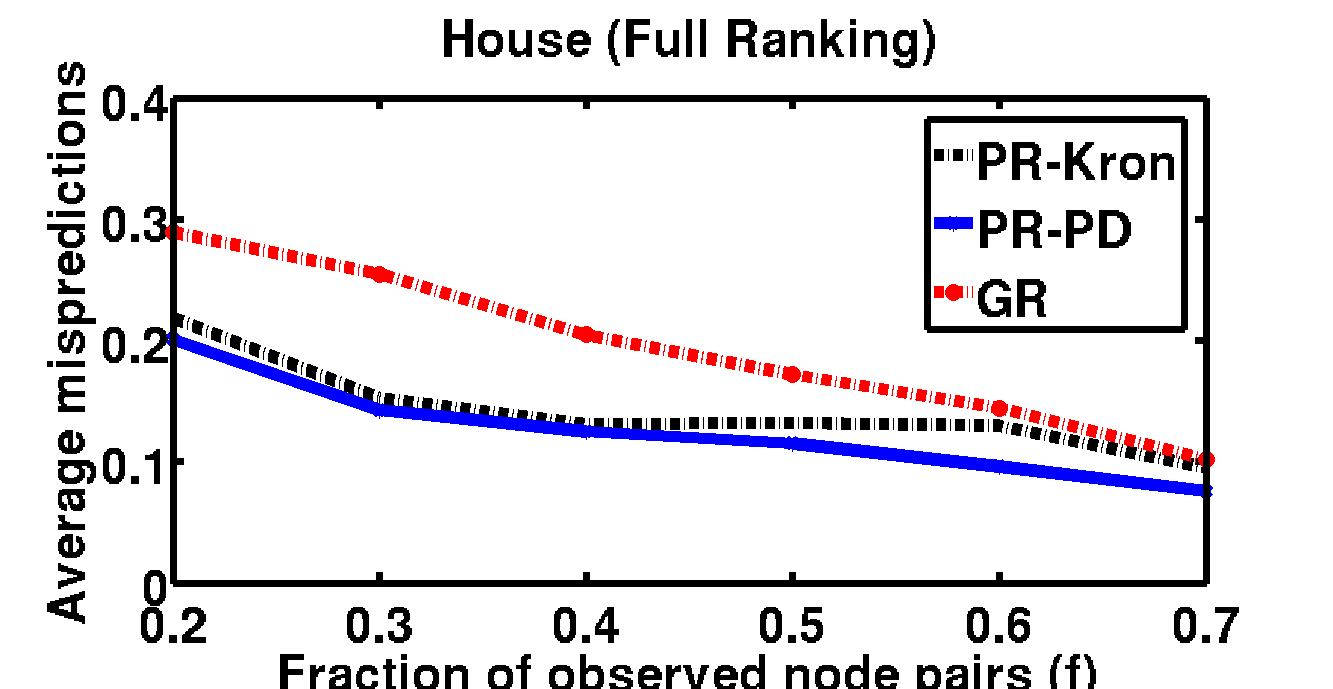}
\hspace{-5pt}
\includegraphics[trim={2.8cm 1.1cm 2.7cm 0.4cm},clip,scale=0.2,width=0.25\textwidth]{./fourclass_t.png}

\hspace{-15pt}
\includegraphics[trim={2.8cm 1.1cm 2.7cm 0.4cm},clip,scale=0.2,width=0.25\textwidth]{./vowel_t.png}
\hspace{-5pt}
\includegraphics[trim={2.8cm 1.1cm 2.7cm 0.4cm},clip,scale=0.2,width=0.25\textwidth]{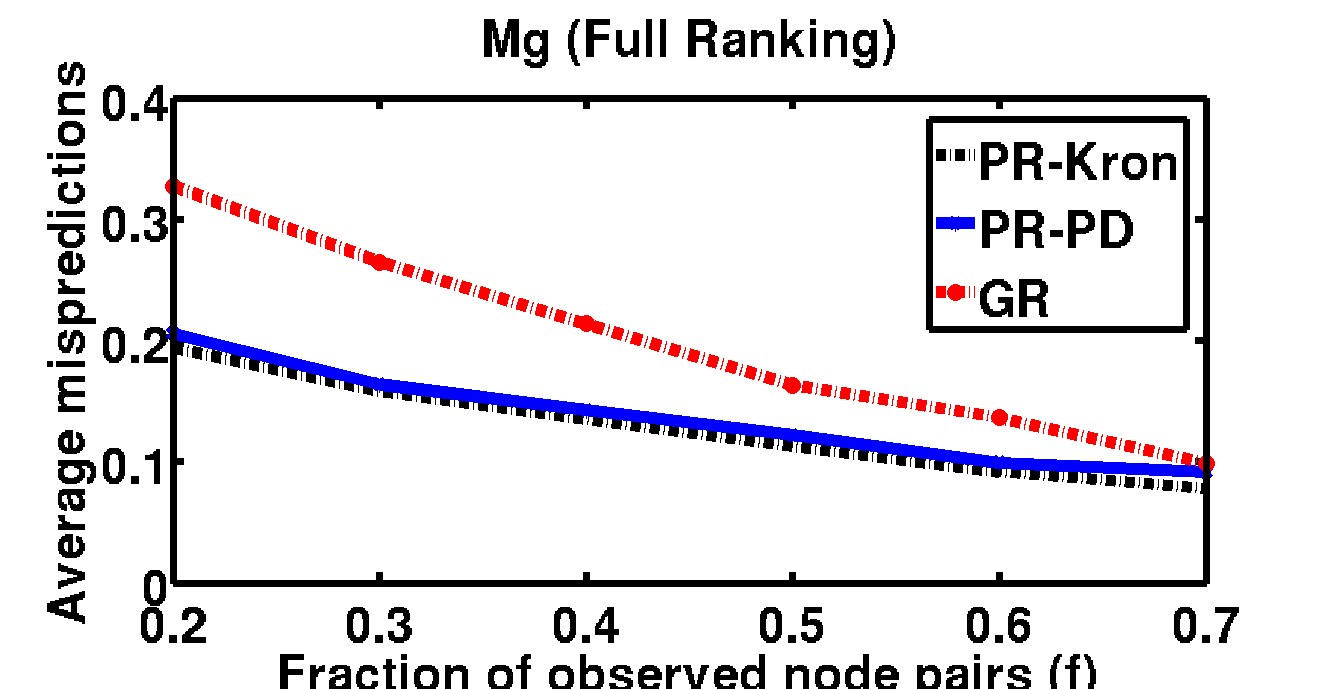}
\label{fig:plots_real2}
\vspace*{-10pt}
\caption{Real Data: Average number of misprediction ($er_{D}^{\ell^{0\text{-}1}}(\f)$, Eqn. \ref{eq:err_gen}) vs fraction of sampled pairs$(f)$}
\hspace{-40pt}
%\end{widepage}
\end{figure}
\vspace{-10pt}

\subsection{More Experiments on Real Datasets}
\label{app:real_add}

\textbf{Datasets.} $a$. \emph{Ionosphere} and \emph{Diabetes} for \textbf{BR} $b$. \emph{Bodyfat} for \textbf{FR}.

\vspace{-10pt}
\begin{figure}[H]
%\begin{widepage}
\hspace{-33pt}
\includegraphics[trim={2.8cm 1.1cm 2.7cm 0.4cm},clip,scale=0.2,width=0.18\textwidth]{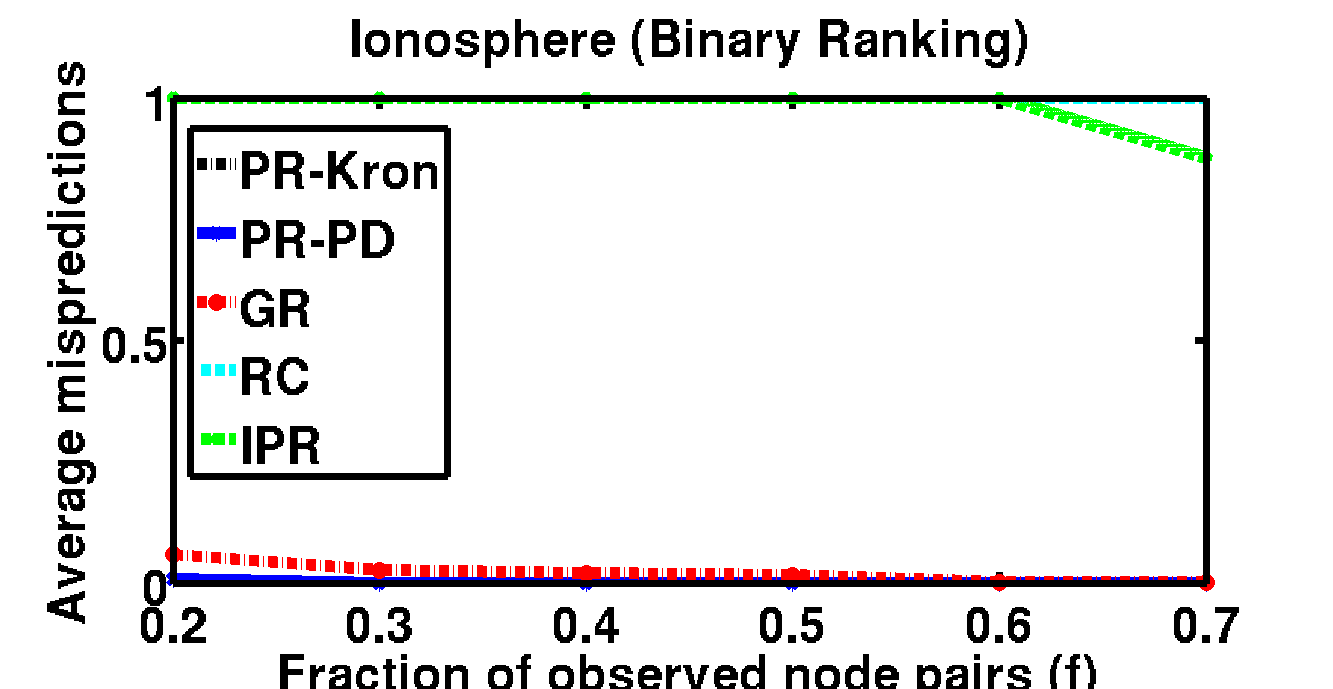}
\hspace{-6pt}
\includegraphics[trim={2.8cm 1.1cm 2.7cm 0.4cm},clip,scale=0.2,width=0.18\textwidth]{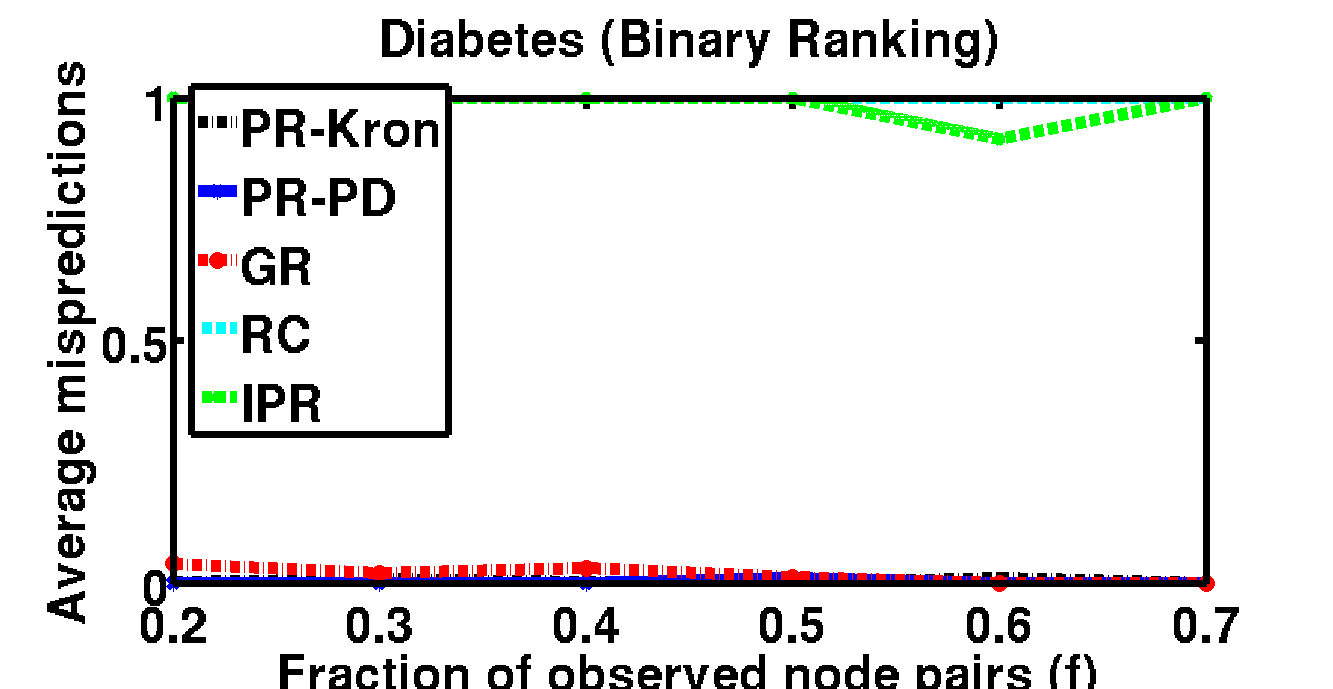}
\hspace{-6pt}
\includegraphics[trim={2.8cm 1.1cm 2.7cm 0.4cm},clip,scale=0.2,width=0.18\textwidth]{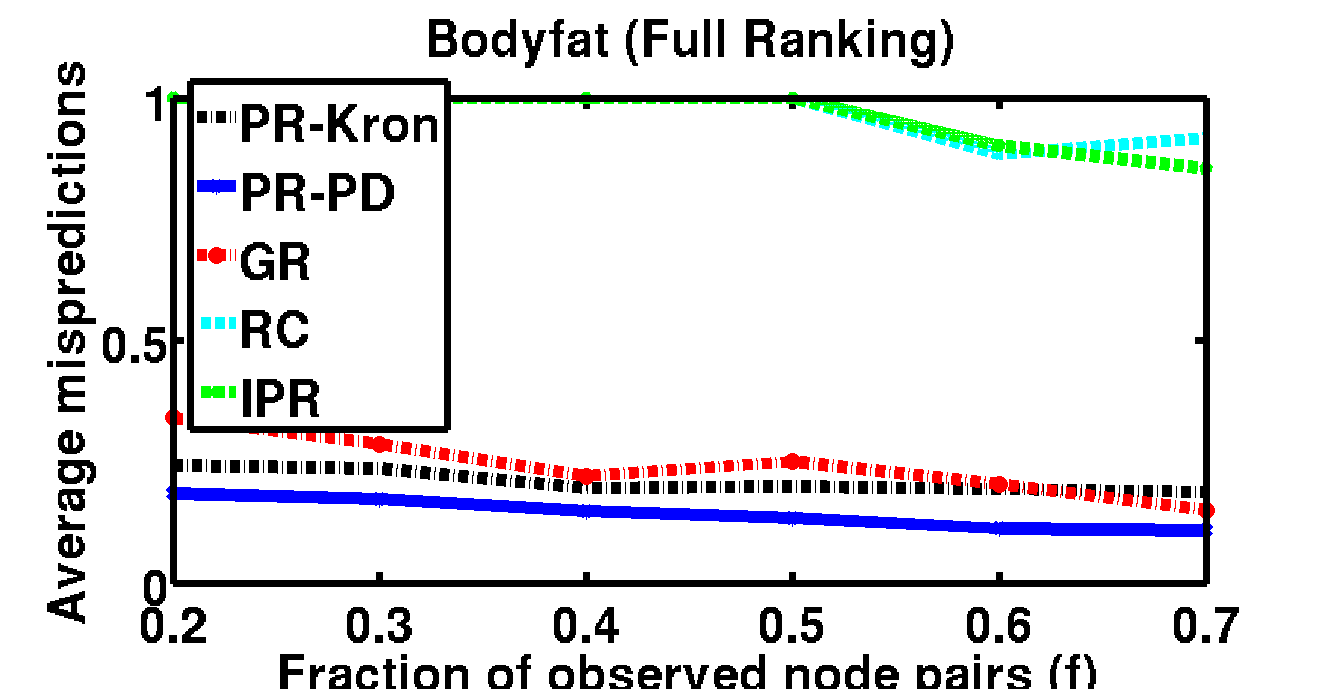}
\label{fig:plots_real3}
\vspace*{-20pt}
\caption{Real Data: Average number of misprediction ($er_{D}^{\ell^{0\text{-}1}}(\f)$, Eqn. \ref{eq:err_gen}) vs fraction of sampled pairs$(f)$}
\hspace{-40pt}
%\end{widepage}
\end{figure}
\vspace{-10pt}

\textbf{Plots comparing only PR-Kron, PR-PD, and GR}

\vspace{-10pt}
\begin{figure}[H]
%\begin{widepage}
\hspace{-33pt}
\includegraphics[trim={2.8cm 1.1cm 2.7cm 0.4cm},clip,scale=0.2,width=0.18\textwidth]{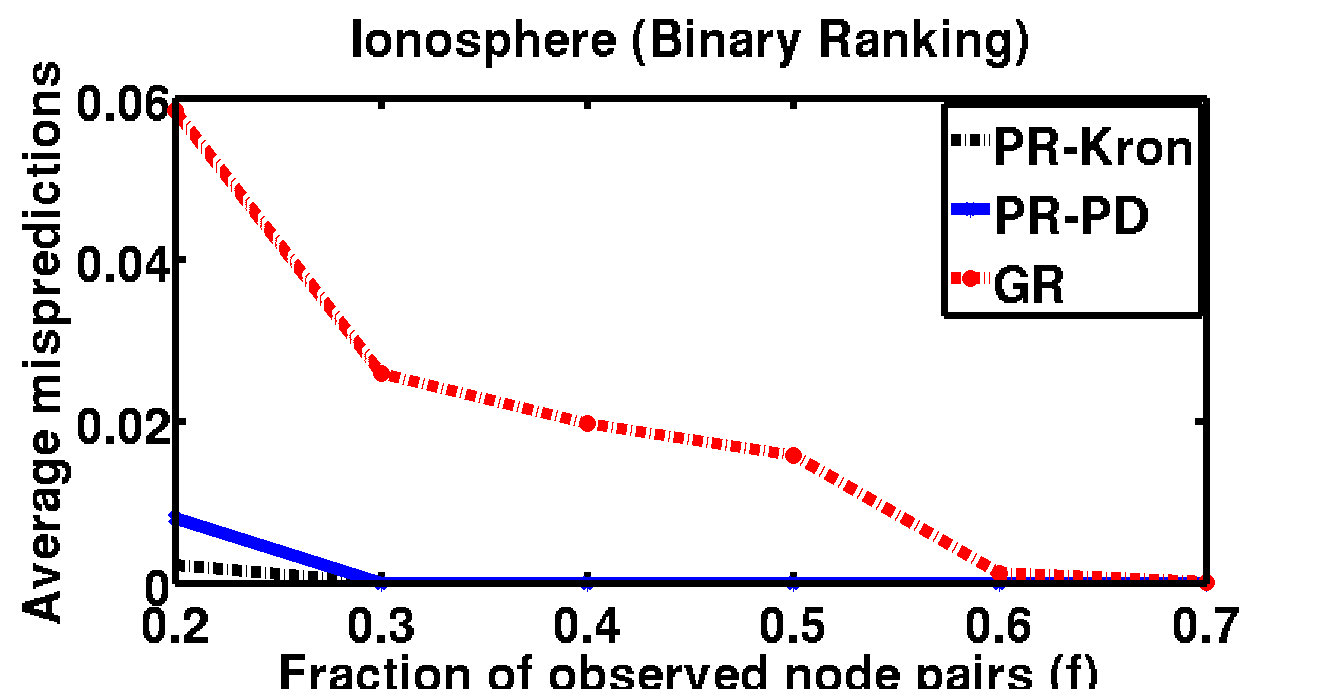}
\hspace{-6pt}
\includegraphics[trim={2.8cm 1.1cm 2.7cm 0.4cm},clip,scale=0.2,width=0.18\textwidth]{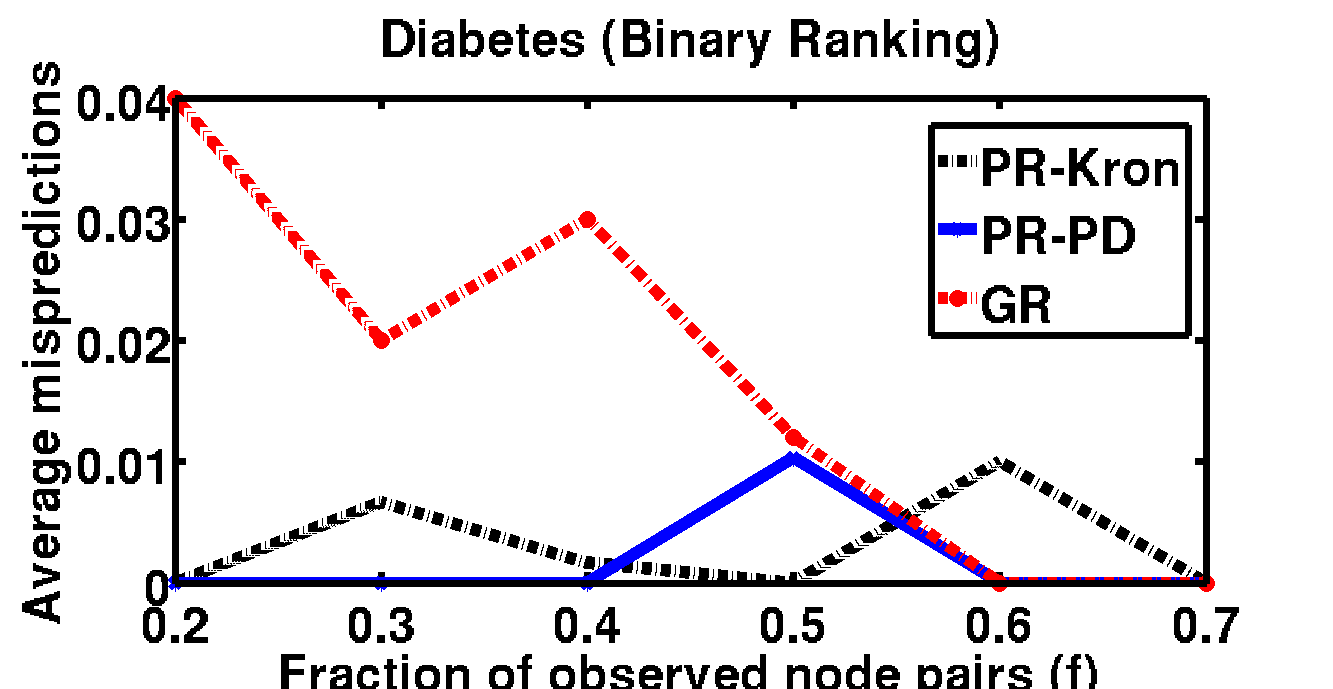}
\hspace{-6pt}
\includegraphics[trim={2.8cm 1.1cm 2.7cm 0.4cm},clip,scale=0.2,width=0.18\textwidth]{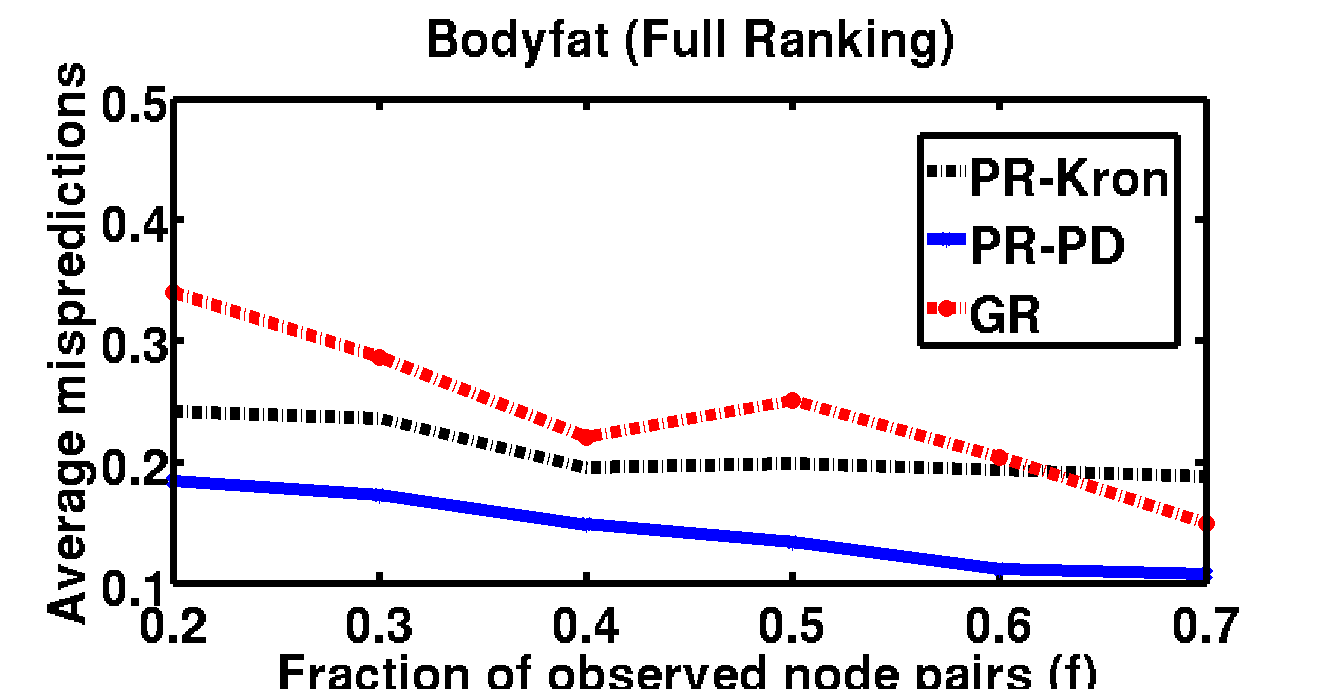}
\label{fig:plots_real4}
\vspace*{-20pt}
\caption{Real Data: Average number of misprediction ($er_{D}^{\ell^{0\text{-}1}}(\f)$, Eqn. \ref{eq:err_gen}) vs fraction of sampled pairs$(f)$}
\hspace{-40pt}
%\end{widepage}
\end{figure}
\vspace{-10pt}

% ---------------------------------------------

\end{document}